\documentclass[11pt]{article}

\usepackage{booktabs} 
\usepackage{amsmath}
\usepackage{amsthm}
\usepackage{amssymb}
\usepackage{algorithm}
\usepackage{algorithmic}
\usepackage{subcaption}
\usepackage{placeins}

\usepackage[unicode=true,%
pdfauthor={R.H. Moulton, H.L. Viktor, N. Japkowicz, and J. Gama},%
pdftitle={Contextual One-Class Classification in Data Streams},%
pdfkeywords={Data streams, Supervised learning, One-class classification, Anomaly detection, Context},%
hidelinks]{hyperref}

\usepackage{graphicx}
\usepackage{natbib}
\usepackage[strings]{underscore}
\usepackage{authblk}

\newtheorem{theorem}{Theorem}
\newtheorem{lemma}{Lemma}

\theoremstyle{definition}
\newtheorem{definition}{Definition}

\DeclareMathOperator{\dom}{dom}
\DeclareMathOperator*{\argmin}{arg\,min}

\begin{document}
\title{Contextual One-Class Classification in Data Streams}

\author[1]{Richard Hugh Moulton}
\author[2]{Herna L.\ Viktor}
\author[3]{Nathalie Japkowicz}
\author[4]{Jo\~{a}o Gama}

\affil[1]{\small Department of Electrical and Computer Engineering, Queen's University}
\affil[2]{\small School of Electrical Engineering and Computer Science, University of Ottawa}
\affil[3]{\small Department of Computer Science, American University}
\affil[4]{\small LIAAD -- INESC TEC and Faculty of Economics, University of Porto}

\date{6 February 2019}

\maketitle

\begin{abstract}%
In machine learning, the one-class classification problem occurs when training instances are only available from one class. It has been observed that making use of this class's structure, or its different contexts, may improve one-class classifier performance. Although this observation has been demonstrated for static data, a rigorous application of the idea within the data stream environment is lacking. To address this gap, we propose the use of context to guide one-class classifier learning in data streams, paying particular attention to the challenges presented by the dynamic learning environment. We present three frameworks that learn contexts and conduct experiments with synthetic and benchmark data streams. We conclude that the paradigm of contexts in data streams can be used to improve the performance of streaming one-class classifiers.
\end{abstract}

\section*{Keywords}
Data streams, Supervised learning, One-class classification, Anomaly detection, Context

\section{Introduction}
One-class classification (OCC), the extreme form of the class imbalance problem, is a well known task in machine learning. OCC allows learning to occur when training instances are only available from one class, which changes the classification task from one of discrimination to one of recognition. Several aspects of OCC have been studied for static data sets, relatively less work has been done regarding the task in the data stream environment. This is not simply a question of theoretical interest. Modern machine learning applications commonly require OCC in data streams, these include network intrusion detection real-time airborne sensors and scientific instruments that record measurements at a high rate. We are specifically interested in data streams where concept drift may be present, resulting in a challenging environment for stream learning algorithms.

One approach that has been successfully applied to OCC in static data sets is ``divide and conquer.'' Recognizing that the majority class may not be the ideal way of framing the problem, several authors have proposed methods of decomposing the problem of recognizing the majority class into more tractable sub-problems. This has largely been done for static data sets, however, while the work in the data stream environment has concentrated on either semi-supervised learning or novelty detection where ground truth for multiple classes is available. No specific work has been done regarding how to effectively decompose a data stream's majority class in order to perform OCC.

In this paper we use the mathematical formulation of context given by \citet{Turney1993a} in order to decompose the task of learning a data stream's complex majority class on the basis of contextual information. This approach is similar to that used by \citet{Sharma2018} where majority class structure was used to improve one-class classifier performance in three different scenarios for static data sets. While it is intuitive that this approach should be equally applicable for data streams, this hypothesis has not been tested.

Our research question is ``how can contextual knowledge be used to improve one-class classifier performance in data streams?'' We adapt both Turney's idea of context and Sharma et al.'s idea of majority class structure for use with data streams, with particular attention paid to the new challenges presented by this learning environment, resulting in three new frameworks. Intermediate results include a theoretical proof regarding the minimum window size required by these frameworks and a novel cluster distance function. We then demonstrate that our three frameworks can be used to improve one-class classifier performance for both synthetic and benchmark data streams.

\section{Background}
We begin with a review of challenges specific to the OCC problem and the data stream environment. The intersection of these two topics is explored and we highlight the particular difficulties faced by streaming one-class classifiers. We then consider ``context'' in machine learning and a method of formalizing this intuitive notion.

\subsection{One-Class Classification}
OCC is closely related to the detection of anomalies, outliers and novelties; the terminology used depends on the semantic significance placed on the results by the user~\citep{Tax2001,Chandola2009}. In this paper we focus on OCC: the extreme form of class imbalance where training instances are only available from the majority class and no information is available regarding the minority class(es). The significance of this is that instead of learning to discriminate between classes, as in binary or multi-class classification, the classifier must learn to recognize the majority class (Figure~\ref{fig:discriminationVSrecognition}).
\begin{figure}[htb]
\centering
\caption{Two approaches to binary classification~\citep[from][]{Japkowicz2001}}
\label{fig:discriminationVSrecognition}
\begin{subfigure}{0.35\textwidth}
\includegraphics[width=\textwidth]{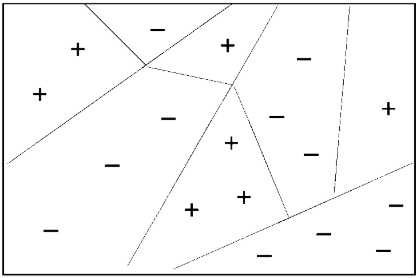}
\subcaption{Discrimination}
\end{subfigure}
\begin{subfigure}{0.35\textwidth}
\includegraphics[width=\textwidth]{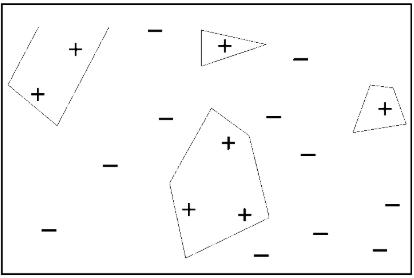}
\subcaption{Recognition}
\end{subfigure}
\end{figure}

Two challenges posed by imbalanced data sets are the absolute size of the training sets and domain complexity~\citep{Japkowicz2002}. These challenges were investigated by \citet{Bellinger2017}, who considered the impact of class imbalance on the performance of both binary and one-class classifiers. Results of their experiments showed that binary classifier performance decreases as class imbalance increases, and that more complex data distributions lead to sharper decreases, while one-class classifier performance stays the same. They also observed that sampling improved binary classifier performance in both scenarios~\citep{Bellinger2017}.

\subsubsection{Training Set Size}
\citet{Domingos2012} observed that for machine learning, ``more data beats a cleverer algorithm''. In OCC, the need for more data is embodied by small disjuncts: portions of the majority class that are only represented by rare cases. These may cause the classifier to learn too tight a boundary around this disjunct or to disregard it altogether. \citet{Jo2004} argued that small disjuncts are, in fact, an underlying source of difficulty in the class imbalance problem and recommended dealing with the problems simultaneously.


One way to correct for absolute rarity is oversampling, as seen in~\citet{Bellinger2017}. Oversampling by na\"{\i}ve replication, however, readily leads to overfitting because this encourages classifiers to learn boundaries around very specific areas of the feature space instead of around the likely area occupied by the minority class~\citep{Chawla2002}. \citet{Chawla2002} introduced the synthetic minority over-sampling technique algorithm (SMOTE) to avoid this by creating synthetic instances whose attributes each lie on the line segment between two existing training instances.

\subsubsection{Complex Domains} \label{sec:complexDomains}
Complex domains are seen throughout machine learning, motivating techniques such as feature engineering to make data sets more suitable for learning~\citep{Domingos2012}. A common type of complex domain is class overlap, where the ground truth classes themselves are not cleanly distinguishable. In our case, the training examples for the OCC problem are all from one class. This does not mean, however, that this class is easily described or that all its constituent instances are generated by the same process.


One way to address complex domains is to use sub-divisions to improve resampling. \citet{Weiss2013} observed that many resampling techniques address imbalance between classes and not within classes; he suggested using resampling to balance the size of majority class disjuncts in OCC training data. This was demonstrated empirically by \citet{Nickerson2001}, who used subcomponent structure to guide resampling and improve oversampling quality, as well as \citet{Jo2004}, who showed that cluster-based oversampling could be used successfully to increase the number of training examples from small disjuncts.


As a way of using context to address complex domains, \citet{Sharma2018} represented the majority class in OCC problems with its internal structure. The authors identified three scenarios (Table~\ref{tab:subConceptScenarios}) and showed experimentally that using the majority class's structure improved the performance of both autoencoders and OCSVMs. Two limitations of these experiments are that they used k-means as the sole clustering algorithm without specific justification and that they only considered static data sets.

\begin{table}[htb]
\centering
\caption{The scenarios of knowing majority class structure~\citep[adapted from][]{Sharma2018}}
\label{tab:subConceptScenarios}
\begin{tabular}{c p{8cm}}
\toprule
Scenario & Description \\
 \midrule
Complete Knowledge &  The structure is known and it is available for both training and testing.\\[0.1cm]
Fuzzy Knowledge & The structure is known, but only available during training. \\[0.1cm]
No Knowledge &  A structure exists, but is unknown.\\
\bottomrule
\end{tabular}
\end{table}

\subsection{Data Streams}
In contrast with static data sets, data streams are characterized by three \emph{V}s. Stream learning algorithms must compress potentially infinite data into a finite model (\emph{volume}), have limited processing time (\emph{velocity}), and must appropriately forget information (\emph{volatility})~\citep{Krempl2014}.

\citet{Webb2016} described data streams as data sets with a temporal aspect. Viewed this way, a data stream is generated by some underlying process that can be modelled as a random variable and the data stream's instances are objects drawn from this random variable.

An object, $o$, is a pair $\langle x,y \rangle$ where $x$ is the object's feature vector and $y$ is the object's class label. Each is drawn from a different random variable, $X$ and $Y$: $x \in \text{dom}(X)$, $y \in \text{dom}(Y)$ and $o \in \text{dom}(X,Y) = \text{dom}(\chi)$~\citep{Webb2016}. A concept in a data stream is therefore defined as the probability distribution associated with an underlying generative process, as in Definition~\ref{def:conceptQuantitativeDefinition}~\citep{Gama2014,Webb2016}.

\begin{definition}[Concept] \label{def:conceptQuantitativeDefinition}
 $Concept = P(X,Y) = P(\chi)$
 \end{definition}

 

It is possible that the data stream's concept is not the same at times $t$ and $u$, this is called \emph{concept drift} (Definition~\ref{def:conceptDriftDefinition}). Concept drift is an important topic, accounting for data stream volatility, and has motivated a lot of study in its own right.

 \begin{definition}[Concept Drift] \label{def:conceptDriftDefinition}
$P_t(X,Y) \neq P_u(X,Y)$
\end{definition}


\subsection{One-Class Classification in Data Streams}
OCC is likely to be performed in data streams where normal behaviour constitutes the vast majority of instances, while the instances of interests -- i.e.\ anomalies or outliers -- occur infrequently. Scenarios could include real-time analysis of sensor data, screening for medical conditions, or detecting computer network intrusions. The difficulties of OCC in static data sets must  be overcome in the data stream environment as well and can even be made more challenging by it. We review these challenges, taking into account both the nature of the data stream and the nature of stream learning algorithms.

\subsubsection{Challenges Related to the Data}
Complex domains are made more complex by the volatility of data streams. Consider contexts which result in small disjuncts and which were vulnerable to absolute rarity in static data sets. In data streams, while most windows of instances may contain a representative sample of instances, it is possible that some windows will not. In fact, given enough windows (a reasonable assumption due to the data stream's volume) this is guaranteed to occur, even for some disjuncts that are generally not small or rare.

One way to address this is to store instances in context-based buffers. Although tempting, \citet{Chen2013} highlighted that this can fall afoul of volatility: using old instances to learn small disjuncts might be inappropriate if concept drift has occurred. This problem also affects any buffers being used for oversampling techniques. 

Another way to address small disjuncts is by selecting an adequate window size: the more instances there are in the window the more instances there are likely to be in any disjunct. We provide an in-depth discussion of how to do so in Section~\ref{sec:obsWindowSize}.

Finally, many strategies have been proposed for dealing with concept drift in a data stream as a whole. These can be broken into three kinds: always learning a new model over the most recent batch of instances, which likely leads to unnecessary forgetting; learning a new model only when signalled by a concept drift detection method, which requires explicit concept drift detection; or continuously update the model, which requires the learner to be capable of incremental updating~\citep{Chen2013,Gama2014}.

\subsubsection{Challenges Related to the Learner}
For all data stream classification, the possibility of concept drift means that a stream classifier's model must be adaptable. Although the performance of an algorithm is certainly an important consideration, a more fundamental limiting factor is that not all of a data stream's instance can be maintained.  These requirements have led to the development of stream learners that are \emph{incremental} and \emph{online}. \citet{Losing2018} define incremental learners as those that, given a data stream $\{x^1,x^2,\dots,x^N\}$, produce a sequence of models $h_1,h_2,\dots,h_N$ where model $h_i$ depends solely upon model $h_{i-1}$ and a strictly limited number of recent training objects.

More exactingly, online learners are incremental learners that are also restricted in terms of model complexity and run-time. An online learner can learn forever while consuming only limited resources, respecting data stream velocity. An incremental nature ensures that online learners can adapt their models without retraining from scratch and has the added advantage of enabling passive adaptation to concept drift~\citep{Losing2018}.

Finally, learners must be trained before testing: decision trees must be grown; neural networks must have their weights converge; and nearest neighbour approaches need a neighbourhood. This training must be accounted for at the beginning of any data stream and can be achieved by initializing the one-class classifiers with instances that are used only for training and not for testing.

\subsection{Streaming One-Class Classifiers} \label{sec:oneClassClassifiers}
Classifiers designed for OCC are called one-class classifiers; \citet{Tax2001} grouped these into three approaches: density-based; reconstruction-based; and boundary-based. This taxonomy also applies to streaming one-class classifiers~\citep[pg.~38-41]{Moulton2018b} and we make use of it here.

\subsubsection{Density-based} These classifiers define the majority class according to its probability-density. In practice this can be done using any probability density function, but the normal distribution (either singly or in a mixture model) is most commonly used~\citep[pg.~64-66]{Tax2001}. 

These methods model the majority class for the entire feature space; because of the nature of data streams, this model is often a tree structure. The upside of this approach is that an accurate model provides very complete knowledge. One drawback of probability density functions is that they may result in expressions that are difficult to evaluate or analyze.

\paragraph{One Class Very Fast Decision Tree}
\citet{Li2009} noted that little work had been done regarding the OCC problem for data streams and developed the OcVFDT as a one-class adaptation of the Very Fast Decision Tree (VFDT) algorithm.

Different OcVFDTs are grown for a variety of levels of class imbalance in the stream, one of which is then selected using a set of validation instances. For test instances, the OcVFDT's leaf nodes are labelled with the label of the majority of instances that have traversed there. Experimental results showed that OcVFDT could nearly match the accuracy and $F1$ scores of a VFDT classifier, even with up to 80\% of the data stream unlabelled~\citep{Li2009}.

\paragraph{Streaming Half-Space Trees}
\citet{Tan2011} introduced Streaming HS-Trees, which is similar to a random forest. Streaming HS-Trees detects anomalies using the relative mass of the leaves in its trees. These masses are updated blindly after every window, providing passive adaptation to concept drift. Uniquely, these trees are not induced based on instances but instead on random perturbations of the data space itself: each node randomly selects an attribute and grows two children leafs to represent either half of the attribute (Fig.~\ref{fig:hstreePartition})~\citep{Tan2011}.

\begin{figure}[htb]
\centering
\caption{A Streaming HS-Tree's partition of the data space (from~\citet{Tan2011})}
\label{fig:hstreePartition}
\includegraphics[width=0.5\textwidth]{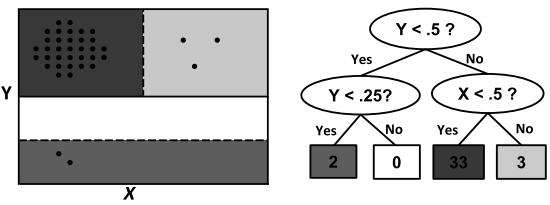}
\end{figure}

A test instance's anomaly score is generated by having each tree send the instance through its structure until a terminal node, $Node^*$, is reached. That node's mass is calculated as a combination of depth, $Node^*.k$ and the number of instances present, $Node^*.r$, \eqref{eq:terminalNodeMass}. These masses are then summed to provide the instance's final anomaly score \eqref{eq:totalAnomalyScore}~\citep{Tan2011}. Note that normal instances will have higher scores than outliers or anomalies.

\begin{equation}  \label{eq:terminalNodeMass}
AnomalyScore_T(x) = Node^*.r \times 2^{Node^*.k}
\end{equation}
\begin{equation} \label{eq:totalAnomalyScore}
AnomalyScore_{Total}(x) = \sum_{T \in HS-Trees} AnomalyScore_T(x)
\end{equation}

\subsubsection{Reconstruction-based} Other algorithms take a compressed representation of the majority class and check whether it recognizes the test instance~\citep[pg.~72-73]{Tax2001}. These algorithms have a model that is independent of training set size, making them a natural option for stream learning.

Reconstruction-based methods have the strengths that they produce a compact model of the majority class and that, similar to density-based methods' probabilistic scores, their reconstruction error usefully approximates the likelihood that an instance belongs to the majority class. Drawbacks are that the model may be difficult to interpret and that a representative training set is required.

\paragraph{Streaming Autoencoders}
Autoencoders are easily applied to data streams as they incrementally adjust their parameters, which led to the streaming autoencoders (SA) described by \citet{Dong2018}.

An SA is initialized with multiple epochs of training over an initial window of instances to reduce the effect of random starting weights. Once initialized, the SA is able to receive test instances while continuing to update its weights incrementally~\citep{Dong2018}. This is an advantage in data streams and permits passive adaptation to concept drift. \citet{Dong2018} found that threaded ensembles of SAs outperformed VFDTs in terms of area under the curve (AUC) on benchmark data streams and that they had faster runtimes than both the VFDTs and Streaming Multilayer Perceptrons.

\paragraph{DETECTNOD}
\citet{Hayat2010} proposed using the Direct Cosine Transform (DCT) as the basis for detecting novelties and concept drifts in data streams. DETECTNOD compresses the data stream into a few DCT coefficients and then uses these as a model for the stream's ``normal'' behaviour. A similar process occurs throughout the stream and the distance between the new DCT coefficients and those in memory is computed using Equation~\ref{eq:dctUnknownScore}~\citep{Hayat2010}.

\begin{equation} \label{eq:dctUnknownScore}
UnExSco(x) = \sqrt{\sum_{i = 1}^{NumDCT} (GenerativeModel_i - CurrentModel_i)^2}
\end{equation}

\subsubsection{Boundary-based} These methods offer an intuitive way of modelling and visualizing the majority class. Because boundaries are decided based on local factors only, these methods are capable of producing quality models from limited training data. A major weakness is that boundary tests inherently produce binary scores instead of more informative probabilistic scores~\citep[pg.~67-78]{Tax2001}.

\paragraph{Nearest Neighbour Data Description} \label{sec:nndDefinition}
Lazy learners are often used with static data sets because they take no time to train and instead compare instances at classification time only~\citep[pg.~423]{Han2011}. This seems to be immediately applicable to data streams and has been exploited in the literature (e.g.\ \citet{Losing2018}).

Inspired by this, we adapted Tax's Nearest Neighbour Data Description (NN-d) method \citep[pg.~69-71]{Tax2001} to data streams. This required only one change: tracking a dynamic neighbourhood instead of a static one. This neighbourhood is initialized as a first in-first out (FIFO) list of fixed length and filled with future instances only if the classifier believes they are from the majority class. The decision function, \eqref{eq:nndDecisionRule}, compares the distance between the test instance to its nearest neighbour with the distance between that neighbour and its nearest neighbour.

\begin{equation}\label{eq:nndDecisionRule}
f(x) = I\left(\frac{\|x - NN^{tr}(x)\|}{\|NN^{tr}(x) - NN^{tr}(NN^{tr}(x))\|} \leq \tau \right)
\end{equation}

\paragraph{Incremental Weighted One Class Support Vector Machine}
Although One Class Support Vector Machines (OCSVMs) are not inherently incremental -- the hyper-plane and support vectors are learned all at once -- modified versions do exist, e.g.\ Weighted OCSVM~\citep{Krawczyk2013}. At the heart of this method is an incremental calculation of support vectors that weights newer or more outlying instances more heavily. Although \citet{Krawczyk2013} noted that performance seemed to be data set-specific, they did observe that weighting training instances based on their age and gradually decreasing them did produce the best results.

\subsubsection{Summary}
The taxonomy of density-, reconstruction-, and boundary-based methods, introduced by \citet{Tax2001} for OCC in static data sets is useful for describing the three approaches used by streaming one-class classifiers, each of which has different strengths and weaknesses.

Density-based methods use mathematical techniques to produce inherently probabilistic scores that are useful for understanding the classifier's confidence in an instance's label. Drawbacks include needing training data that present an unskewed sample of the entire feature space and that the resulting model may be difficult to interpret. Another is that the tree structures often used to model the majority class in data streams can be problematic: tree complexity may scale with the number of instances, working against the limited resources available to stream learners; and relearning decision trees after concept drift may be computationally expensive \citep{Dong2018}.

Reconstruction-based methods produce compact models whose size is independent of the training set's size. Some models can be updated incrementally, e.g.\ SAs, which is useful for stream learning. These methods commonly classify instances via reconstruction error, which acts like the probabilistic score produced by density-based methods. Drawbacks include difficulty of interpretation (e.g.\ the weights of a SA) and the requirement that the training data represent the whole feature space.

Boundary-based methods come with different characteristics, models can be built using any available data and are easily visualized. Unfortunately, model size does depend on the training set size and complexity and the model produces binary scores. Individual classifiers do vary, however: NN-d trains quickly but has a model whose size is directly related to the size of its training set; by contrast the Weighted OCSVM trains slowly but has a model consisting solely of support vectors.

\subsection{Divide and Conquer in Data Streams} \label{sec:complexDomainReviewDS}
Constructing ensembles of stream learners makes use of the idea that a problem can be usefully divided into more easily solved sub-problems. \citet{Gomes2017a} noted that combining several weak classifiers into a strong ensemble can be easier than learning an equally strong classifier. \citet{Krawczyk2017} highlighted that a strength of ensembles is that they allow different classifiers to concentrate on performing well in their own areas.

Divide and conquer has been successfully applied to a range of stream learning tasks in recent years including semi-supervised learning \citep{Hosseini2016,Al-Jarrah2018}, active learning \citep{Abdallah2016} and novelty detection \citep{Faria2015}). None, however, has specifically addressed the problem of OCC in data streams.

\subsection{Context in Machine Learning}
The idea of context in a machine learning task is intuitive and there is a consensus that it is useful for determining what assumptions a learner can make~\citep[e.g.][]{Brezillon1999,Dey2001}. Systems have been successfully designed to track context in data streams \citep{Gomes2012}, identify recurrent concepts \citep{Gama2014a} and perform fault detection \citep{Kalish2016}.

In an early paper in the field, \citet{Turney1993a} formalized the idea of context mathematically. As in the formalism for data streams, every instance is a pair $<x,y>$ where $x = (x_1,x_2,\dots,x_n)$ is the feature vector and $y$ is the class label. Using these mathematical objects, \citet{Turney1993a} described features based on their utility for making predictions about an instance's class label:
\begin{description}
\item[Primary Feature] A feature $x_i$ is a primary feature for predicting a class value $y_j$ if there exists a value $a_i$ for $x_i$ such that
\begin{equation}
P(y = y_j | x_i = a_i) \neq P(y=y_j).
\end{equation}

\item[Contextual Feature] A feature $x_i$ is a contextual feature for predicting a class value $y_j$ if $x_i$ is not a primary feature and there exists a value $a$ for the whole vector $x$ such that
\begin{align}
P(y=y_j | x = a) \neq P(y = y_j | &x_1 = a_1, x_2=a_2,\dots,\\
&x_{i-1}=a_{i-1},x_{i+1}=a_{i+1},\dots,x_n=a_n)\nonumber
\end{align}

\item[Context-sensitive Feature] The primary feature $x_i$ is a context-sensitive feature with respect to the contextual feature $x_j$ if there exists a class value $y_k$ and values $a_i$ and $a_j$ for features $x_i$ and $x_j$ respectively such that:
\begin{equation}
P(y=y_k | x_i=a_i, x_j = a_j) \neq P(y=y_k | x_i=a_i)
\end{equation}

\item[Irrelevant Feature] A feature $x_i$ is an irrelevant feature if it is neither a primary feature nor a contextual feature. 
\end{description}

The usefulness of this formalization is that once contextual attributes have been identified, they can be used to implement one of five different strategies  \citep[Table \ref{tab:fiveConceptStrategies},][]{Turney1993a}. In later work, \citet{Turney2002} also considered the possibility that context could be implicit. In this case the context must be recovered before it can be used; Turney suggested two methods of doing so: clustering the data or dividing the data into temporal sequences~\citep{Turney2002}

\begin{table}[htb]
\centering
\caption{The strategies for using context~\citep[adapted from][]{Turney1993a}}
\label{tab:fiveConceptStrategies}
\begin{tabular}{p{2.5cm} p{9cm}}
\toprule
Feature Space Strategies & Description \\
 \midrule
Contextual normalization & Normalizes those features that are sensitive to context.\\[0.2cm]
Contextual expansion & Adds additional, contextual, features for a classifier to learn from. \\[0.2cm]
Contextual weighting &  Weights features in a context-dependent manner.\\
\bottomrule
\toprule
Classifier Strategies & Description \\
 \midrule
 Contextual classifier selection & Learns a different classifier for each context and selects the most appropriate one at test time.\\[0.2cm]
 Contextual classification adjustment & Makes context-dependent adjustments to the prediction of a single model.\\
 \bottomrule
\end{tabular}
\end{table}

\subsection{Summary}
We have reviewed the basic topics of OCC and data streams. The former, the extreme case of class imbalance, was established to have challenges that included small training set sizes and complex domains. The latter, what \citet{Webb2016} called datasets with a temporal aspect, have their own challenges as well, summarized by the three \emph{V}s: volume, velocity and volatility.

Of particular interest for this work is the intersection of these two topics. In reviewing OCC in data streams, we noted that the data stream environment exacerbates challenges related to both the data and the learner. There are a number of streaming one-class classifiers in the literature, however, that propose to address these challenges. As is the case for one-class classifiers for static datasets, we noted that the taxonomy of density-, reconstruction-, and boundary-based methods did a good job of capturing the characteristics of streaming one-class classifiers.

Finally, we looked at the ideas of ``divide and conquer'' and the utility of ensembles. Inspired by the approach taken by \citet{Sharma2018} in decomposing the majority class, we identified Turney's formalization of context as a promising avenue for applying these ideas. In the next section we lay out our proposed frameworks and provide theoretical supports for some of our design decisions.

\section{Proposed Frameworks} \label{sec:proposedFrameworks}
The central contribution of this paper is to demonstrate how one-class classifier performance can be improved by incorporating contextual information about the majority class's structure. In Table \ref{tab:threeScenariosDS} we identify three scenarios in the data stream environment, based on \citet{Sharma2018}.

\begin{table}[htb]
\centering
\caption{The scenarios of context in data streams}
\label{tab:threeScenariosDS}
\begin{tabular}{c c c c}
\toprule
 & & \multicolumn{2}{c}{Context Availability}\\
Scenario Name & Context Type & Training Phase & Testing Phase \\
 \midrule
Complete Knowledge & Explicit & Yes & Yes \\[0.1cm]
Fuzzy Knowledge & Explicit & Yes & No\\[0.1cm]
No Knowledge &  Implicit & No & No\\
\bottomrule
\end{tabular}
\end{table}

\subsection{Problem Formalisms}
We name the probabilities associated with a data stream's underlying process its \emph{concept}, in line with \citet{Webb2016}, and we name each constituent component that makes up this process a \emph{context}. Once we have an instances' context, either because it was provided explicitly or implicit and recovered, it is used to perform contextual classifier selection.

Drawing on \citet{Turney1993a} and \citet{Webb2016}, we define a data stream object $o$ as a triple $<c, x, y>$ where the context $c$ is drawn from a random variable $C$, the predictive feature vector $x$ is drawn from a random variable $X$ and the class label $y$ is drawn from a random variable $Y$. The distribution of $X$ is conditioned on the value of $c$ and the distribution of $Y$ is conditioned on the values of both $c$ and $x$.




\subsection{Scenario 1: Context is Explicit or Easy to Infer} \label{sec:completeKnowledge}
In this scenario the context is either explicit and available as a contextual attribute, or it is implicit and readily available from an ``oracle'' function. An example is streaming data from airborne sensors: location data providing context for sensor readings is available during training as well as in real-time for test instances that arriving in an online manner.

\subsubsection{OCComplete Initialization Phase}
The OCComplete framework has access to $<c, x, y>$ for the offline initialization phase and one model is trained for each context -- value of $c$:
\begin{equation} \label{eq:model}
M_i: \dom{(X)} \to \dom{(Y)}\quad \forall\ c_i \in \dom{(C)}
\end{equation}
During initialization, the first $initialPoints$ instances from the data stream are buffered by context and then used to train the base classifiers. Of course, sufficient examples may not be available for each of the contexts. In this case, context-based oversampling can be performed before training the one-class classifiers (Algorithm~\ref{alg:ocCompleteInitialize}, Step~\ref{alg:smoteSTEP}). We use SMOTE, which is a well-known method for oversampling that has been shown to increase the number of minority class instances available for training without leading to over-training or inappropriately expanding the minority class's region~\citep{Chawla2002}. Further discussion is provided in Section~\ref{sec:smoteDiscussion}.

With this in mind, OCComplete requires the following parameters: $initialPoints$, the number of training instances to use during initialization; $minPoints$, the minimum number of training instances required for each context; $j$, the number of contexts present -- determined by domain knowledge or by inspection; $\tau$, the threshold anomaly score to label a test instance as an outlier; and $Classifier$, the base classifier to use.
\begin{algorithm}[htb]
\caption{OCComplete - Initialization Phase}
\label{alg:ocCompleteInitialize}
\begin{algorithmic}[5]
\STATE{Initialize $j$ instance buffers: $\mathrm{Buffer}_{1\dots j}$}
\STATE{$numInstances \leftarrow 0$}
\WHILE {($\mathcal{DS}$ has more instances) $\wedge (numInstances < initialPoints)$}
\STATE{Get next instance from $\mathcal{DS}$, $inst$, and add it to its context's buffer}
\STATE {$numInstances++$}
\ENDWHILE
\FORALL{$c \in [1,j]$}
\WHILE{$|\mathrm{Buffer}_c| < minPoints$}
\STATE{Generate a synthetic instance using SMOTE and add it to $\mathrm{Buffer}_c$} \label{alg:smoteSTEP}
\ENDWHILE
\STATE {Train a $Classifier$ on all of the instances in $\mathrm{Buffer}_c$}
\ENDFOR
\RETURN {$j$ $\mathrm{Classifier}$s, each trained on a different context}
\end{algorithmic}
\end{algorithm}

\subsubsection{OCComplete Online Phase}
OCComplete receives $<c,x>$ during the online phase and uses $c$ to perform contextual classifier selection. In this way, irrelevant or confusing information from other contexts can be ignored while deciding the nature of the test instance.
\begin{algorithm}[htb]
\caption{OCComplete - Online Phase}
\label{alg:ocCompleteOnline}
\begin{algorithmic}[5]
\WHILE {($\mathcal{DS}$ has more instances)}
\STATE{Get next instance from $\mathcal{DS}$, $inst$, and determine which context it belongs to, $c$}
\STATE {$AnomalyScore \leftarrow$ the anomaly score returned by $Classifier_c$ for $inst$}
\IF {$AnomalyScore > \tau$}
\STATE{Label $inst$ as an \textbf{OUTLIER}}
\ELSE
\STATE {Label $inst$ as \textbf{NORMAL}}
\ENDIF
\IF{Conditions for training are met} \label{alg:trainingConditionSTEP}
\STATE {Train $\mathrm{Classifier}_c$ on $inst$}
\ENDIF
\ENDWHILE
\end{algorithmic}
\end{algorithm}

Throughout, the base classifiers should be kept up to date in order to cope with concept drift. This is done by training each base classifier on instances from the data stream which belong to its respective context. In the case where a base classifier can only be trained on majority class instances (e.g.\ Nearest Neighbour Description) this training can occur when a label is received or on the basis of the instance's anomaly score (Algorithm~\ref{alg:ocCompleteOnline}, Step~\ref{alg:trainingConditionSTEP}).

\subsection{Scenario 2: Context is Hard to Infer} \label{sec:fuzzyknowledge}
In this scenario the context is implicit and obtainable from an ``oracle'' function. This oracle is impractical to use during the online phase, however, so the framework must learn its function as well. This is done by training a multi-class classifier to distinguish between contexts during initialization. This framework uses the temporal sequence of the instances in the data stream to recover context: the models \eqref{eq:model} and \eqref{eq:contextSelector} are updated on the basis of windows, so instances that are close together in the data stream are treated the same way. An example of this scenario is diagnosing a medical condition: a medical professional can determine context membership, but this is an expensive process and can only be done during initialization.

\subsubsection{OCFuzzy Initialization Phase}
OCFuzzy has access to $<c, x, y>$ for the offline initialization phase, with $c$ provided by an oracle, and trains one model \eqref{eq:model} for each context. It also trains another model to predict $c$ on the basis of $x$:
\begin{equation} \label{eq:contextSelector}
\mathcal{M}: \dom{(X)} \to \dom{(C)}
\end{equation}
The parameters $initialPoints$, $minPoints$, $j$, $\tau$, and $Classifier$ are required, as before. An additional parameter is $ConceptDecider$, the multi-class classifier that will learn the model \ref{eq:contextSelector} and assign test instances to a context during the online phase.
\begin{algorithm}[htb]
\caption{OCFuzzy - Initialization Phase}
\label{alg:ocFuzzyInitialize}
\begin{algorithmic}[5]
\STATE{Initialize $j$ instance buffers: $\mathrm{Buffer}_{1\dots j}$}
\STATE{$numInstances \leftarrow 0$}
\WHILE {$\mathcal{DS}$ has more instances $\wedge$ $numInstances < initialPoints$}
\STATE{Get next instance from $\mathcal{DS}$, $inst$, and add it to its context's buffer}
\STATE {$numInstances++$}
\ENDWHILE
\FORALL{$c \in [1,j]$}
\WHILE{$|\mathrm{Buffer}_c| < minPoints$}
\STATE{Generate a synthetic instances using SMOTE and add it to $\mathrm{Buffer}_c$}
\ENDWHILE
\STATE {Train a $Classifier$ on all of the instances in $\mathrm{Buffer}_c$}
\ENDFOR
\STATE{Train $\mathrm{ConceptDecider}$ to discriminate between contexts using all of the instances in $Buffer_{1\dots j}$}
\RETURN {$\mathrm{ConceptDecider}$ trained to discriminate between contexts; $j$ $\mathrm{Classifier}$s, each trained on a different context}
\end{algorithmic}
\end{algorithm}

\subsubsection{OCFuzzy Online Phase}
OCFuzzy receives only $<x>$ during the online test-and-train phase as consulting the oracle is too expensive, impractical or simply impossible. Instead the model \eqref{eq:contextSelector} is used to predict the context $$\hat{c} = \mathcal{M}(x).$$ This prediction, $\hat{c}$, is used to for contextual classifier selection. As before, only the chosen base classifier provides an anomaly score -- ensuring that only relevant knowledge is applied.
\begin{algorithm}[htb]
\caption{OCFuzzy - Online Phase}
\label{alg:ocFuzzyOnline}
\begin{algorithmic}[5]
\WHILE {$\mathcal{DS}$ has more instances}
\STATE{Get next instance from $\mathcal{DS}$, $inst$}
\STATE {$\hat{c} \leftarrow \mathrm{ConceptDecider}$'s classification of $inst$}
\STATE {$AnomalyScore \leftarrow $ the anomaly score returned by $Classifier_{\hat{c}}$ for $inst$}
\IF {$AnomalyScore > \tau$}
\STATE{Label $inst$ as an \textbf{OUTLIER}}
\ELSE
\STATE {Label $inst$ as \textbf{NORMAL}}
\ENDIF
\IF{Conditions for one-class training are met} 
\STATE {Train $\mathrm{Classifier}_{\hat{c}}$ on $inst$}
\ENDIF
\IF{Conditions for multi-class training are met} \label{alg:multiTrainSTEP}
\STATE {Train $\mathrm{ConceptDecider}$ on $inst$}
\ENDIF
\ENDWHILE
\end{algorithmic}
\end{algorithm}

While the base classifiers can be trained throughout the data stream as described for scenario~1, the multi-class classifier is more challenging. After all, getting the context for a test instance is an expensive process, which is what motivated this framework in the first place. If concept drift is anticipated, then the multi-class classifier could be trained using active learning or a concept drift detection method could trigger a complete re-initialization of the framework (Algorithm~\ref{alg:ocFuzzyOnline}, Step~\ref{alg:multiTrainSTEP}).

\subsection{Scenario 3: Context must be Recovered} \label{sec:noKnowledge}
In this scenario the context is implicit and no ``oracle'' function is available. The temporal sequence of instances is again used to recover the context, but now unsupervised learning is used as well. This approach has the additional benefit that we will be able to track how these contexts evolve, since data stream clustering algorithms update incrementally and are able to produce clusterings at any time. An example application is network intrusion detection: an intrusion detection system monitoring a computer network would benefit from knowing the context during which a specific packet or system trace occurred, but it is unclear exactly how to best determine these contexts.

\subsubsection{OCCluster Initialization Phase}
The framework only has access to $<x, y>$ for the offline initialization phase. It begins by clustering the data, $\mathcal{C}$, and treats each cluster $\sigma \in \mathcal{C}$ as a separate context. One model \eqref{eq:model} is trained for each cluster.

Any cluster whose weight is less than a given threshold is removed in order to ensure that the framework only learns classifiers over instances that are likely to continue occurring (equations~\ref{eq:clusterWeight}-\ref{eq:clusterWeightThreshold}). On the basis of \citet{Moulton2018a} we use ClusTree with silhouette k-means \citep{Kranen2011} as the data stream clustering algorithm in this paper after it was found to produce robust, high-quality clusterings in the presence of concept drift. While context-based oversampling can be performed as in the previous scenarios, it should only be done after clustering.
\begin{equation} \label{eq:clusterWeight}
\mathrm{weight}(c) = \frac{\text{num. points in }c}{\text{total points clustered}}
\end{equation}
\begin{equation} \label{eq:clusterWeightThreshold}
\mathrm{Threshold} = \frac{1}{\|\mathcal{C}\|^2} \text{ where } \|\mathcal{C}\| \text{ is the number of clusters in } \mathcal{C} \text{.}
\end{equation}

OCCluster's parameters include $initialPoints$, $minPoints$, $\tau$, and $Classifier$ as before. Additional parameters are: $ClusterAlgorithm$, the data stream clustering algorithm with which to cluster the the data stream; $w$, the number of instances between updates of $ClusterAlgorithm$'s clustering; and $T$, the cluster distance threshold above which two clusters are considered to be different.
\begin{algorithm}[htb]
\caption{OCCluster - Initialization Phase}
\label{alg:ocClusterInitialize}
\begin{algorithmic}[5]
\STATE{Initialize one instance buffer: $\mathrm{InstBuffer}$}
\STATE{$numInstances \leftarrow 0$}
\WHILE {($\mathcal{DS}$ has more instances) $\wedge$ ($numInstances < initialPoints$)}
\STATE{Get next instance from $\mathcal{DS}$, $inst$, and add it to $\mathrm{InstBuffer}$}
\STATE{Train $ClusterAlgorithm$ on $inst$}
\STATE {$numInstances++$}
\ENDWHILE
\STATE{$\mathcal{C} \leftarrow$ Clustering produced by $ClusterAlgorithm$}
\FORALL{clusters $\sigma$ in $\mathcal{C}'$}
\IF{$(weight(\sigma) > \frac{1.0}{\|\mathcal{C}\|^2})$} \label{alg:initClusterWeightSTEP}
\STATE{Remove $\sigma$ from $\mathcal{C}$}
\ENDIF
\ENDFOR
\STATE{$k \leftarrow$ the number of clusters in $\mathcal{C}$}
\STATE{Initialize $k$ instance buffers: $\mathrm{Buffer}_{1\dots k}$}
\FORALL{Instances $inst$ in $\mathrm{InstBuffer}$}
\STATE{Determine the cluster,  $\sigma_i$ in $\mathcal{C}$, to which $inst$ is closest}
\STATE{Add $inst$ to $\mathrm{Buffer}_i$}
\ENDFOR
\FORALL{$c \in [1,k]$}
\WHILE{$|\mathrm{Buffer}_c| < minPoints$}
\STATE{Generate a synthetic instance using SMOTE and add it to $\mathrm{Buffer}_c$}
\ENDWHILE
\STATE {Train a $Classifier$ on $\mathrm{Buffer}_c$}
\ENDFOR
\RETURN {A clustering $\mathcal{C}$ with $k$ clusters; $k$ $Classifier$s, each trained on a different context}
\end{algorithmic}
\end{algorithm}

\subsubsection{OCCluster Online Phase}
OCCluster receives only $<x>$ during the online phase. It uses the clustering to predict the context by finding the cluster to which the test instance is closest \eqref{eq:instanceClusterDistance}. This cluster, $\sigma$, is used for contextual classifier selection. The clustering is updated every $w$ instances, where $w$ is chosen to minimize the fluctuations in the number of instances available from each context and to maximize the framework's adaptability to concept drift.
\begin{equation} \label{eq:instanceClusterDistance}
\argmin_{\sigma \in \mathcal{C}} d(x,\sigma)
\end{equation}
\begin{algorithm}[htb]
\caption{OCCluster - Online Phase}
\label{alg:ocClusterOnline}
\begin{algorithmic}[5]
\WHILE {$\mathcal{DS}$ has more instances}
\STATE{Get next instance from $\mathcal{DS}$, $inst$, and determine which cluster $\sigma_i$ in $\mathcal{C}$ it is closest to}
\STATE {$AnomalyScore \leftarrow$ the anomaly score returned by $\mathrm{Classifier}_i$ for $inst$}
\IF {$AnomalyScore > \tau$}
\STATE{Label $inst$ as an \textbf{OUTLIER}}
\ELSE
\STATE {Label $inst$ as \textbf{NORMAL}}
\ENDIF
\IF{Conditions for one-class training are met} 
\STATE {Train $\mathrm{Classifier}_\sigma$ on $inst$}
\ENDIF
\STATE {Train $ClusterAlgorithm$ on $inst$}
\IF{$numInstances\ \mathrm{ mod }\ w = 0$}
\STATE {$\mathcal{C}' \leftarrow$ Clustering produced by $ClusterAlgorithm$}
\FORALL{clusters $\sigma'$ in $\mathcal{C}'$}
\IF{$weight(\sigma') > \frac{1.0}{\|\mathcal{C}\|^2} \wedge (\operatorname*{arg\,min}_{\sigma \in \mathcal{C}} CD(\sigma,\sigma') < T)$} \label{alg:clusterDistanceSTEP}
\STATE{Assign the ${Classifier_\sigma}$ to $\sigma'$.}
\ELSE
\STATE{Train new $Classifier$ over instances belonging to $\sigma'$.}
\ENDIF
\ENDFOR
\ENDIF
\STATE{$numInstances++$}
\ENDWHILE
\end{algorithmic}
\end{algorithm}

Throughout the online phase (Algorithm~\ref{alg:ocClusterOnline}, Step~\ref{alg:clusterDistanceSTEP}), clusters are pruned if their weight is below a threshold to avoid learning classifiers over anomalies or noise in a given window. To allow this threshold to dynamically adapt to clusterings with different numbers of clusters, we use the same formulation as in the initialization phase (Equation~\ref{eq:clusterWeightThreshold}).

Updating the clustering/classifier pairing after a new clustering is acquired is a challenging task. To do so, each of the new clusters, $\sigma'$ in $\mathcal{C}'$, is compared to the old clusters, $\sigma$ in $\mathcal{C}$ using a cluster distance function that we define in Section~\ref{sec:clusterDistanceFunction} (Algorithm~\ref{alg:ocClusterOnline}, Step~\ref{alg:clusterDistanceSTEP}). If the distance from a new cluster to its closest old cluster is below the threshold $T$ then the old cluster's classifier is assigned to the new cluster; otherwise a new classifier is trained over the new cluster. As with the parameter $w$, $T$ must be chosen to minimize the framework's forgetting of relevant information and to maximize the framework's forgetting of irrelevant information. Ideal values for these parameters are likely data stream-/domain-specific and users should take these conditions into account.

\subsection{Context-Based Oversampling} \label{sec:smoteDiscussion}
Although synthesizing new instances is more computationally expensive than replication, we use SMOTE's generation process to over-sample those contexts with an insufficient number of training instances. This is to help the base classifier learn the context's whole area in the feature space rather than overly concentrating on the instances that have been seen.

The idea of context-based oversampling has been seen before in the literature \citep{Nickerson2001,Jo2004,Weiss2013} and is related to the strategy of contextual weighting~\citep{Turney1993a}. Instead of features being weighted, however, instances are ``weighted'' by their use in generating synthetic instances. This has balances the number of instances available between contexts and ensures that a model can be learned for each.

We considered undersampling as well, but noted that a constant challenge during experiments was to ensure that there are enough training instances from each context. Therefore, although we do not specifically undersample, the goal of minimizing unnecessary instances is achieved by selecting the smallest window size that guarantees at least $\tau$ instances from each context; this is discussed in the next section.

\subsection{An Observation Regarding Window Size} \label{sec:obsWindowSize}
Each framework makes use of sliding windows, which are a technique used by stream learning algorithms to achieve the twin objectives of keeping memory requirements bounded and weighting recent instances more heavily than older instances. Although window sizes can be fixed or variable~\citep{Gama2014}, the latter requires a signal to decide what the correct window size is at any given point in the data stream. This would work at cross purposes to our desire that our frameworks learn as passively as possible from the data stream, so we would prefer to use fixed size sliding windows.

This decision naturally raises the question of what fixed window size should be chosen. \citet{Zliobaite2009} addressed a related problem for dynamic windows in small sample size classification. They, however, solved the distinct problem of using labelled instances to find the optimal window size for classifier training after concept drift. In our case, we wish to ensure, to a desired degree of confidence, that there will be enough instances from each context within our windows throughout the data stream and we formalize this approach in Theorem~\ref{the:windowSize}).
\begin{theorem}[Window Size to Ensure a Minimal Number of Instances] \label{the:windowSize}
Consider a data stream, $\mathcal{DS}$, with an underlying concept, $\chi$, in line with Definition~\ref{def:conceptQuantitativeDefinition}. The concept $\chi$ is composed of $j$ contexts, $c_{1\dots j}$, and an instance drawn from $\mathcal{DS}$, $o \in \chi$, belongs to one of these contexts according to their underlying probabilities, $p_i$ for $i \in 1\dots j$:

\begin{equation}
p_i = P(o \in c_i)
\end{equation}

The minimal window size required to ensure, with degree of confidence $\overline{c}$, that at least $\tau$ instances from each context are present is $n$ such that:

\begin{equation}
np - (x \times \sqrt{np_i(1 - p_i)}) = \tau
\end{equation}

where $p_{min} = \operatorname{min}_{i \in 1\dots j} p_i$,  $x$ is chosen such that $\Phi(x) - \Phi(-x) = \overline{c}$, and $n$ and $p_{min}$ are such that a normal approximation of the binomial distribution can be used.
\end{theorem}

\begin{lemma}[Normal approximation of the binomial distribution] \label{lem:binomialToNormal}
The binomial distribution $B(n,p)$ can be approximated by the normal distribution $\mathcal{N}(np, np(1 - p))$ if:

\begin{equation}
n > 9\frac{1-p}{p} \text{ and } n > 9\frac{p}{1-p}
\end{equation}
\end{lemma}

\begin{proof}
Consider a data stream, $\mathcal{DS}$, with an underlying concept, $\chi$, which is composed of a series of $j$ contexts, $c_{1\dots j}$. The probability that a given data stream object belongs to context $c_i$ is $p_i$.

The number of instances from a given context in such a window is a random variable following the binomial distribution: $B(n, p_i)$.  Assuming that our window size is that $n$ is large enough to make use of lemma~\ref{lem:binomialToNormal}, we can use the approximation $$B(n, p_i) \sim \mathcal{N}(np_i, np_i(1 - p_i)).$$ Using the properties of the normal distribution, we then set the lower bound on the number of instances for each context. In general, using the cumulative density function of the normal distribution, $\Phi$, we determine $\tau$ for a specific confidence level, $\overline{c}$, as follows:

\begin{equation}
\tau = np - (x \times \sqrt{np_i(1 - p_i)}) \text{ such that } \Phi(x) - \Phi(-x) = \overline{c}
\end{equation}
\end{proof}

\subsection{Defining a Cluster Distance Function} \label{sec:clusterDistanceFunction}
In the OCCluster framework we want to use knowledge learned from old clusters to bootstrap learning from a new clustering. To do so, we must be able to answer the question `how close are these two clusters?' In order for OCCluster to be as general as possible, our desire is for this distance function to be general as well.

\subsubsection{Cluster Types}
\citet{Ntoutsi2009} identified three kinds of clusters based on their definition: type A, which are defined as geometric objects, type B1, which are defined as a set of data records, and type B2, which are defined as a distribution. They note that all clusters can be given a type B1 definition since they are formed on top of a data set~\citep{Ntoutsi2009}. Given that all clusters result in a (hyper-) volume within the feature space, however, they can also be given a type A definition. Even if this is not an easily defined geometric object, this volume can be defined by the inclusion probability (IP) function. For clusters with certain inclusion (as opposed to fuzzy inclusion) this is simply an indicator function (equation~\ref{eq:IPcluster}).

\begin{equation} \label{eq:IPcluster}
IP_{\mathcal{C}}(x) =
\begin{cases} 
1 & \text{if } x \in \mathcal{C} \\
0 & \text{if } x \notin \mathcal{C}
\end{cases}
\end{equation}

\subsubsection{Distance Functions}
A function to tell us how far apart two clusters are will be, mathematically speaking, a distance function. A more stringent type of distance, which aligns with our experiences with the Euclidean distance, is a metric~\citep[Definition~\ref{def:metric}, from][Ch.1]{Deza2009}.
 
\begin{definition}[Metric] \label{def:metric}
Let $X$ be a set. A function: $d: X \times X \rightarrow \mathbb{R}$ is called a \textbf{metric} on $X$ if, for all $x,y,z \in X$, there holds:
\begin{enumerate}
\item $d(x,y) \geq 0$ \hfill[non-negativity]
\item $d(x,y) = 0 \iff x = y$ \hfill[identity of indiscernibles]
\item $d(x,y) = d(y,x)$ \hfill[symmetry]
\item $d(x,y) \leq d(x,z) + d(z,y)$ \hfill[triangle inequality] 
\end{enumerate}
\end{definition}

Many authors have touched on questions that are tangential to ours. Each of their methods has some aspect that makes it difficult to adapt; they either measure the distance between whole clusterings, measure the distance between probability distributions, require all of the underlying points, or are not symmetric~\citep[pg.~60-66]{Moulton2018b}. We therefore propose our own Cluster Distance Function that can be used in the data stream environment, takes only two clusters as its arguments, and is provably a metric.

\subsubsection{A New Cluster Distance Function} \label{sec:clusterDistance}
We will consider only clusters with certain inclusion for our cluster distance function and leave the question of measuring distance between fuzzy clusters to future inquiry. Instances will be considered to either belong or not belong to a given cluster, as shown by the IP.

Here we consider only numerical features and with the Euclidean distance between points. Analysis for nominal dimensions would depend on the ordering (if any) of values for that dimension as well as the distance function defined between values. This extension is possible if a distance and an IP can be defined, however this is left to future work.

With these assumptions, we express all clusters as (hyper-) volumes of the feature space, each defined by an IP. We then conceive of the distance between two clusters as the amount that these two (hyper-) volumes do not overlap, inspired by the Hellinger distance. This approach has the additional benefit of provably being a metric.

\begin{definition}[Cluster Distance Function]
For a given $n$-dimensional feature space, $\mathbb{F}$, we define the distance between two clusters, the (hyper-) volumes $\mathcal{C} \in \mathbb{F}$ and $\mathcal{C}'\in \mathbb{F}$, as the amount that these two (hyper-) volumes do not overlap:

$$CD(\mathcal{C}, \mathcal{C}') = \int_{\mathbb{F}} |IP_{\mathcal{C}}(x) - IP_{\mathcal{C}'}(x)|$$
\end{definition}

The proof that this Cluster Distance Function is a metric is included in the supplemental material.

\subsubsection{Summary}
The Cluster Distance Function defines the distance between two clusters based on their respective IP functions. As limitations, our proofs have only considered numeric dimensions and we restrict the argument clusters to being certain clusters instead of fuzzy clusters.

We note that the formulation presented is unable to distinguish between multiple cases of disjoint clusters. When considering our desire to transfer classifiers from one cluster to another, however, the precondition of having some overlap between clusters seems reasonable. Therefore, we accept this limitation and leave further development to future work.

\section{Experimental Design}
Our research question is ``how can contextual knowledge be used to improve one-class classifier performance in data streams?'' The hypothesis is that guiding a streaming one-class classifier with the contexts that occur within the majority class will result in better classification results than using the streaming one-class classifier alone.

\subsection{Software and Hardware Specification}
All experiments were done on a laptop with 64-bit Linux Ubuntu 16.04 installed, 15.6 GiB of memory and eight 2.60 GHz processors. All data stream algorithms were implemented in MOA 17.06. MOA is an open source framework with the goal of being a benchmark for data stream mining research; it is implemented in Java and easily extendable~\citep{bifet2010moa}.

\subsection{Framework and Classifier Settings}
In applying the frameworks to these data streams it is important to ensure that their parameters, described in Sections~\ref{sec:completeKnowledge} to \ref{sec:noKnowledge}, are set to values that are reasonable and consistent. Given that the potential parameter-space for this experiment is very large, values were chosen by inspection.

\subsubsection{Single Classifier}
We will use SAs, Streaming HS-Trees, and our streaming adaptation of NN-d as base classifiers. These represent the three approaches to OCC and will allow us to assess the generality of our results. Each classifier is able to passively detect concept drift, removing the requirement for a separate concept drift detection method. This results in simpler frameworks as well as simpler experimental design. Each of the base classifiers\footnote{available online: https://doi.org/10.5281/zenodo.1287732} was implemented for use with MOA 17.06~\citep{bifet2010moa}.

\paragraph{Streaming Autoencoder}
The SA's structure is as described by \citet{Dong2018}: an input layer with one neuron for each non-class attribute, a hidden layer of two neurons and an output layer with one neuron for each non-class attribute. The logistic function is used as the activation function for the neurons. Since an autoencoder attempts to compress and reconstruct the input, the squared error between input and output is used as the anomaly score \eqref{eq:autoencoderAnomalyScore}. The SA's learning rate, which controls the magnitude of weight updates during backpropagation, is set to $0.5$. 
\begin{equation} \label{eq:autoencoderAnomalyScore}
AnomalyScore(x) = \frac{1}{2} \times \|x_{input} - x_{output}\|^2
\end{equation}

\paragraph{Streaming Half-Space Trees}
The Streaming HS-Trees algorithm is implemented as described by \citet{Tan2011} and with parameter values as shown in Table~\ref{tab:hstreesParams}.
\begin{table}[htb]
\centering
\caption{Values chosen for Streaming HS-Trees' parameters}
\label{tab:hstreesParams}
\begin{tabular}{c c c}
\toprule
Symbol & Parameter & Value \\
\midrule
$\psi$ & Size of Window & $500$ \\
t & Number of Trees & $5$ \\
h & Maximum Tree Depth & $12$ \\
- & Size Limit & $0.1$\\
\bottomrule
\end{tabular}
\end{table}

\paragraph{Nearest Neighbour Data Description}
The NN-d algorithm is implemented as described in Section~\ref{sec:nndDefinition} and produces an anomaly score using equation~\ref{eq:nndDecisionRule}. The only parameter for this algorithm is neighbourhood size, which we set at $100$ instances.

\subsubsection{Frameworks}
The main component for each framework is the base classifier to be used, this is one of the independent variables in our experiments. The initialization window is set so that each model receives $2000$ training instances and context-based oversampling is applied using SMOTE to ensure that each context has at least $1000$ instances. Information about the data stream is used as described in Section~\ref{sec:proposedFrameworks} and other parameters are set according to Table~\ref{tab:occlusterParams}.

\begin{table}[htb]
\centering
\caption{Parameter values chosen for each framework}
\label{tab:occlusterParams}
\begin{tabular}{cc c}
\toprule
Framework & Parameter & Value\\
\midrule
OCComplete & Nil & Nil \\
\midrule
OCFuzzy & Concept Decider & Na\"{\i}ve Bayes \\
\midrule
OCCluster & Clustering Algorithm & ClusTree\\
& Window Size & 2000\\
& Inclusion Threshold for Training & 1.0\\
& Cluster Movement Threshold & 0.2\\
\bottomrule
\end{tabular}
\end{table}

\subsection{Data Streams}
In order for the results of this experiment to answer the research question, it is important that the data streams used be both representative and valid. Data streams were synthesized or selected to represent realistic cases of class imbalance with an underlying context structure. Each instance was therefore marked with both a class (majority/minority) and a context (for example Figure~\ref{fig:classesVconcepts}). An instance's context label was derived from either the generator's internal model for synthetic data streams or domain knowledge for the benchmark data streams.

\begin{figure}[htb]
\centering
\caption{Class markup versus context markup}
\begin{subfigure}{0.35\textwidth}
\caption{Class markup}
\includegraphics[width=\textwidth]{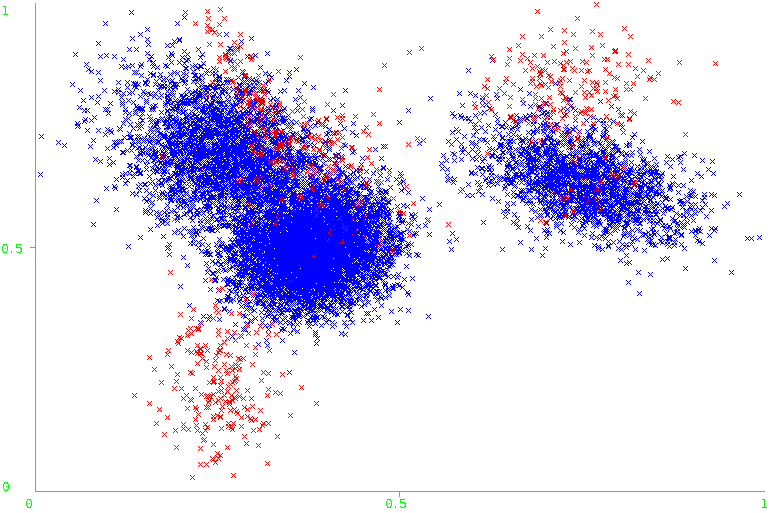}
\end{subfigure}
\begin{subfigure}{0.35\textwidth}
\caption{Context markup}
\includegraphics[width=\textwidth]{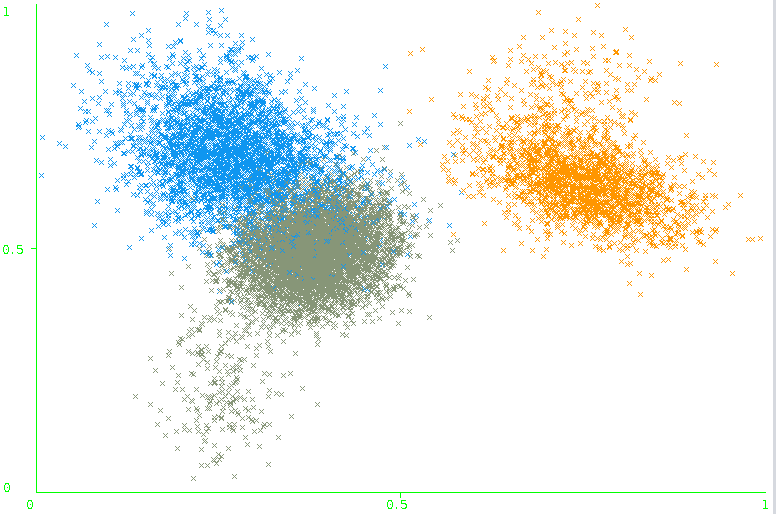}
\end{subfigure}
\label{fig:classesVconcepts}
\end{figure}

\subsection{Evaluation}
We use performance measures based on two classic paradigms for evaluating classifier performance: the confusion matrix and the ROC curve. Our performance measures are selected to convey useful information about a classifier's discriminating ability with a specific emphasis on imbalanced datasets.

\subsubsection{Confusion Matrix}
The confusion matrix is a widely used method of analyzing the performance of a classifier~\citep{Kubat1998} and many simple performance measures that can be derived from this confusion matrix (accuracy, recall, precision, etc.). These performance measures aren't appropriate for imbalanced data sets because they don't account for biases in the user's levels of interest and in the data set itself~\citep{Branco2015}. 


The g-mean \eqref{eq:gmean} is independent of class distribution and its non-linearity scales the cost of misclassification by the number of examples in that class that have been misclassified~\citep{Kubat1998}. \citet{Japkowicz2013} assessed that the g-mean is an appropriate threshold-based measure for assessing classifier performance in imbalanced data sets, noting that it gives gives equal weight to both classes.
\begin{equation} \label{eq:gmean}
g-mean = \sqrt{sensitivity \times specificity}
\end{equation}

\subsubsection{Receiver Operating Characteristic Curve}
The ROC curve has the advantage of illustrating an algorithm's ability to discriminate for all possible threshold values. Calculating the AUC allows ROC curves to be summarized as a single number: $1$ represents the performance of an ideal classifier and $0.5$ represents the performance of a random classifier~\citep{Japkowicz2013}.

\paragraph{Prequential Area Under the Curve}
Prequential AUC makes use of a sliding window to calculate ROC curves for stream learning algorithms~\citep{Brzezinski2017}. The prequential (``test-then-train'') aspect provides as large a test set as possible. The sliding window, where the ROC curve considers only the last number of instances allows classifier performance to be tracked accurately throughout the data stream.

As a result of experiments, \citet{Brzezinski2017} concluded that prequential AUC is ``statistically consistent and comparably discriminant with AUC calculated on stationary data'' and that it performed well on a range of synthetic and real world data streams exhibiting varying imbalances and concept drifts. We therefore use Prequential AUC as a performance measure alongside g-mean. Calculation of the Prequential AUC was done using the AUC package\footnote{available online: https://cran.r-project.org/package=AUC} in R~\citep{Ballings2013}.

\subsubsection{Cross-Validation}
Ten-fold cross validation was also used for all tasks, meaning that ten parallel frameworks were constructed on the data stream. Although all frameworks were tested on all test instances, each framework had one fold of the data stream withheld from its training set throughout the data stream.

\subsection{Statistical Significance Testing}
\citet{Benavoli2016} recommend using Bayesian analysis for statistical significance testing over the frequentist NHST, which incorrectly assumes both that the \emph{p}-value contains sufficient information about how probable the null hypothesis is and that practical significance follows on from statistical significance. Bayesian analysis more naturally answers the central question of interest: ``is method A better than method B?''

Benavoli et al. adopt \citet{Kruschke2011}'s concept that it is possible for differences in performance, $\mu$, to be close enough to the null value as to be equivalent for practical purposes. Mathematically this region of practical equivalence (rope) is an interval $[-r, r]$ centred on zero, as shown in Figure~\ref{fig:exampleROPE}. For classifiers, the interval $[-0.01,0.01]$ is likely appropriate, though this may vary by domain~\citep{Benavoli2016}.

\begin{figure}[htb] \label{fig:exampleROPE}
\centering
\caption{An illustration of the rope for the difference in accuracy between two classifiers~\citep[from][]{Benavoli2016}}
\includegraphics[width=0.5\textwidth]{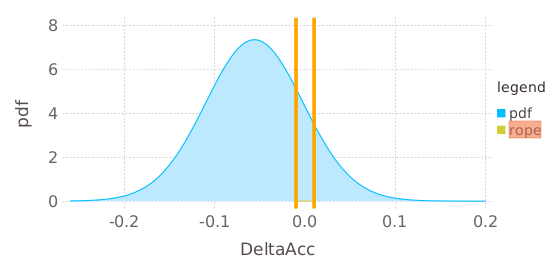}
\label{fig:exampleROPE}
\end{figure}

The Correlated Bayesian t-test (CBTT) analyses cross-validation results for a single data set, accounting for the correlation between data sets and the rope. Benavoli et al.'s implementation of the CBTT\footnote{available in both R and Python at https://github.com/BayesianTestsML/tutorial/} was used~\citep{Benavoli2016}.

\subsection{Summary}
The performance of each framework is evaluated using the g-mean and prequential AUC as measures of performance. The g-means for all data streams is calculated by selecting the average optimal threshold for all of the evaluation windows as determined by Informedness. A single threshold value was calculated in this way, which was deemed to be more realistic than calculating the (potentially different) optimal threshold value for each evaluation windows. The CBTT is used to infer the significance of differences in classifier performance or whether their performance is practically equivalent.

\section{Synthetic Data Streams}
We first test the frameworks on the synthetic data streams in order to to validate our belief that using knowledge of the majority class's contexts will improve classifier performance. We also seek to characterize the performance of each framework.
\begin{table}[htb]
\centering
\caption{Summary of the synthetic data streams}
\label{tab:syntheticDataStreams}
\begin{tabular}{p{2.5cm} c c p{3cm} p{3cm}}
\toprule
Name & Atts. & Context & Majority Class & Minority Class \\
\midrule
Random RBF & 4 & Explicit & Multiple Centroids & Multiple Centroids \\[0.1cm]
Random RBF with Noise  & 4 & Explicit & Multiple Centroids & Multiple Centroids plus uniform noise \\[0.1cm]
Mixture Model & 4 & Explicit & Multiple MVNDs & Multiple MVNDs \\
\bottomrule
\end{tabular}
\end{table}

Three families of synthetic data streams were generated using MOA based on either mixture models of multivariate normal distributions (MVNDs) or random radial basis functions (RBFs). These present a range of conditions for both the majority and minority classes; all three incorporate knowledge of contexts. For each data stream the contexts were explicit and assigned according to the data stream generator's internal model--minority class instances were assigned to the context of the nearest majority class instances.

\subsection{Results and Discussion}
The prequential AUC achieved by each framework throughout the respective data streams is shown in Figures~\ref{fig:mixtureModelResults}-\ref{fig:randomRBFNResults}. Graphs showing the g-mean achieved by each framework are available in Appendix~\ref{app:results}.

\begin{figure}[htbp]
\centering
\caption{Results for the Mixture Model data stream}
\label{fig:mixtureModelResults}
\includegraphics[width=0.32\textwidth]{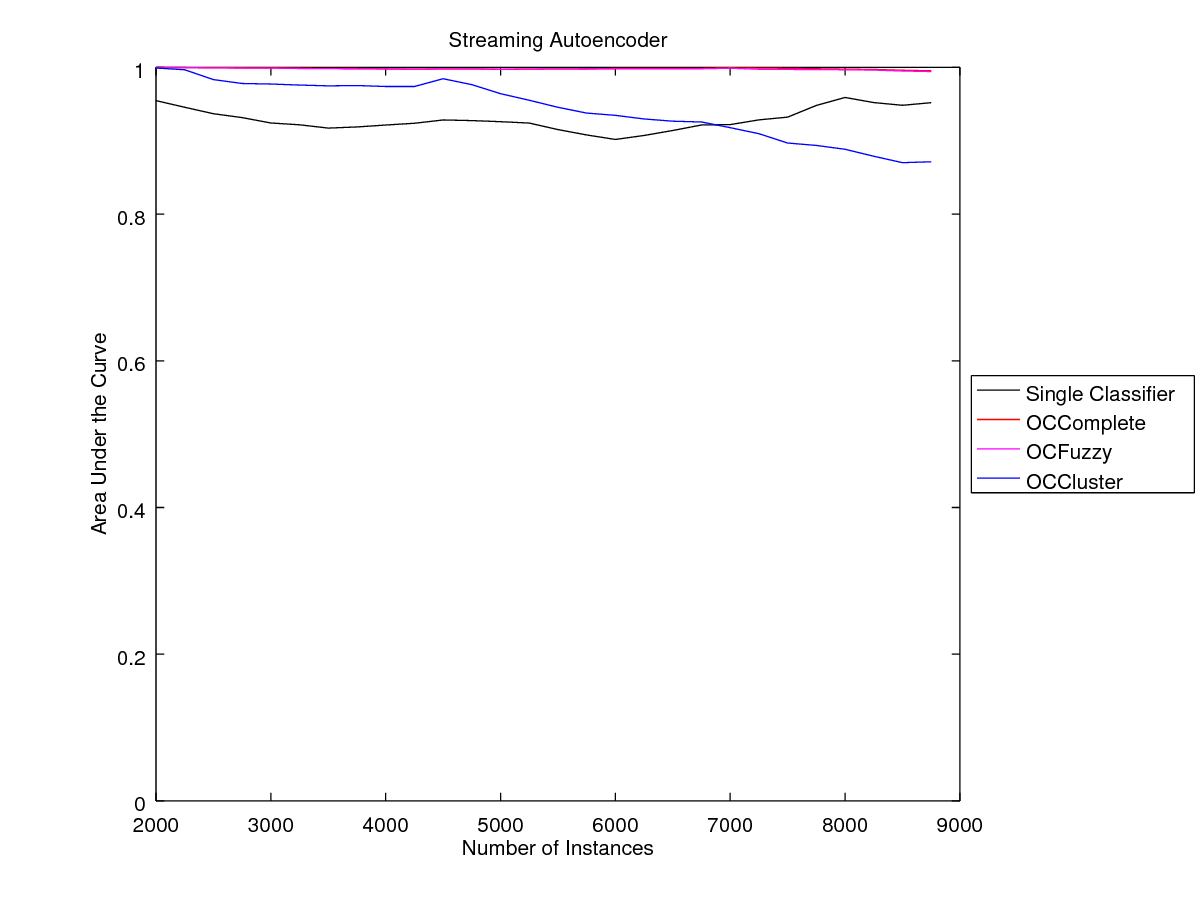}
\includegraphics[width=0.32\textwidth]{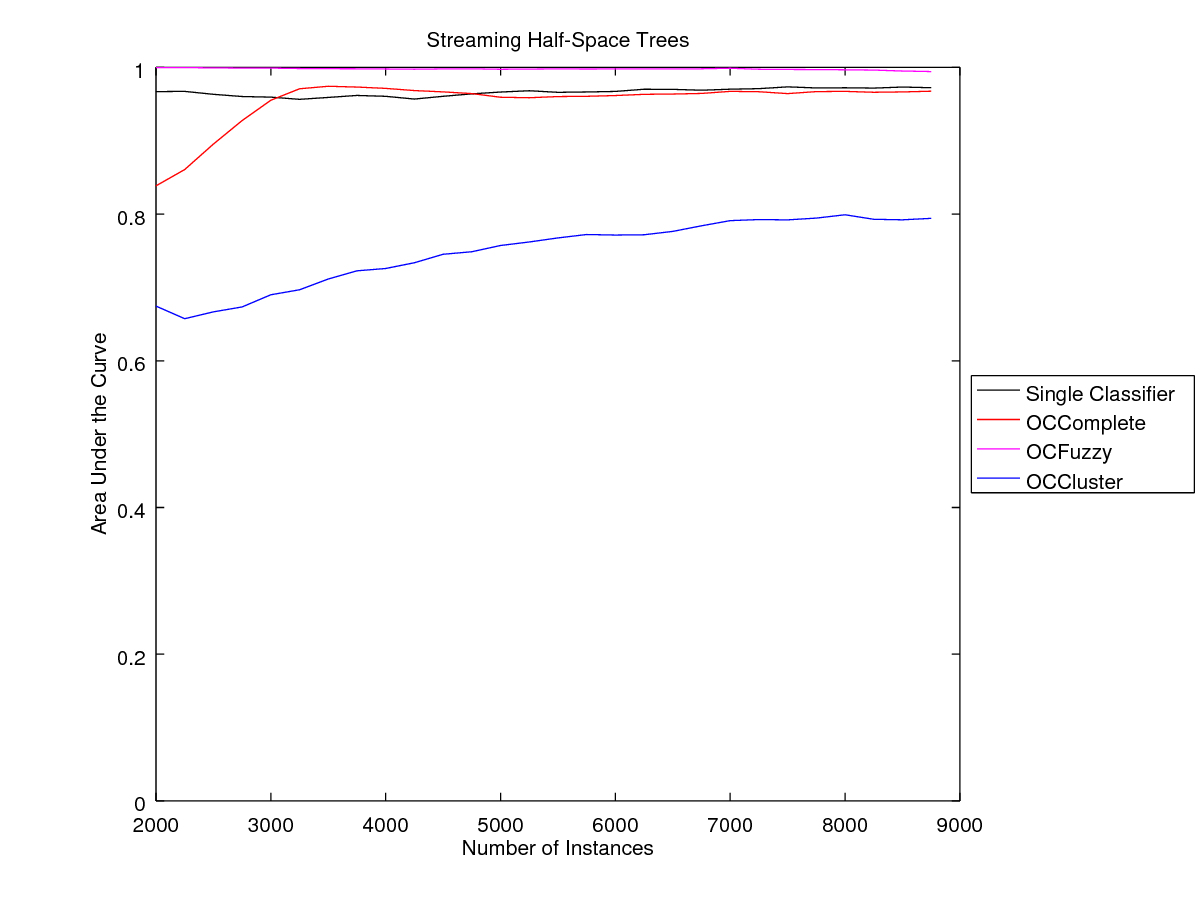}
\includegraphics[width=0.32\textwidth]{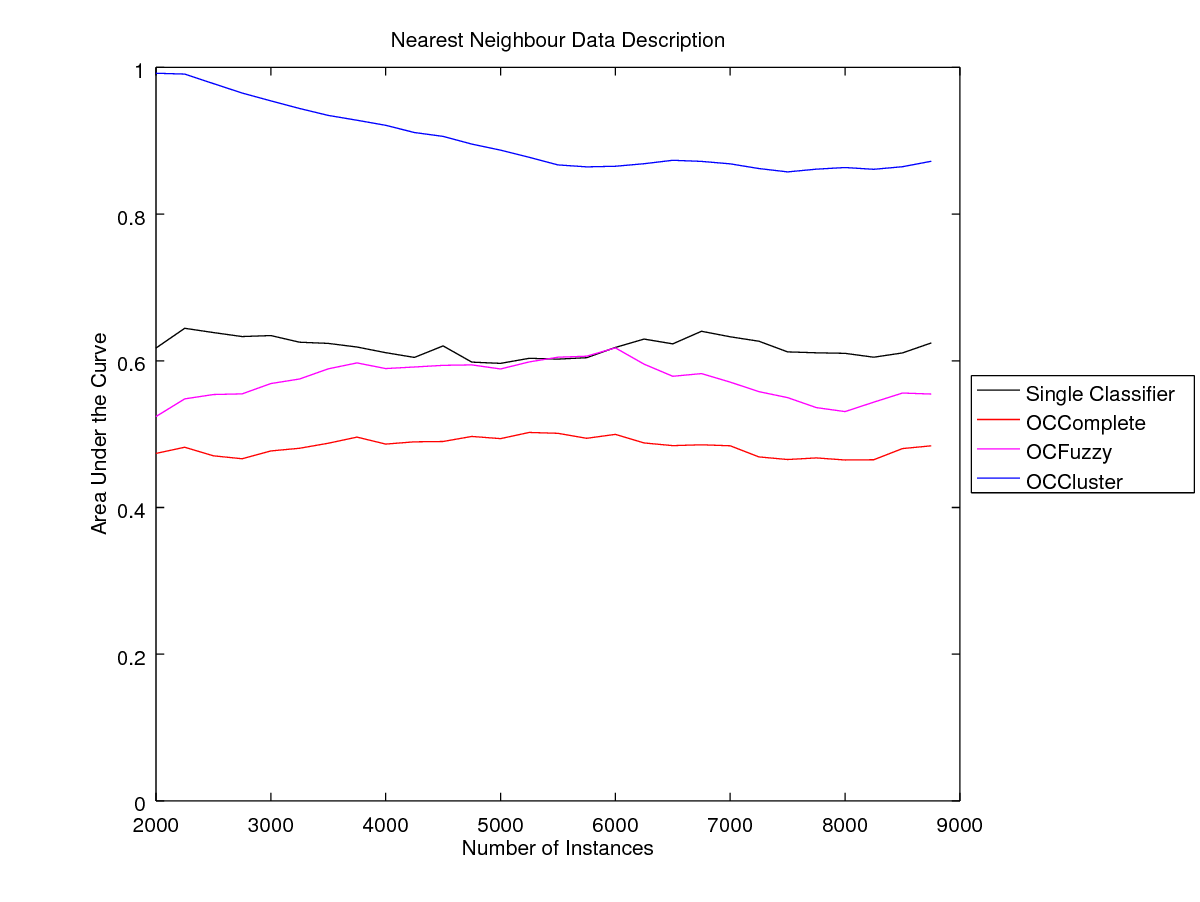}
\end{figure}

\begin{figure}[htbp]
\centering
\caption{Results for the Random RBF data stream}
\label{fig:randomRBF0Results}
\includegraphics[width=0.32\textwidth]{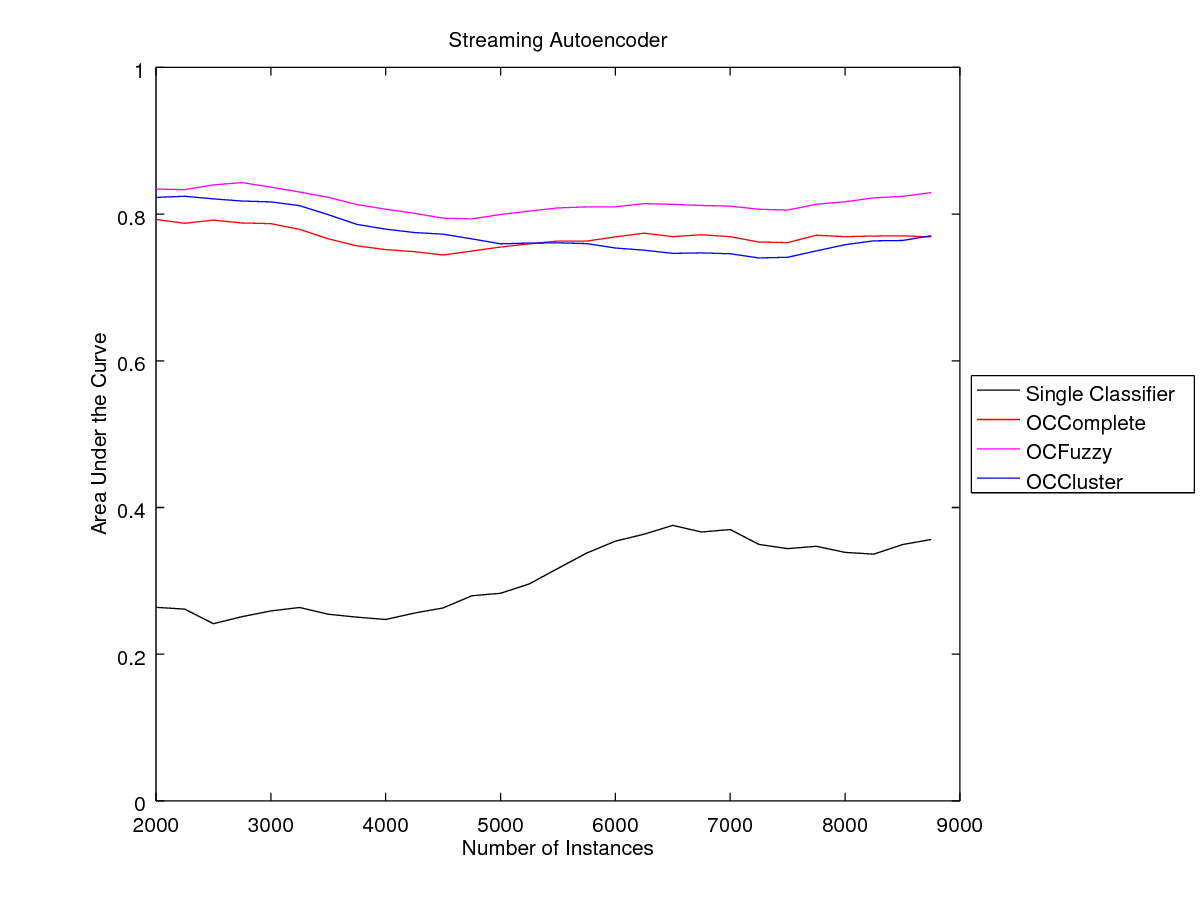}
\includegraphics[width=0.32\textwidth]{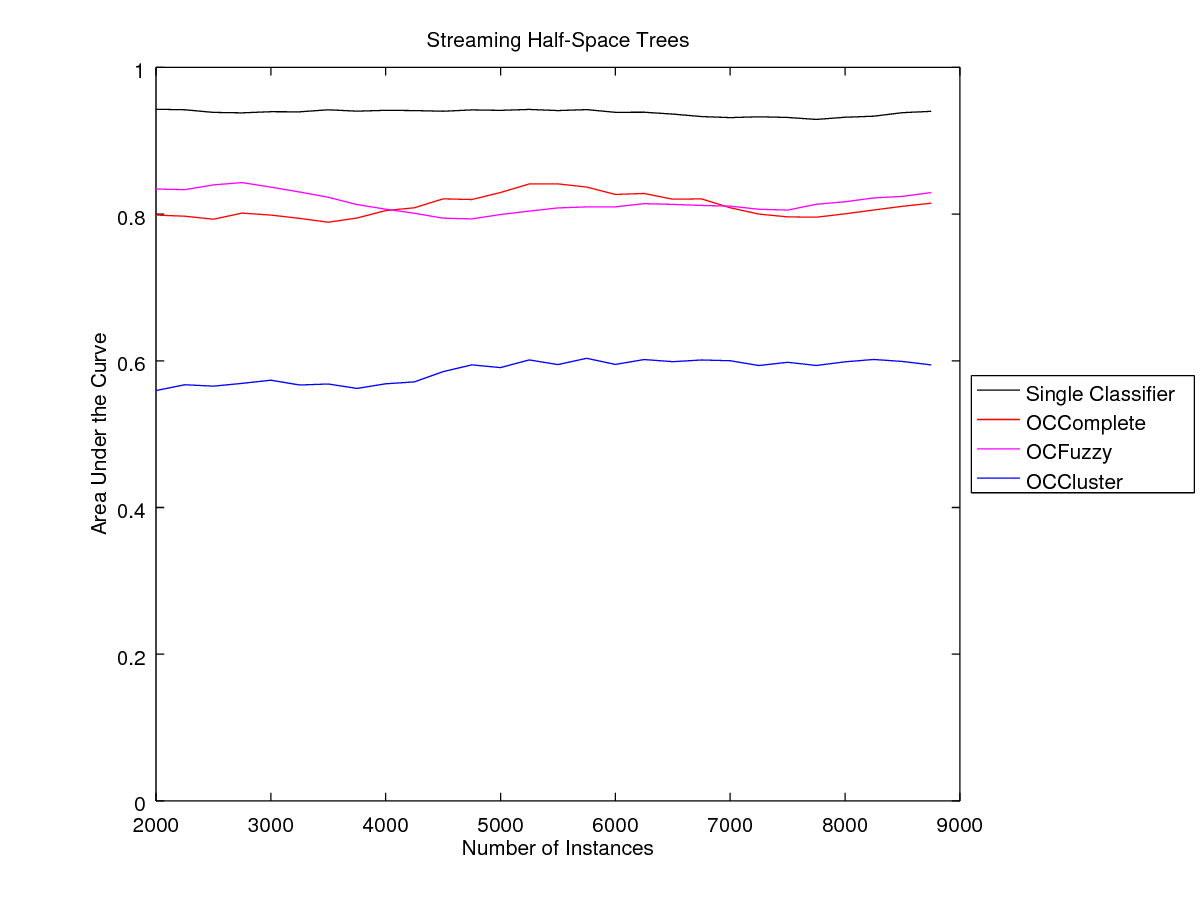}
\includegraphics[width=0.32\textwidth]{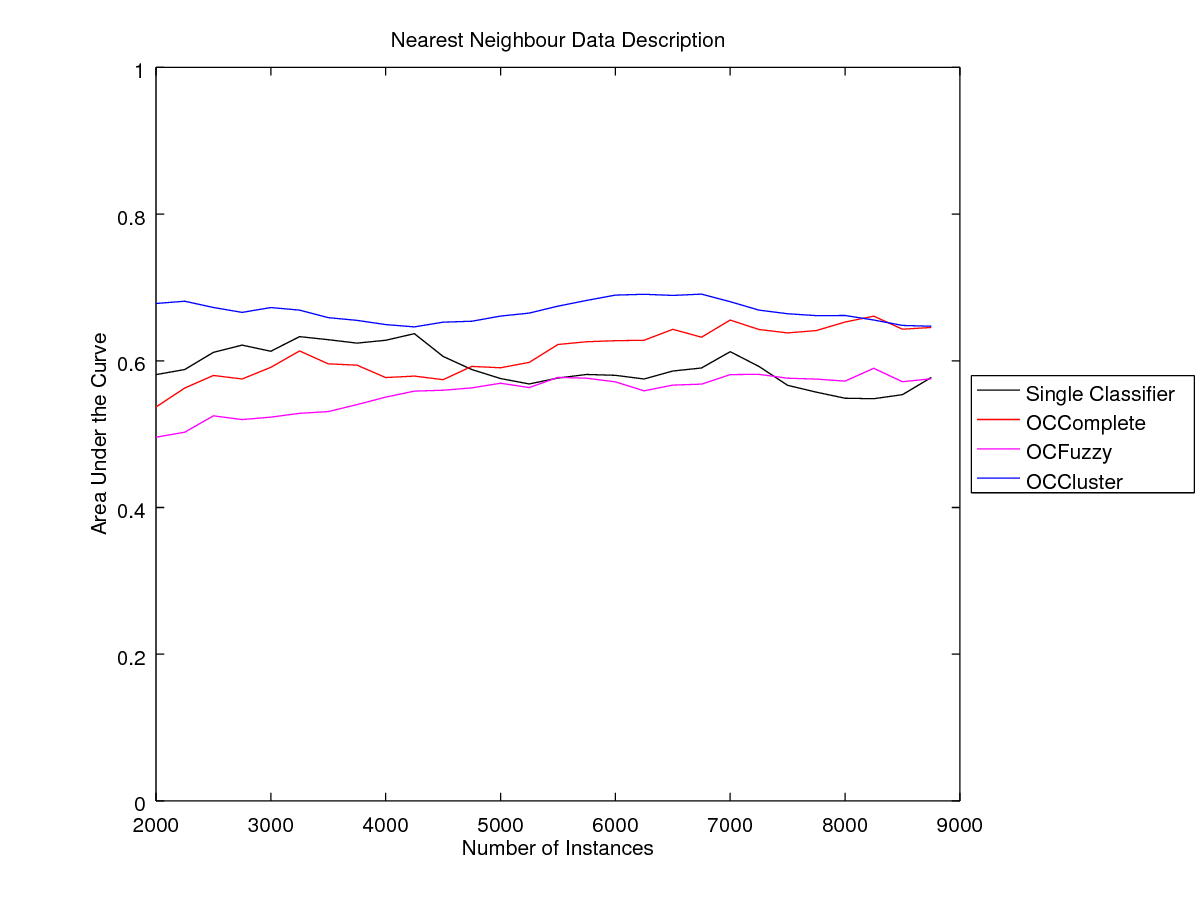}
\end{figure}

\begin{figure}[htbp]
\centering
\caption{Results for the Random RBF data stream with Noise}
\label{fig:randomRBFNResults}
\includegraphics[width=0.32\textwidth]{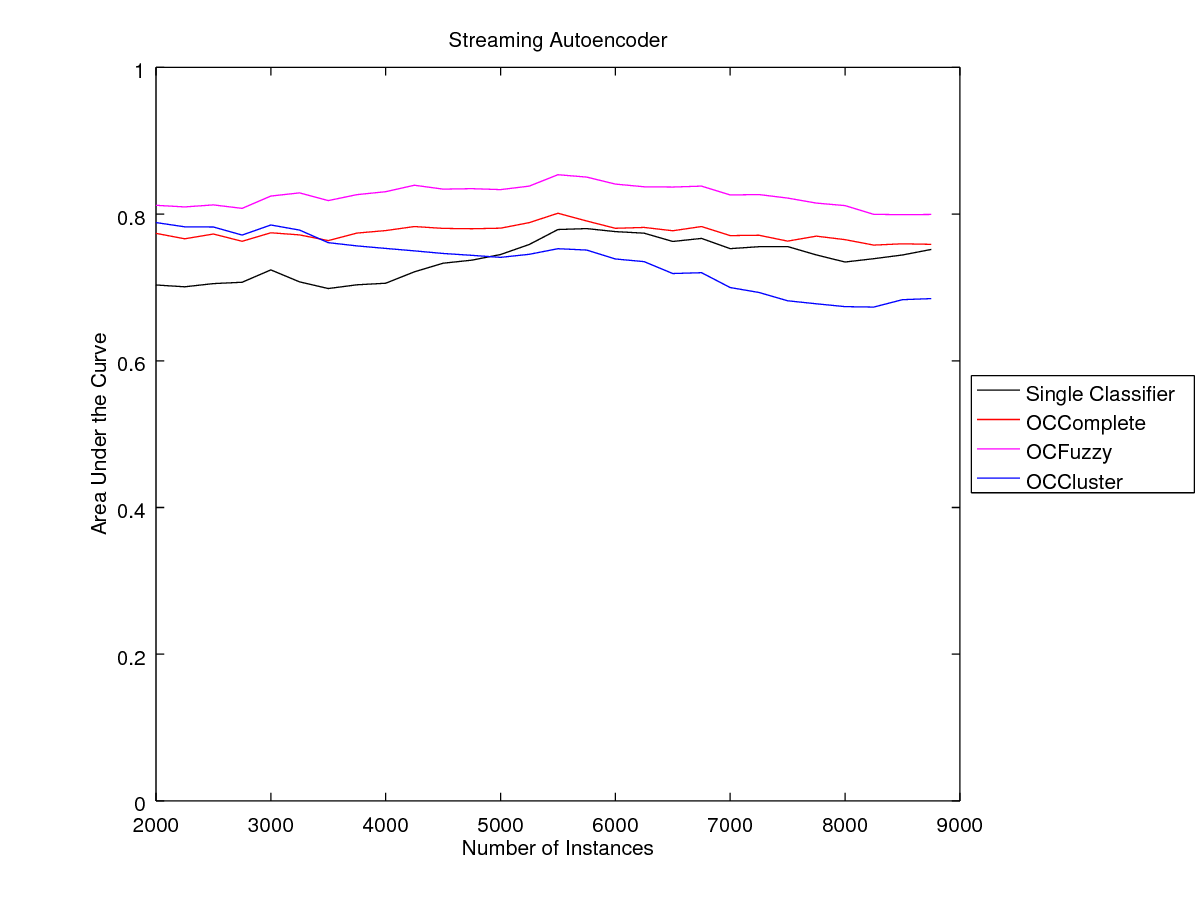}
\includegraphics[width=0.32\textwidth]{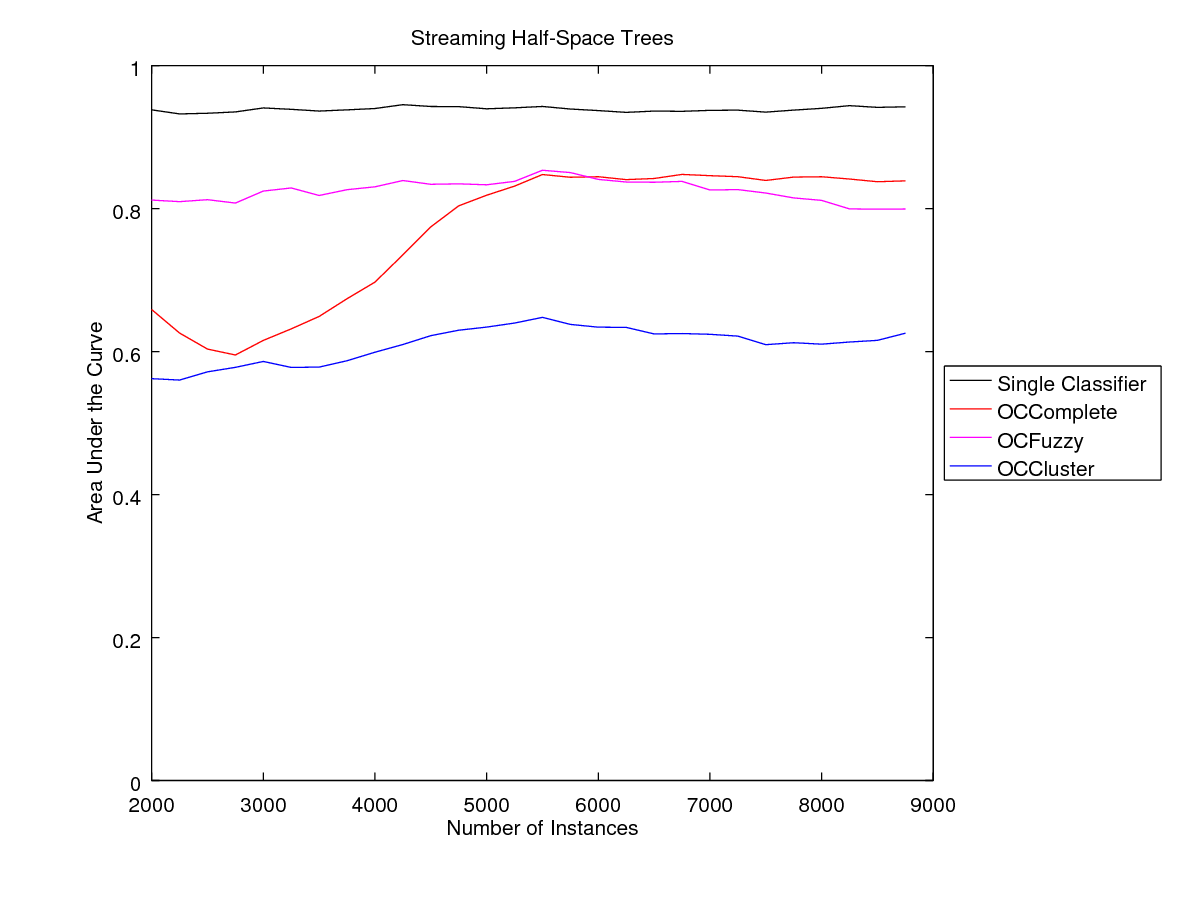}
\includegraphics[width=0.32\textwidth]{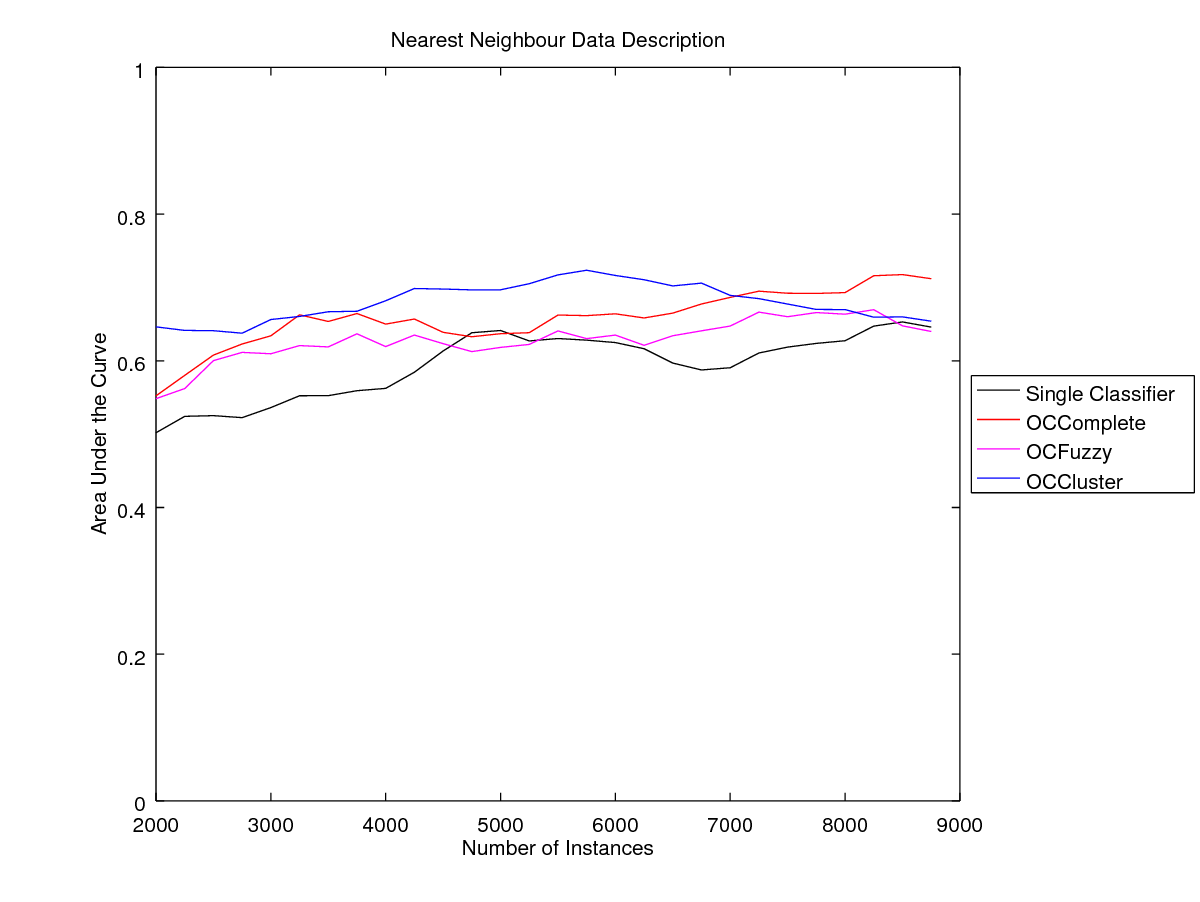}
\end{figure}

Overall, we observed that all three synthetic data streams show consistent results with each of the base classifiers. The performance of both the SA and the NN-d can be improved upon by using contextual knowledge; the case of the Streaming HS-Trees is one where incorporating contexts repeatedly degrades classifier performance.

For the SA, the OCComplete and OCFuzzy frameworks, which make explicit use of contexts, both dominate the single classifier throughout all three data streams. Their AUC scores are relatively consistent throughout the data streams, although this is certainly aided by the streams' synthetic nature. Interestingly, the OCFuzzy framework generally outperforms the OCComplete framework. This suggests that, once contexts are defined, it is better to adapt to the data stream's characteristics than to accept the formal definition of the contexts as the best way of identifying the majority class. The OCCluster framework, which assumes implicit context, shows promise throughout the three data streams. Its performance is susceptible to a decline throughout the data stream, however, which leads it to perform worse than the single classifier after a certain point. A possible cause for this is that the clusters representing the contexts are used to screen training instances for the SA. If any of the clusters begin to incorporate the minority class then this would directly impact the SA's ability to discriminate them.

Next, the NN-d method also benefits from contextual knowledge. Although the single NN-d classifier reliably produced poor AUC scores, these were reliably increased by using the OCCluster framework. Interestingly, the OCComplete and OCFuzzy frameworks only occasionally resulted in higher AUC scores while the OCCluster framework produced AUC scores that dominated the single classifier for all three data streams. It is worth noting that the decision boundary produced by the NN-d method depends only on two points, both of which are very close to the test instance. This suggests that the ClusTree algorithm used by OCCluster is able to find ``local groupings'' of points that are more informative for the NN-d than the formally defined concepts.

The Streaming HS-Trees result in the most disappointing performance. Only the Mixture Model data stream showed any improved performance as a result of using contextual knowledge. Even this is muted by the fact that the single classifier performs very well and the increase in performance from the OCFuzzy framework is of a much smaller magnitude than the decreases in performance seen in the other two data streams. Interestingly, however, when performance using the optimal threshold is considered, the single classifier is generally beaten by OCComplete and occasionally by OCFuzzy as well. For each of the data streams the frameworks see increasing sensitivity (recognition of the minority class) and decreasing specificity (recognition of the majority class) while the single classifier sees both measures stay stable.

Two observations regarding the Streaming HS-Trees are that they embody an ensemble approach and that each HS-Tree's method of recognition involves partitioning the feature space in an ``instance-independent'' manner. This suggests that the effect of training separate HS-Tree ensembles on each context actually has the effect of reducing the information available to each and results in poorer discriminating power.

\section{Benchmark Data Streams}
Following our experiments with synthetic data streams, we test the frameworks on benchmark data streams to determine whether our observations can be transferred from laboratory conditions to real-world problems.

Four imbalanced data streams were constructed from benchmark data sets found in the literature. These benchmark data streams present, in some ways, more challenging tasks than the synthetic data streams. The dimensionality of each is higher and there is a high degree of overlap between the majority and minority classes.

Where possible, contexts were determined from domain knowledge. For the Wine Quality data stream,\footnote{available
from the UCI Machine https://archive.ics.uci.edu/ml/datasets/Wine+Quality} this was whether the \emph{vinho verde} being considered was a red wine or a white wine~\citep{Cortez2009}. For the Covertype data streams,\footnote{retrieved from the MOA website: https://moa.cms.waikato.ac.nz/datasets/} this was the geographic area of the terrain -- instances belonged to one of the Rawah, Comanche Peak, Neota or Cache la Poudre wilderness areas in the Roosevelt National Forest in northern Colorado~\citep{Blackard1999}. For the High Time Resolution Universe Survey (South) data stream (HTRU2),\footnote{available
from the UCI Machine https://archive.ics.uci.edu/ml/datasets/HTRU2 or via DOI 10.6084/m9.figshare.3080389.v1} there are no explicit contexts and they must instead be inferred~\citep{Lyon2016}. These data streams are summarized in Table~\ref{tab:benchmarkDataStreams}.

\begin{table}[htb]
\centering
\caption{Summary of the benchmark data streams}
\label{tab:benchmarkDataStreams}
\begin{tabular}{p{3.5cm} c p{4cm} p{3.5cm}}
\toprule
Name & Atts. & Context & Minority Class \\
\midrule
Wine Quality & 11 & Explicit & Highly overlapped \\
 & & (Red wine, white wine) & \\[0.1cm]
CT 2/5 vs 3/4/6 & 10 & Explicit & Partially overlapped \\
 & & (Wilderness area) & \\[0.1cm]
CT 1/2/5 vs 3/4/6/7 & 10 & Explicit & Partially overlapped \\
 & & (Wilderness area) & \\[0.1cm]
HTRU2 & 8 & Implicit & Little overlap\\
 & & (Unknown) & \\
\bottomrule
\end{tabular}
\end{table}

\subsection{Results and Discussion}
Prequential AUC results for the Wine Quality, Covertype 2/5 vs 3/4/6 and HTRU2 data streams are shown in Figures~\ref{fig:winequalityResults}-\ref{fig:htruResults}. G-mean results for each of these data streams as well as full results for the Covertype 1/2/5 vs 3/4/6/7 data stream are available in Appendix~\ref{app:results}.
\begin{figure}[htbp]
\centering
\caption{Results for the Wine Quality data stream}
\label{fig:winequalityResults}
\includegraphics[width=0.32\textwidth]{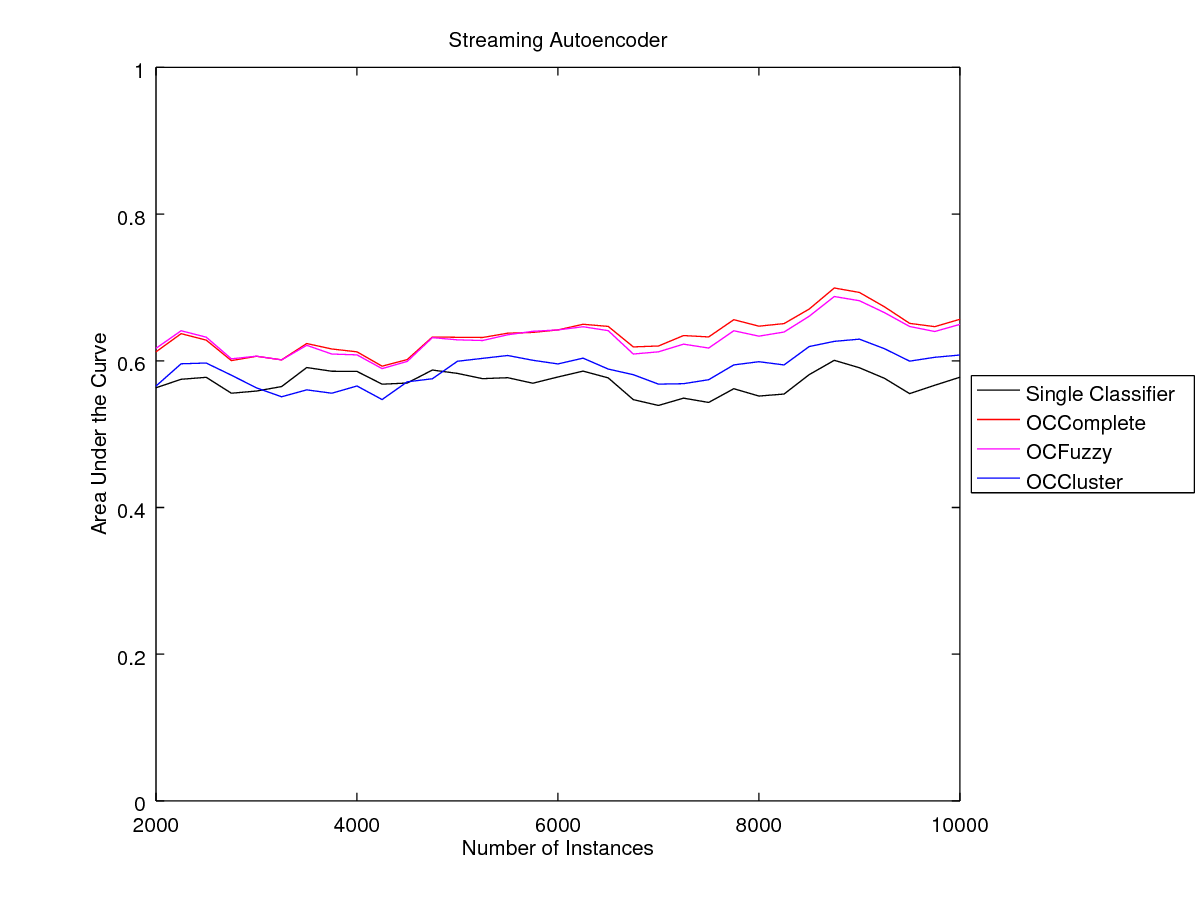}
\includegraphics[width=0.32\textwidth]{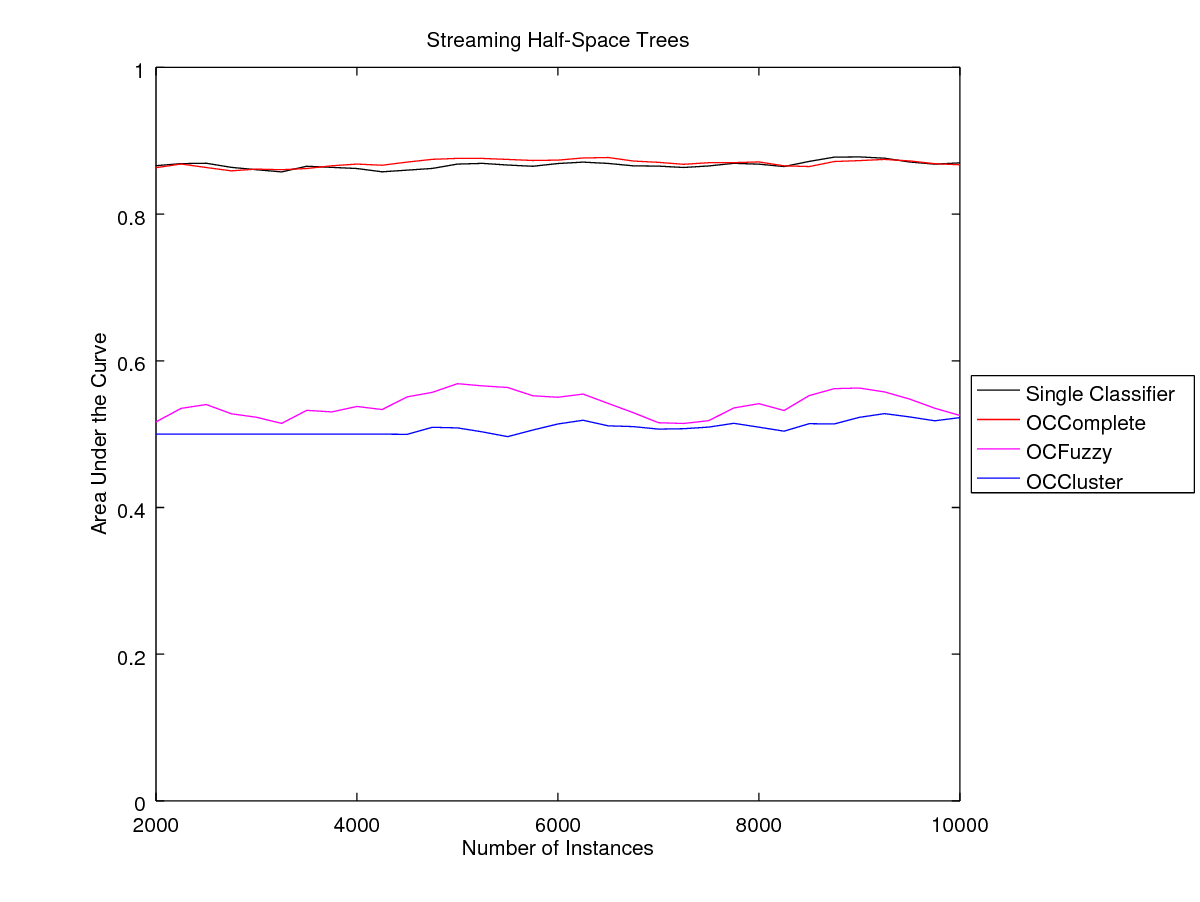}
\includegraphics[width=0.32\textwidth]{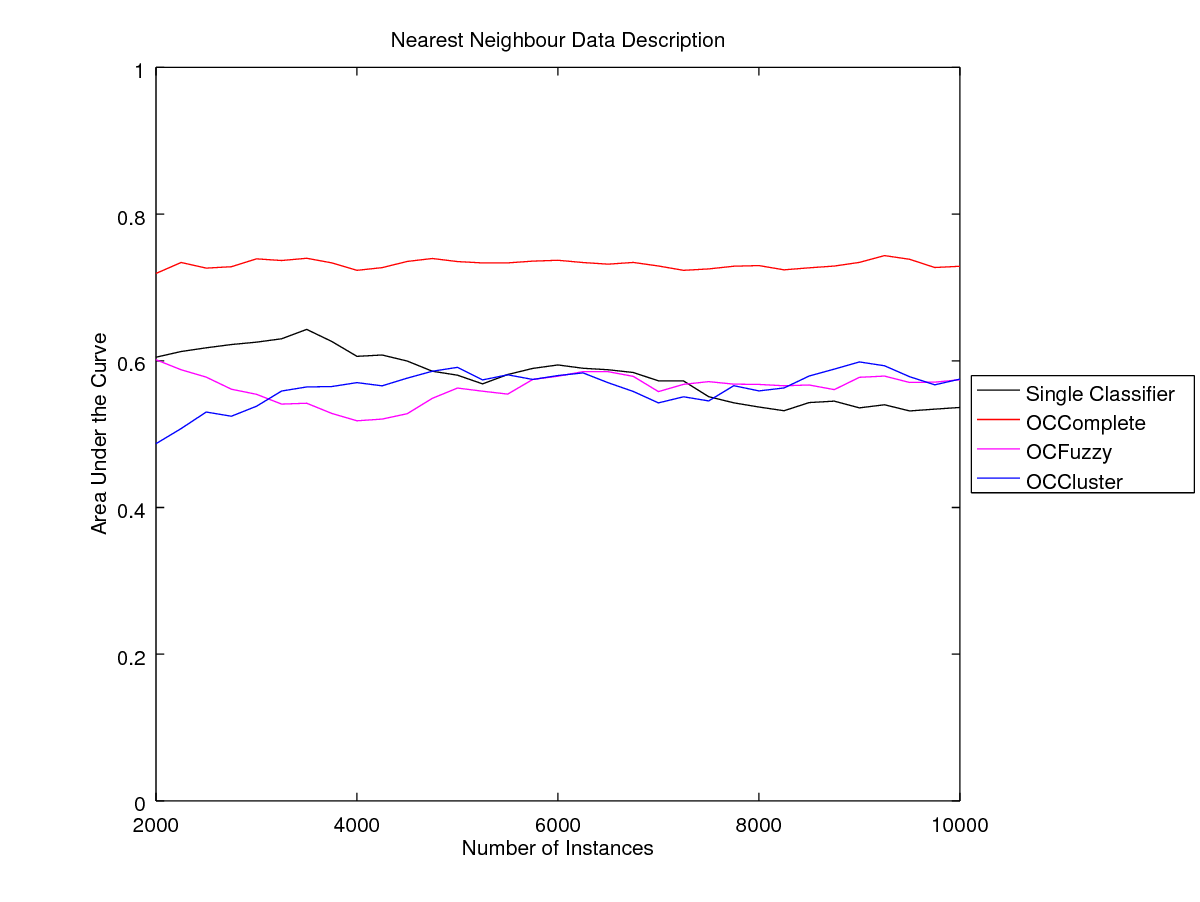}
\end{figure}
\begin{figure}[htbp]
\centering
\caption{Results for the Covertype 2/5 vs 3/4/6 data stream}
\label{fig:covertype25Results}
\includegraphics[width=0.32\textwidth]{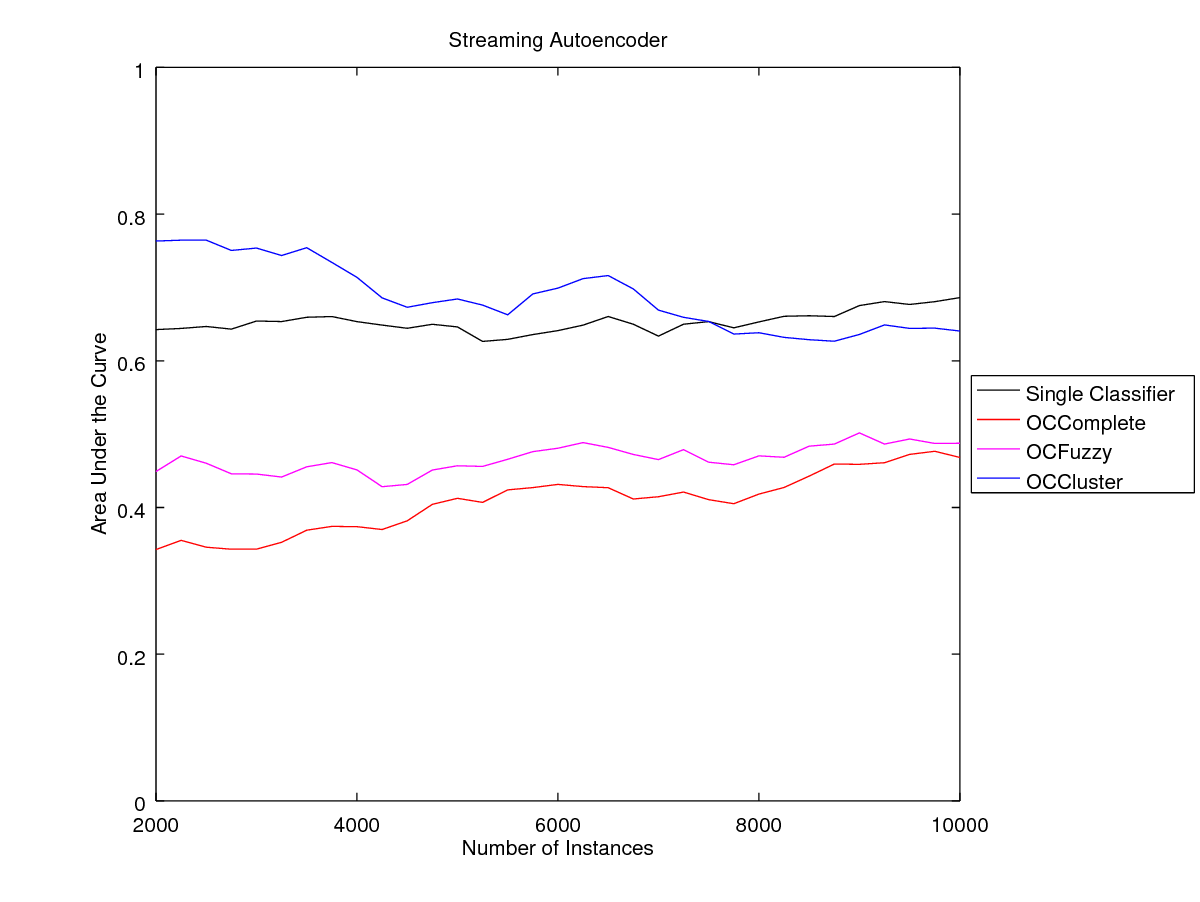}
\includegraphics[width=0.32\textwidth]{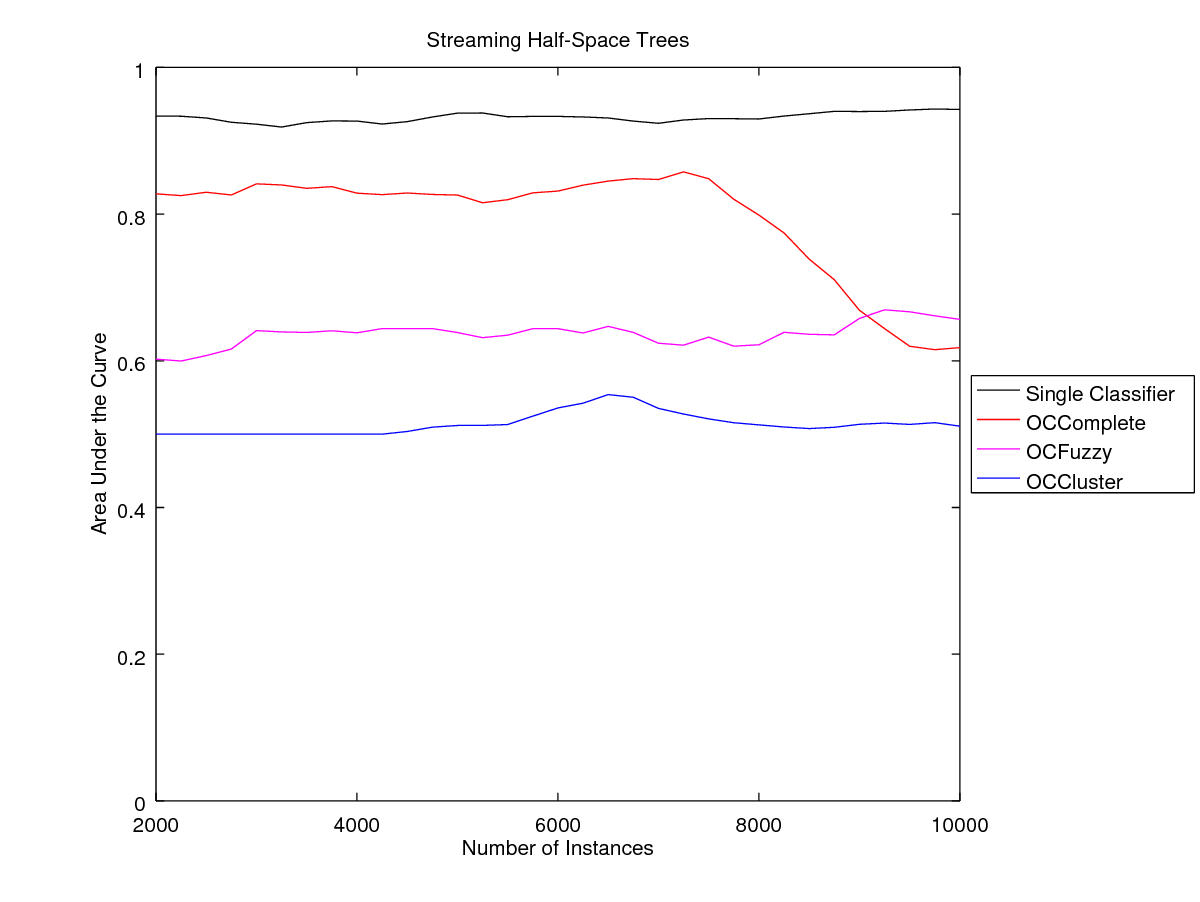}
\includegraphics[width=0.32\textwidth]{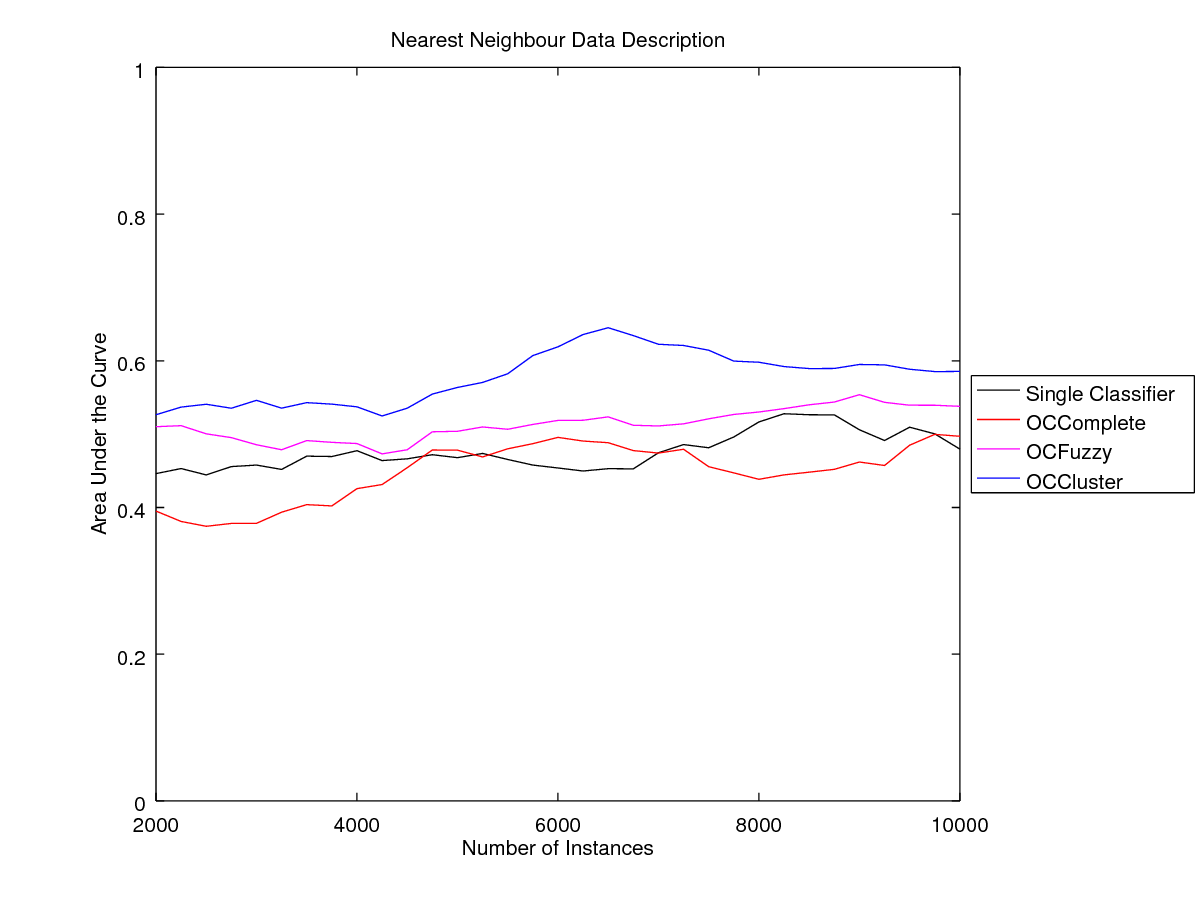}
\end{figure}
\begin{figure}[htbp]
\centering
\caption{Results for the High Time Resolution Universe survey data stream}
\label{fig:htruResults}
\includegraphics[width=0.32\textwidth]{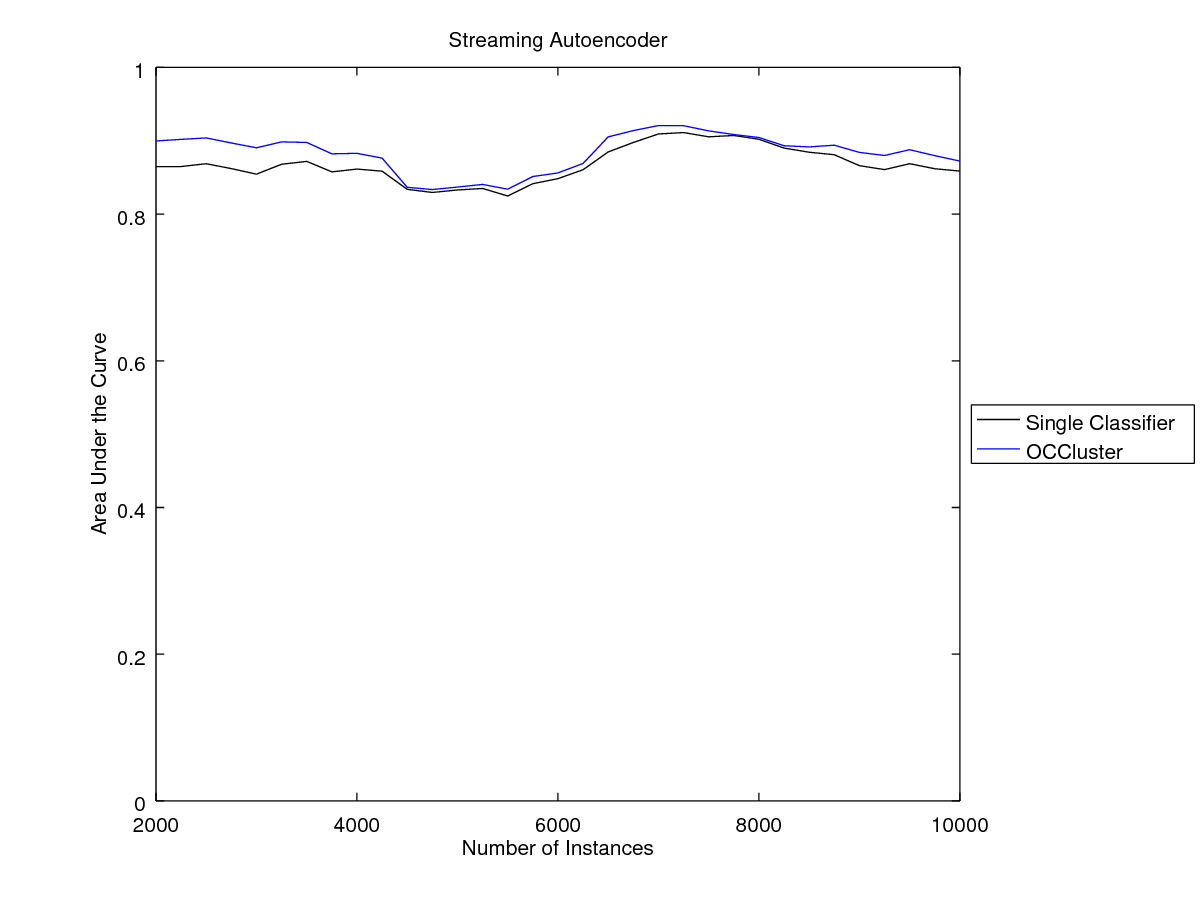}
\includegraphics[width=0.32\textwidth]{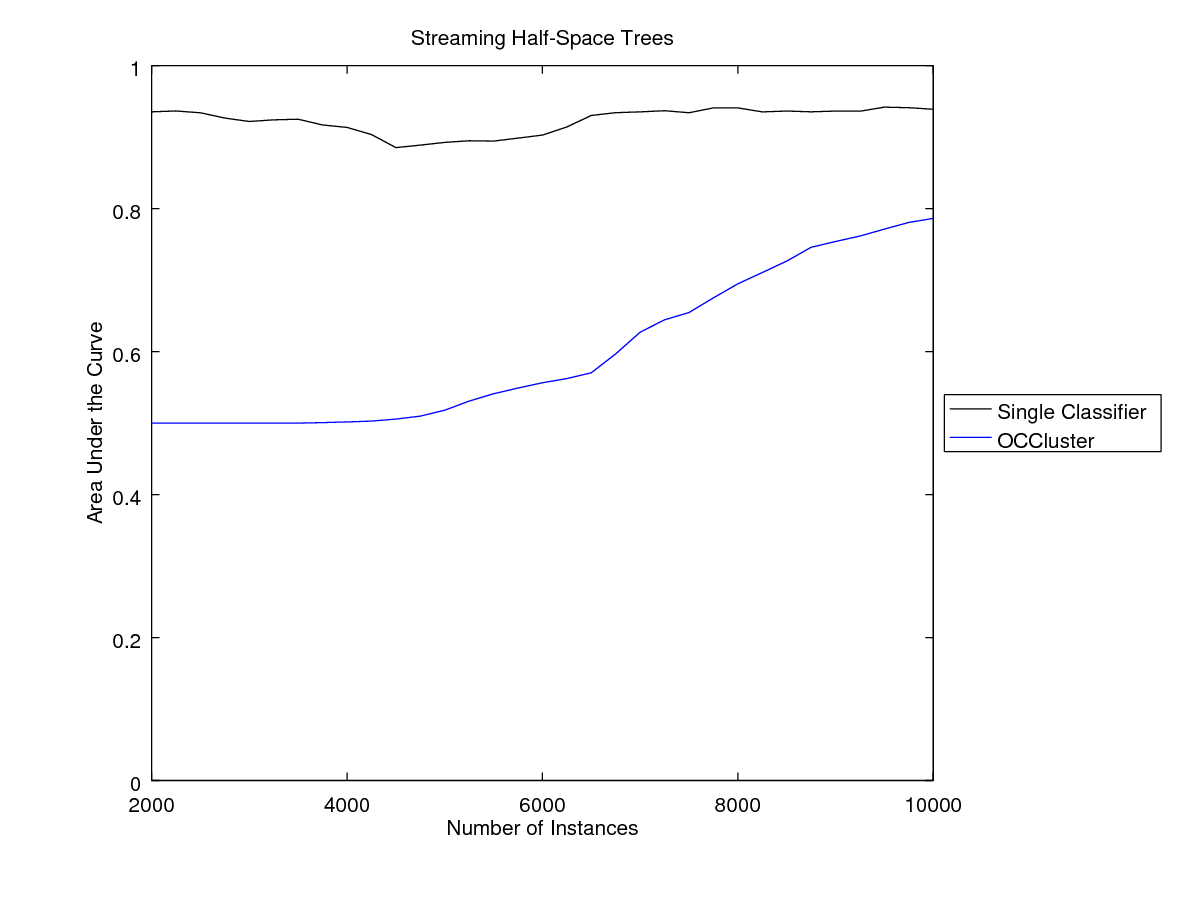}
\includegraphics[width=0.32\textwidth]{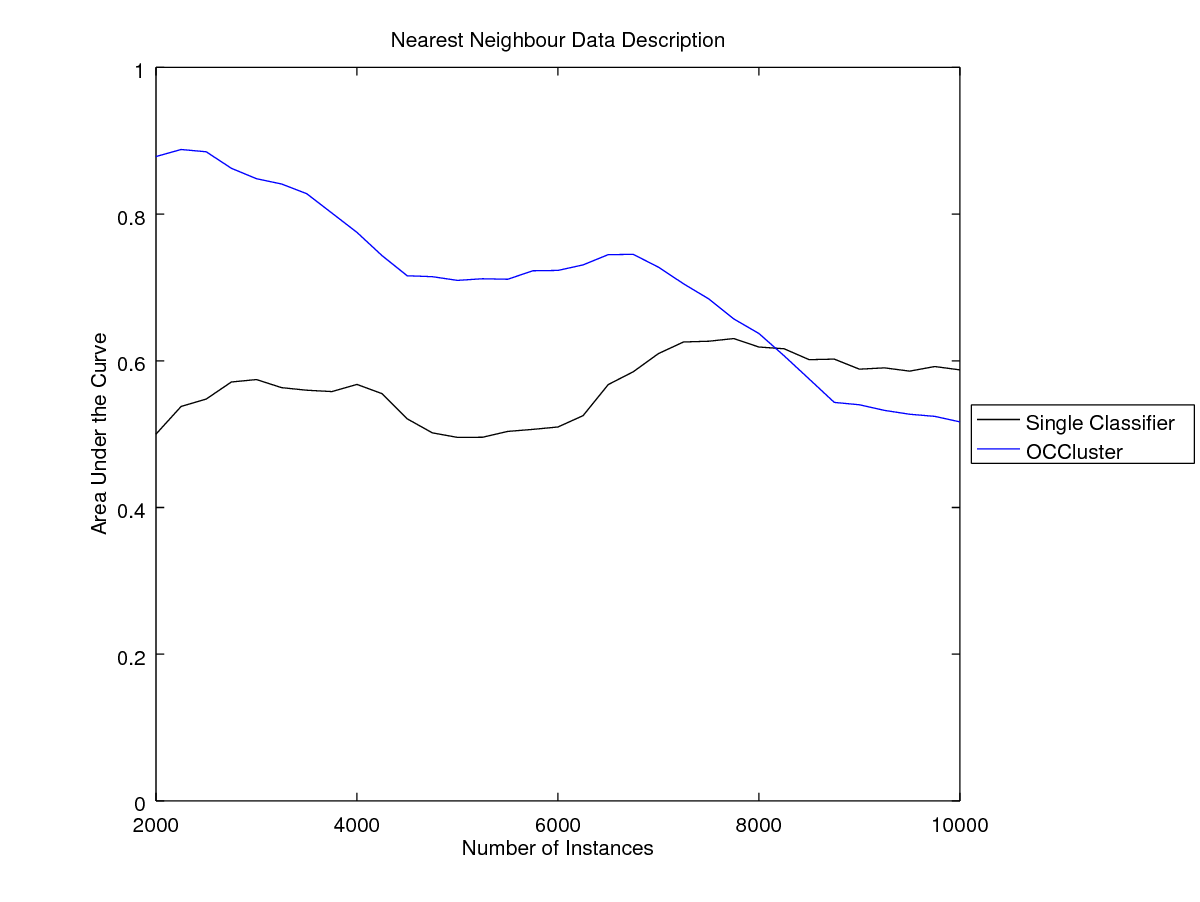}
\end{figure}

The Wine Quality data stream is the only one of the benchmark data stream whose explicitly defined contexts proved useful. Similar results were observed as compared to the synthetic data streams: using context helped both the SA and NN-d classifiers while the single classifier approach to the Streaming HS-Trees performed very well. The most notable aspect of the Wine Quality data stream's results is the excellent performance of the OCComplete framework, which was the best or equal best approach for all three base classifies. Considering this, it seems that the Wine Quality data stream's contexts are hard to discover but very helpful in guiding recognition.

For the Covertype data streams, although OCComplete and OCFuzzy (the two frameworks that use the explicit contexts) do not perform well, OCCluster is able to find useful ways of breaking the majority class down for the SA and NN-d classifiers, resulting in superior performance. The Streaming HS-Trees classifiers, however, saw higher AUC scores as a single classifier than as part of any framework.

For the High Time Resolution Universe survey data stream, no explicit contexts were available. Nonetheless, both the SA and NN-d classifiers were able to improve their AUC scores with the use of the OCCluster framework. This supports the idea that contexts recovered via unsupervised learning can be profitably used by a one-class classifier without needing explicit definition.

For the Streaming HS-Trees, the single classifier approach was again the superior approach. This matches the results obtained on the synthetic data streams and again suggests that the HS-Tree's design is not conducive to additional division by context.

\section{Discussion}
Beginning with our limitations, the most notable was evident in the performance of the Streaming HS-Trees. This classifier only rarely showed an improvement for the frameworks over the single classifier; more distressingly the frameworks generally performed significantly worse. This was especially true for the OCCluster framework, which was never able to produce useful representations for the HS-Trees to learn from. Also notable is that the Streaming HS-Trees as a single classifier generally outperformed the other two classifiers for both synthetic and benchmark data streams, regardless of context usage. As noted in the earlier discussions, there are a few possible explanations for this.

First, Streaming HS-Trees use a density-based method for OCC. As reviewed in Section~\ref{sec:oneClassClassifiers}, it is reasonable to believe that density-based methods are more dependent on a global representation of the majority class than either reconstruction-based or boundary-based approaches. Density-based approaches may benefit more from a global picture of the data stream and suffer more from seeing only a sub-space of the data stream's feature space. Second, Streaming HS-Trees is an ensemble method made up of individual HS-Trees. As reviewed in Section~\ref{sec:complexDomains}, ensembles inherently incorporate the idea of sub-dividing a problem and it possible that there is a limit to which a problem can be usefully sub-divided, limiting the ability of the three frameworks to produce better performance. Finally, it may be that specific contexts are responsible for the difficulties encountered by the SA and NN-d classifiers. If this is the case, then knowledge of the contexts is proving its worth by highlighting which regions of the feature space and which contexts are hardest to properly represent.

Another identified limitation is that the OCCluster framework is susceptible to decreasing AUC scores as the data stream progresses. This obviously limits the usefulness of the technique for data streams, where steady-state behaviour over the long term is important. As discussed earlier in this chapter, the deciding factor for this is likely whether or not the clustering algorithm's clusters are able to track the majority class, whether they track the majority class tightly, and whether they begin to model concentrations of minority class instances as well.

That said, there are positive results as well. The novel Cluster Distance Function and the observation regarding required window sizes are two theoretical contributions that assisted with the development of these frameworks and that can be used by other researchers for their own future work.

In terms of experimental results, both the SA and NN-d classifiers regularly saw improved performance when contextual knowledge was incorporated. In the case of the SA, the explicit contexts used by the OCComplete and OCFuzzy frameworks led to the best performance, though the OCCluster framework generally outperformed the single classifier as well. For the NN-d method we saw a very different trend: the OCCluster framework generally outperformed the other three approaches and did dominate the single classifier for most of the data streams.

Also positive was the demonstration that classifier performance could be improved using contextual knowledge, even if explicit contexts are defined for the data stream. This opens the door to applying these techniques to cases where data exploration suggests that useful contexts exists and does not restrict application to data streams for which domain knowledge can explicitly identify context.

\section{Conclusion}
In this paper we described how using contextual knowledge in OCC for data streams has not been sufficiently investigated. Although the idea has been demonstrated for static data sets, its application to streaming one-class classifiers is not a trivial extension. We addressed these challenges by consulting the literature, by proposing new theoretical ideas, and by developing new experimental evidence. What our experimental results show is that contextual knowledge can be used by streaming one-class classifiers to achieve superior performance, though they also highlight a variety of challenges.

In reviewing these challenges, we note clear areas for future research. These include the further development of the Cluster Distance Function to better capture non-overlapped clusters, perhaps by using ideas from the Wasserstein metric (the earth mover's distance), further investigation into the kinds of streaming one-class classifiers that are able to benefit from contextual knowledge, and a method to avoid the degradation in classifier performance that occurs in some data streams for the OCCluster framework. With the caveats acknowledged, however, it is still the case that contextual knowledge can be employed to improve one-class classifier performance in data streams.

\section*{Acknowledgements}
The authors acknowledge that Queen's University is situated on traditional Anishnaabe and Haudenosaunee Territory and that the University of Ottawa is situated on traditional unceded Algonquin territory. This research was supported by the Natural Sciences and Engineering Research Council of Canada and the Province of Ontario.

\newpage
\appendix
\section{Technical Results} \label{app:results}
\subsection{Proof that the Cluster Distance Function is a metric} \label{sec:proofCDF}
This is the proof that the Cluster Distance Function proposed in Section \ref{sec:clusterDistanceFunction} is a metric.
\begin{proof}
To prove that $CD(\cdot,\cdot)$ is a metric, we must show that it satisfies the four conditions of a metric.

\begin{enumerate}
\item $|IP_{\mathcal{C}}(x) - IP_{\mathcal{C}'}(x)| \geq 0 \text{ }\forall \text{ x } \in \mathbb{F}$ and the integral of a non-negative function is itself non-negative. $\therefore CD(\mathcal{C},\mathcal{C}') \geq 0$ and the condition of \textbf{non-negativity} is satisfied.

\item $|IP_{\mathcal{C}}(x) - IP_{\mathcal{C}}(x)| = 0 \text{ }\forall \text{ x } \in \mathbb{F}$ and the integral of $0$ is itself $0$ $\therefore \mathcal{C} = \mathcal{C}' \implies CD(\mathcal{C},\mathcal{C}') = 0$\\

Because $|IP_{\mathcal{C}}(x) - IP_{\mathcal{C}'}(x)| \geq 0 \text{ }\forall \text{ x } \in \mathbb{F}$,\\ $CD(\mathcal{C},\mathcal{C}') = 0 \implies |IP_{\mathcal{C}}(x) - IP_{\mathcal{C}'}(x)| = 0 \text{ }\forall \text{ x } \in \mathbb{F} \implies IP_{\mathcal{C}}(x) = IP_{\mathcal{C}'}(x) \text{ }\forall \text{ x } \in \mathbb{F}$, if the IP for two clusters is the same for all $x \in \mathbb{F}$, that means that they are the same (hyper-) volumes. $\therefore CD(\mathcal{C},\mathcal{C}') = 0 \implies \mathcal{C} = \mathcal{C}'$\\

$\mathcal{C} = \mathcal{C}' \implies CD(\mathcal{C},\mathcal{C}') = 0 \text{ and } CD(\mathcal{C},\mathcal{C}') = 0 \implies \mathcal{C} = \mathcal{C}' \therefore CD(\mathcal{C},\mathcal{C}') = 0 \iff \mathcal{C} = \mathcal{C}'$ and the \textbf{identity of indiscernibles} is satisfied.

\item $|IP_{\mathcal{C}}(x) - IP_{\mathcal{C}'}(x)| = |IP_{\mathcal{C}'}(x) - IP_{\mathcal{C}}(x)| \text{ }\forall \text{ x } \in \mathbb{F}$ and so $CD(\mathcal{C},\mathcal{C}') = CD(\mathcal{C}',\mathcal{C})$ $\therefore$ the condition of \textbf{symmetry} is satisfied.

\item Consider the feature space, $\mathbb{F}$, Figure~\ref{fig:vennDiagramFeatureSpace}. The different areas in the diagram denote the regions of $\mathbb{F}$ that contains instances of the clusters $\mathcal{A}$, $\mathcal{B}$, and $\mathcal{C}$. For the purposes of this proof, the areas are of arbitrary, non-negative size. We calculate the distance between each pair of these arbitrary clusters and show that the triangle inequality holds.\\

\begin{figure}[htb]
\centering
\caption{Venn Diagram of clusters $\mathcal{A}$, $\mathcal{B}$, and $\mathcal{C}$ in feature space $\mathbb{F}$}
\label{fig:vennDiagramFeatureSpace}
\includegraphics[width=0.25\textwidth]{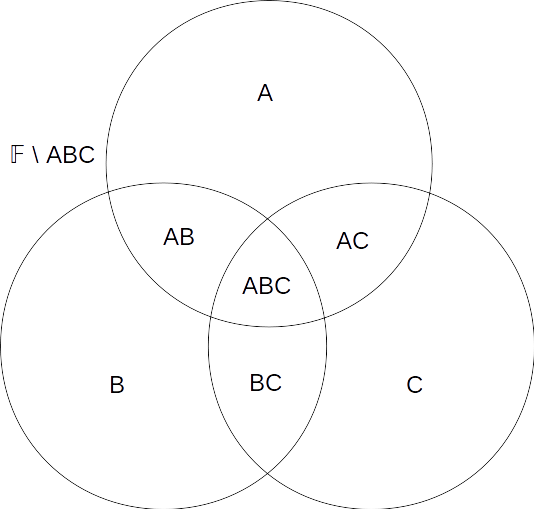}
\end{figure}

\begin{align*}
CD(\mathcal{A}, \mathcal{B}) = &\int_{\mathbb{F}} |IP_{\mathcal{A}}(x) - IP_{\mathcal{B}}(x)|\\
= &\int_{\mathcal{A}} |IP_{\mathcal{A}\setminus\mathcal{B}}(x) - IP_{\mathcal{B}}(x)| + \int_{\mathcal{B}\setminus\mathcal{A}} |IP_{\mathcal{A}}(x) - IP_{\mathcal{B}}(x)| +\\ &\int_{\mathcal{AB}} |IP_{\mathcal{A}}(x) - IP_{\mathcal{B}}(x)| + \int_{\mathbb{F}\setminus\mathcal{A}\mathcal{B}} |IP_{\mathcal{A}}(x) - IP_{\mathcal{B}}(x)|\\
= &\int_{A \cup AC} 1 + \int_{B \cup BC} 1 + \int_{AB \cup ABC} 0 + \int_{\mathbb{F}\setminus\mathcal{A}\mathcal{B}} 0\\
= &|A|+|AC|+|B|+|BC|\\
\end{align*}
Similarly, $CD(\mathcal{A}, \mathcal{C}) = |A|+|AB|+|C|+|BC|$ and $CD(\mathcal{B}, \mathcal{C}) = |B|+|AB|+|C|+|AC|$.\\
\begin{align*}
CD(\mathcal{A},\mathcal{C}) \leq\ &CD(\mathcal{A},\mathcal{B}) + CD(\mathcal{B},\mathcal{C})\\
|A| + |AB| + |C| + |BC| \leq\ &|A| + |AC| + |B| + |BC| + |B| +\\
&|AB| + |C| + |AC|\\
0 \leq\ &2|B| + 2|AC|
\end{align*}
$A$, $B$, and $C$ are all non-negative, therefore $0 \leq 2|B| + 2|AC|$ is true and the \textbf{triangle inequality} is satisfied.
\end{enumerate}

$CD(\cdot,\cdot)$ satisfies all four conditions of a metric $\therefore CD(\cdot,\cdot)$ is a metric.
\end{proof}
\newpage

\section{Experimental Results}
\subsection{Synthetic Data Streams}
\begin{figure}[htbp]
\centering
\caption{Results for the Mixture Model data stream}
\label{fig:mixtureModelGMean}
\includegraphics[width=0.32\textwidth]{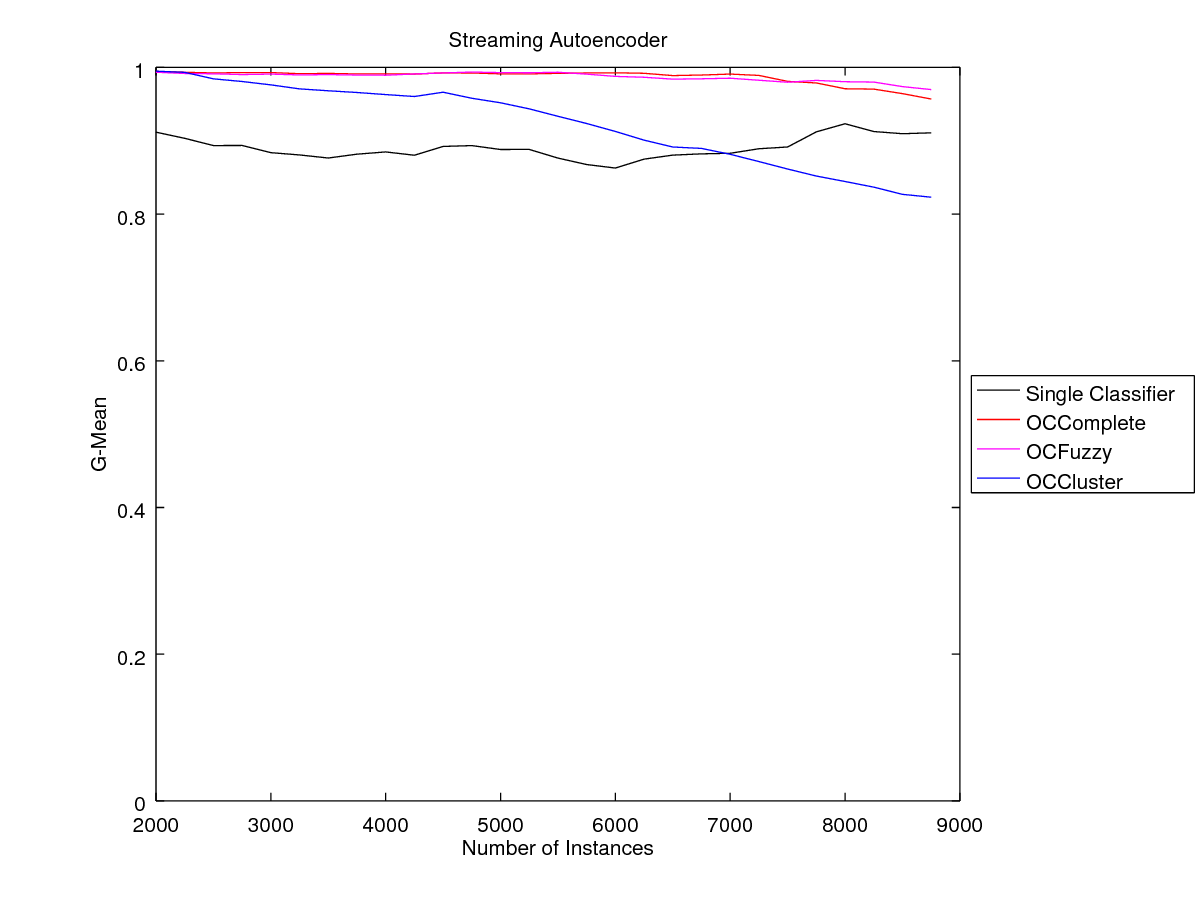}
\includegraphics[width=0.32\textwidth]{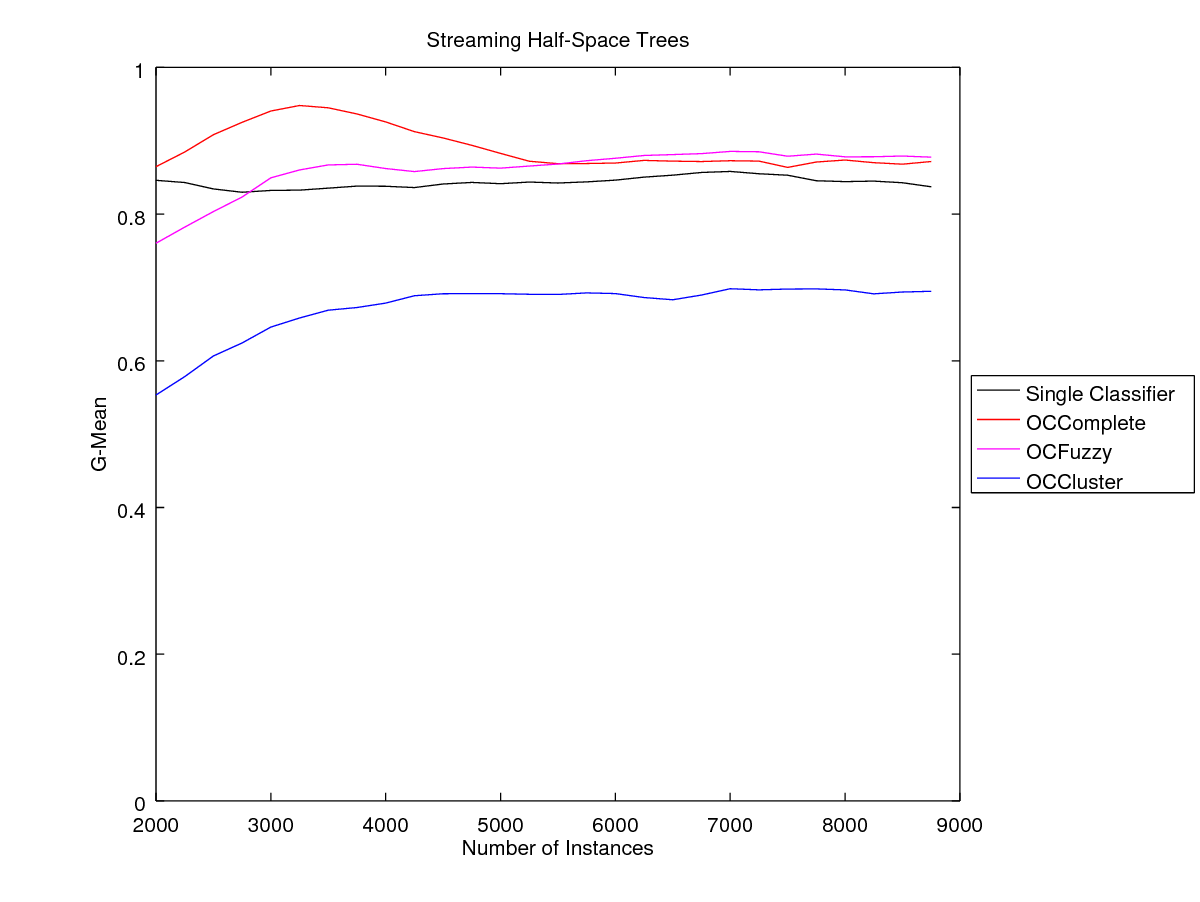}
\includegraphics[width=0.32\textwidth]{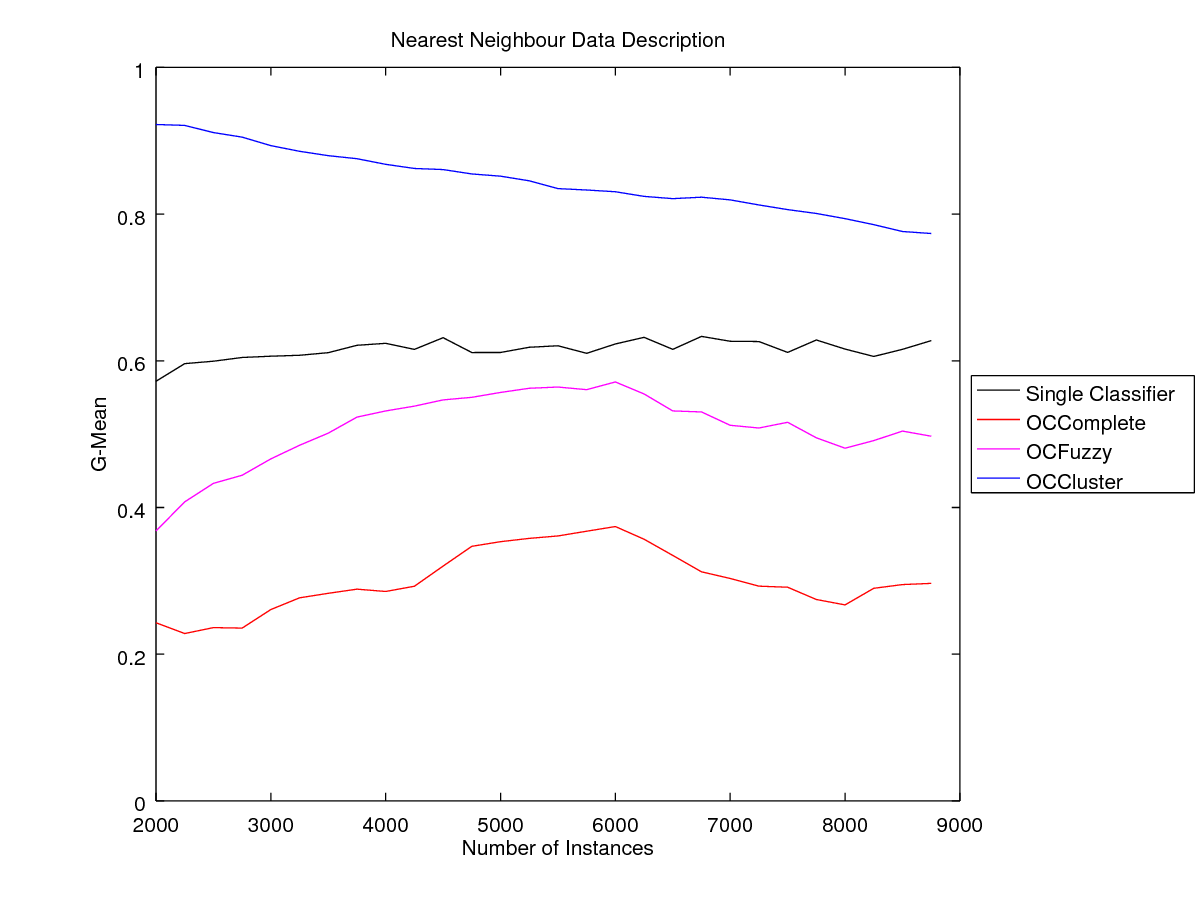}
\end{figure}

\begin{figure}[htbp]
\centering
\caption{Results for the Random RBF data stream}
\label{fig:randomRBF0GMean}
\includegraphics[width=0.32\textwidth]{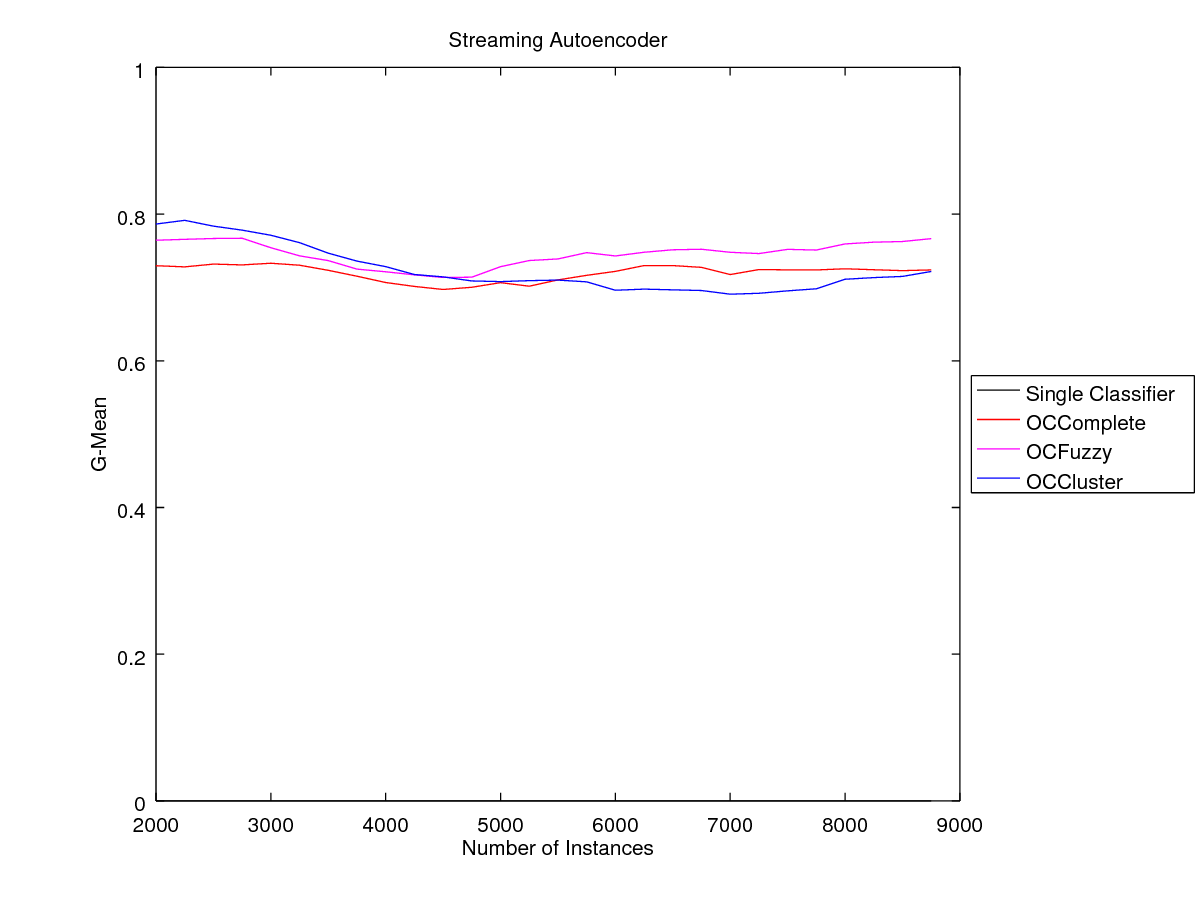}
\includegraphics[width=0.32\textwidth]{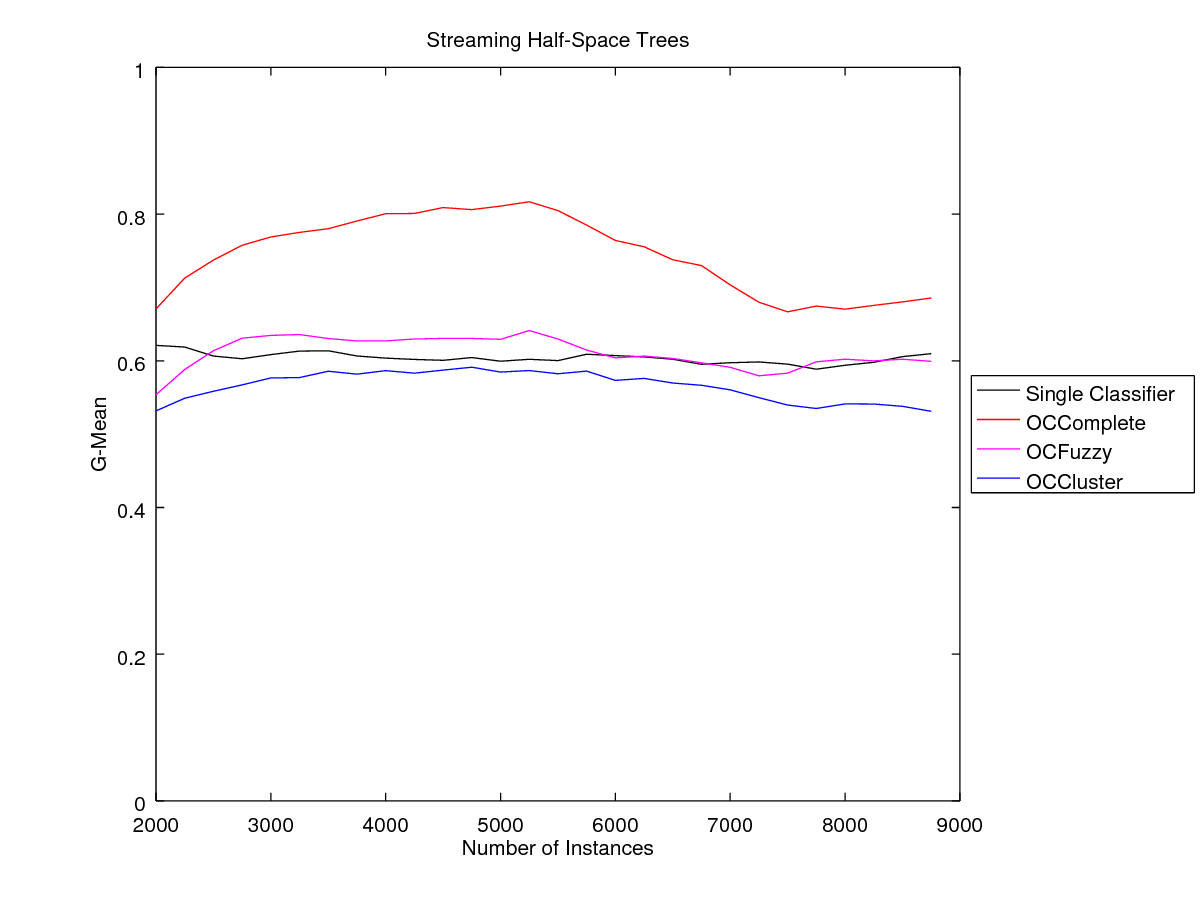}
\includegraphics[width=0.32\textwidth]{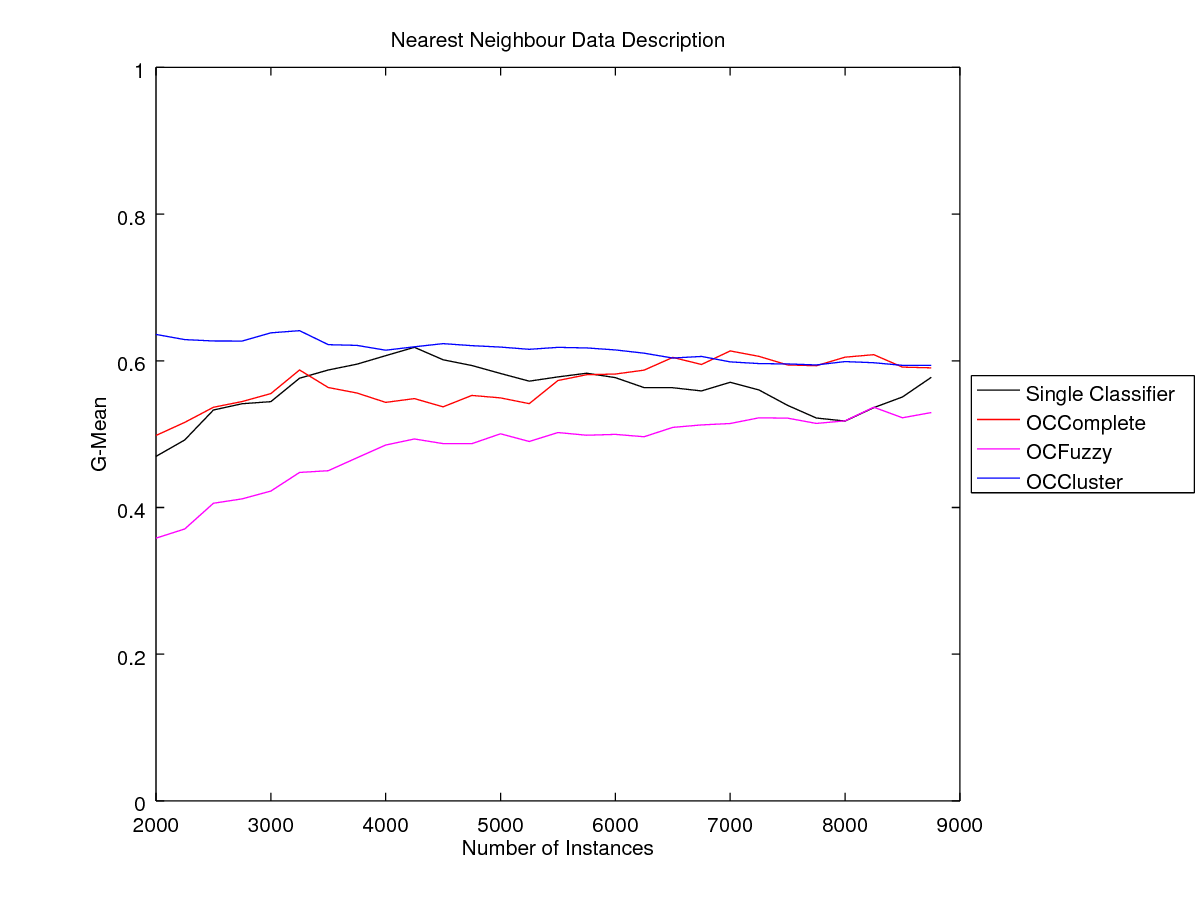}
\end{figure}

\begin{figure}[htbp]
\centering
\caption{Results for the Random RBF data stream with Noise}
\label{fig:randomRBFNGMean}
\includegraphics[width=0.32\textwidth]{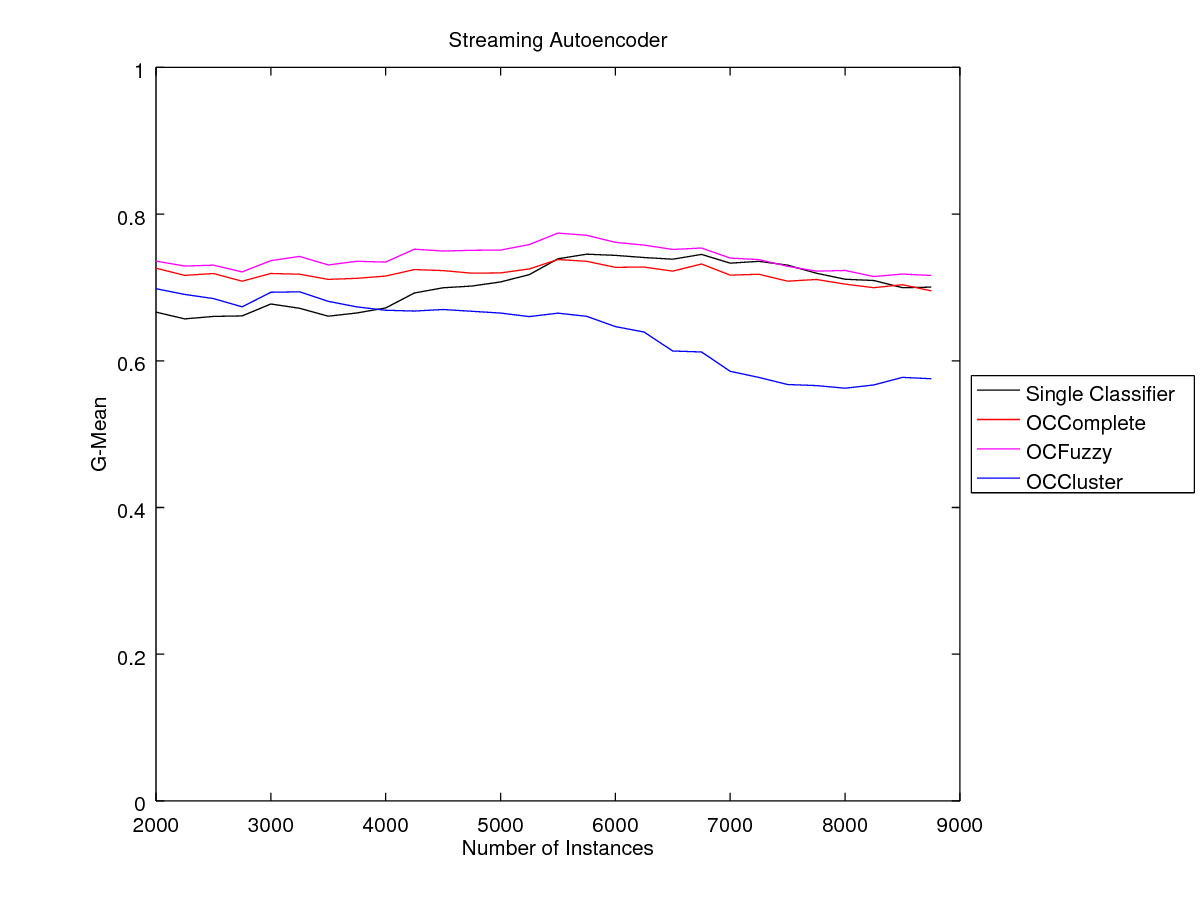}
\includegraphics[width=0.32\textwidth]{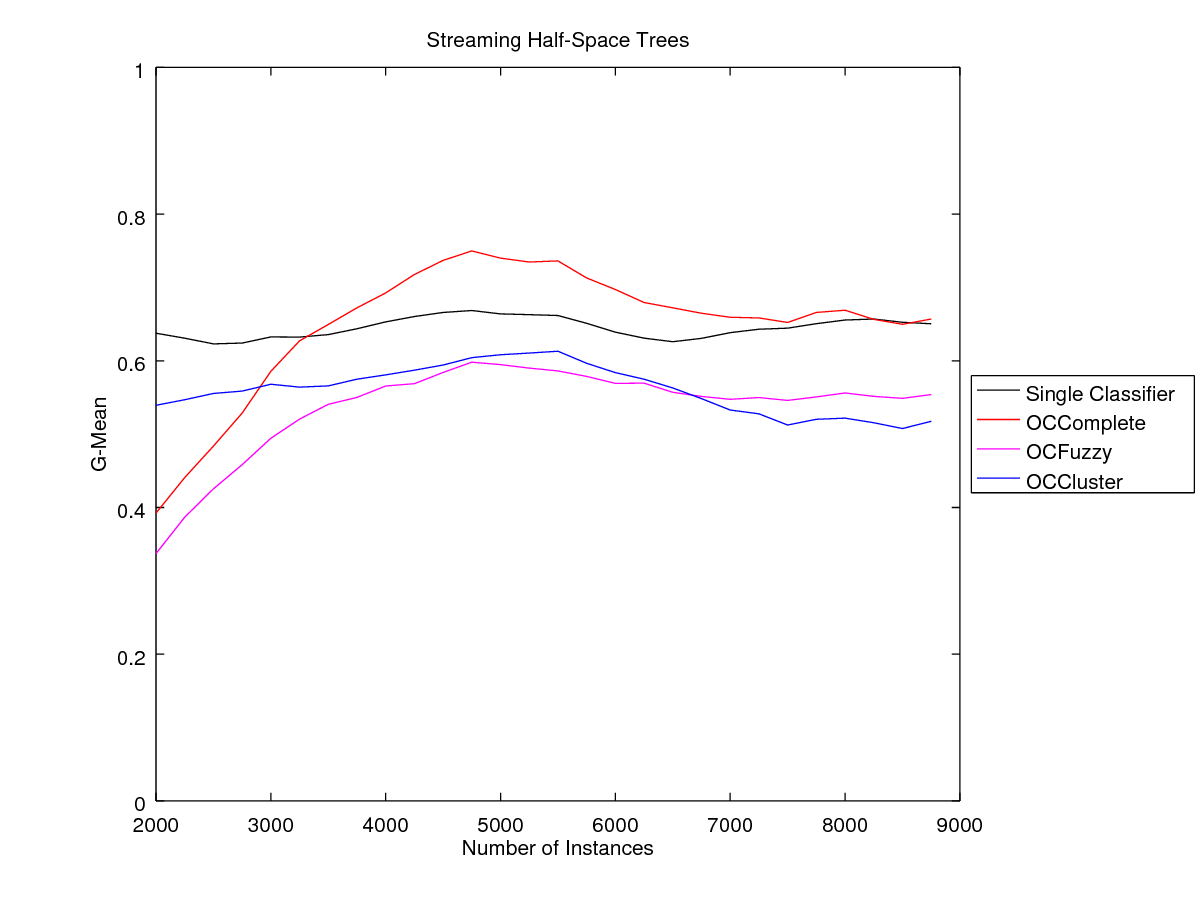}
\includegraphics[width=0.32\textwidth]{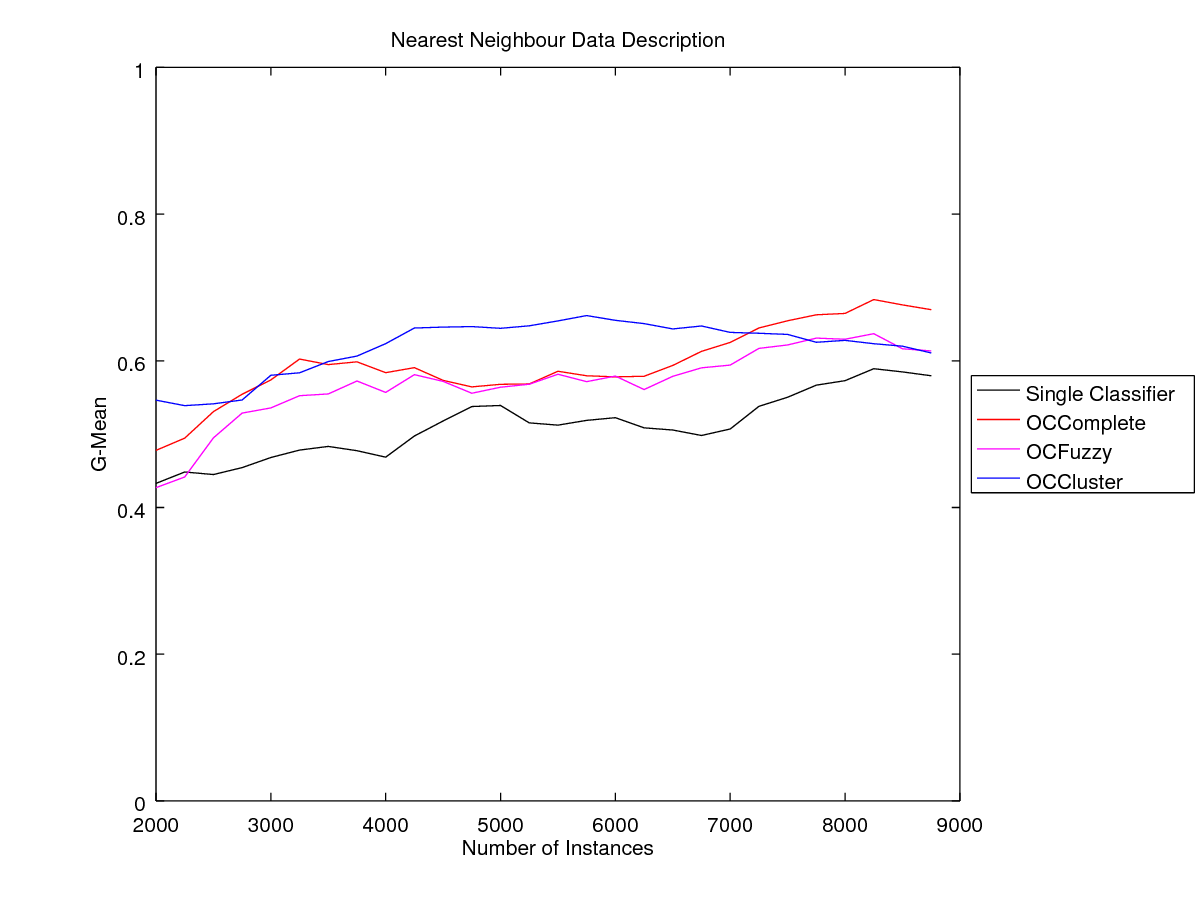}
\end{figure}
\FloatBarrier
\newpage

\subsection{Benchmark Data Streams}
\begin{figure}[htbp]
\centering
\caption{Results for the Wine Quality data stream}
\label{fig:winequalityGMean}
\includegraphics[width=0.32\textwidth]{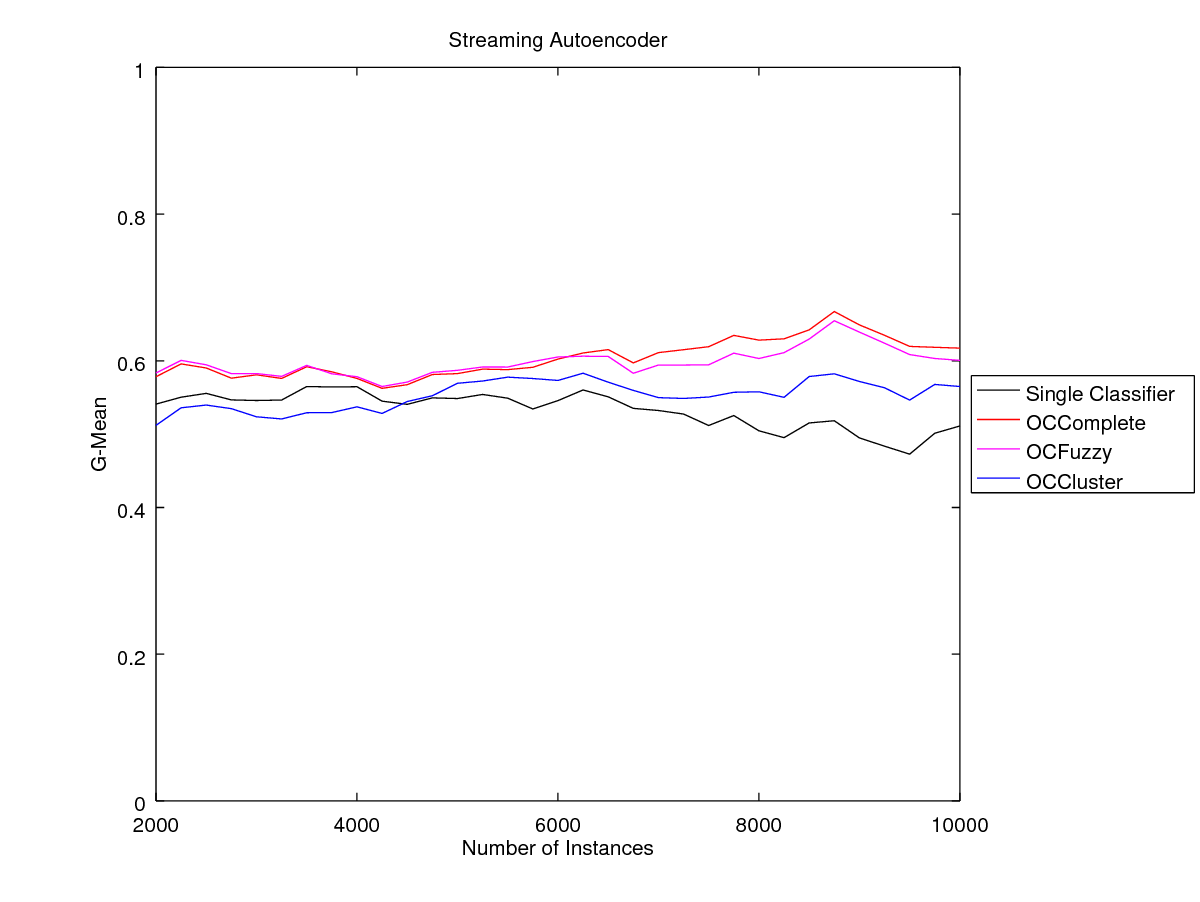}
\includegraphics[width=0.32\textwidth]{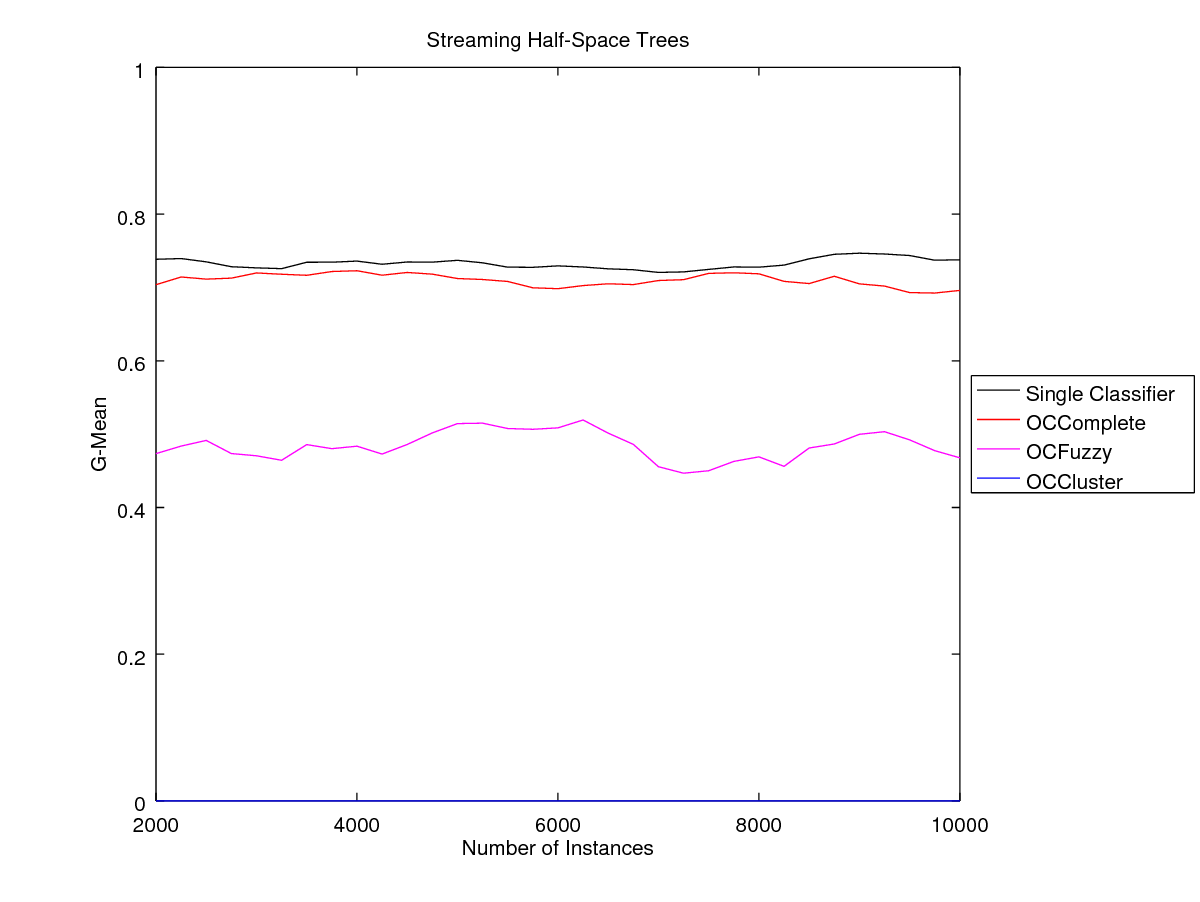}
\includegraphics[width=0.32\textwidth]{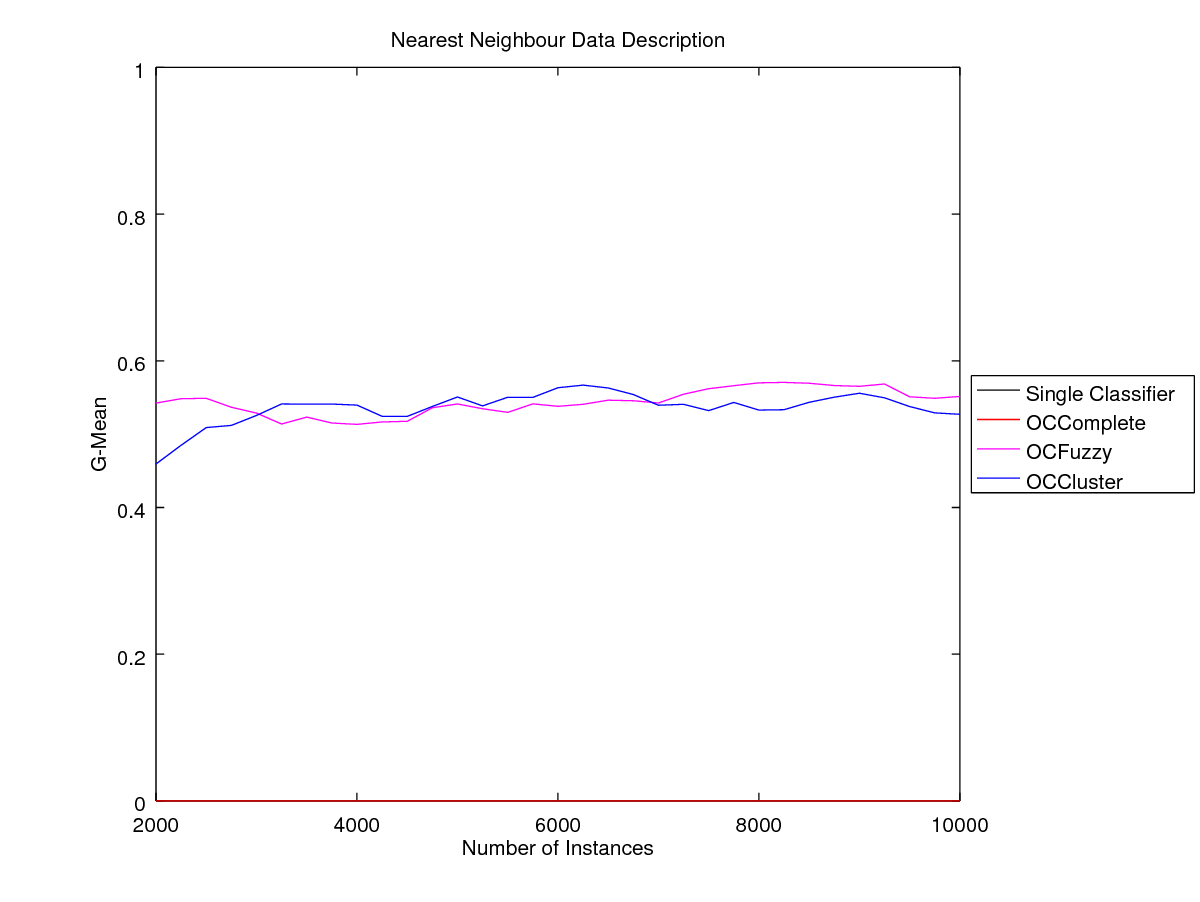}
\end{figure}

\begin{figure}[htbp]
\centering
\caption{Results for the Covertype 2/5 vs 3/4/6 data stream}
\label{fig:covertype25GMean}
\includegraphics[width=0.32\textwidth]{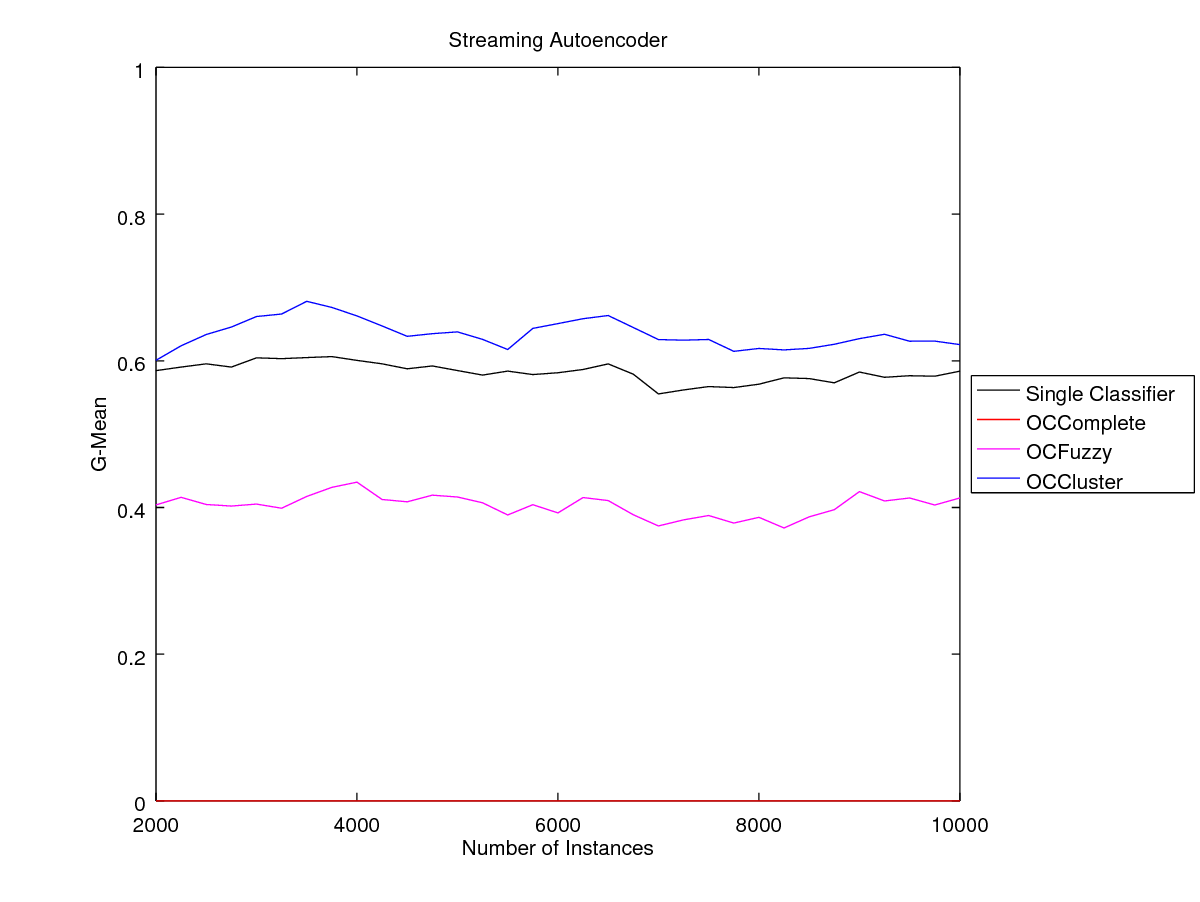}
\includegraphics[width=0.32\textwidth]{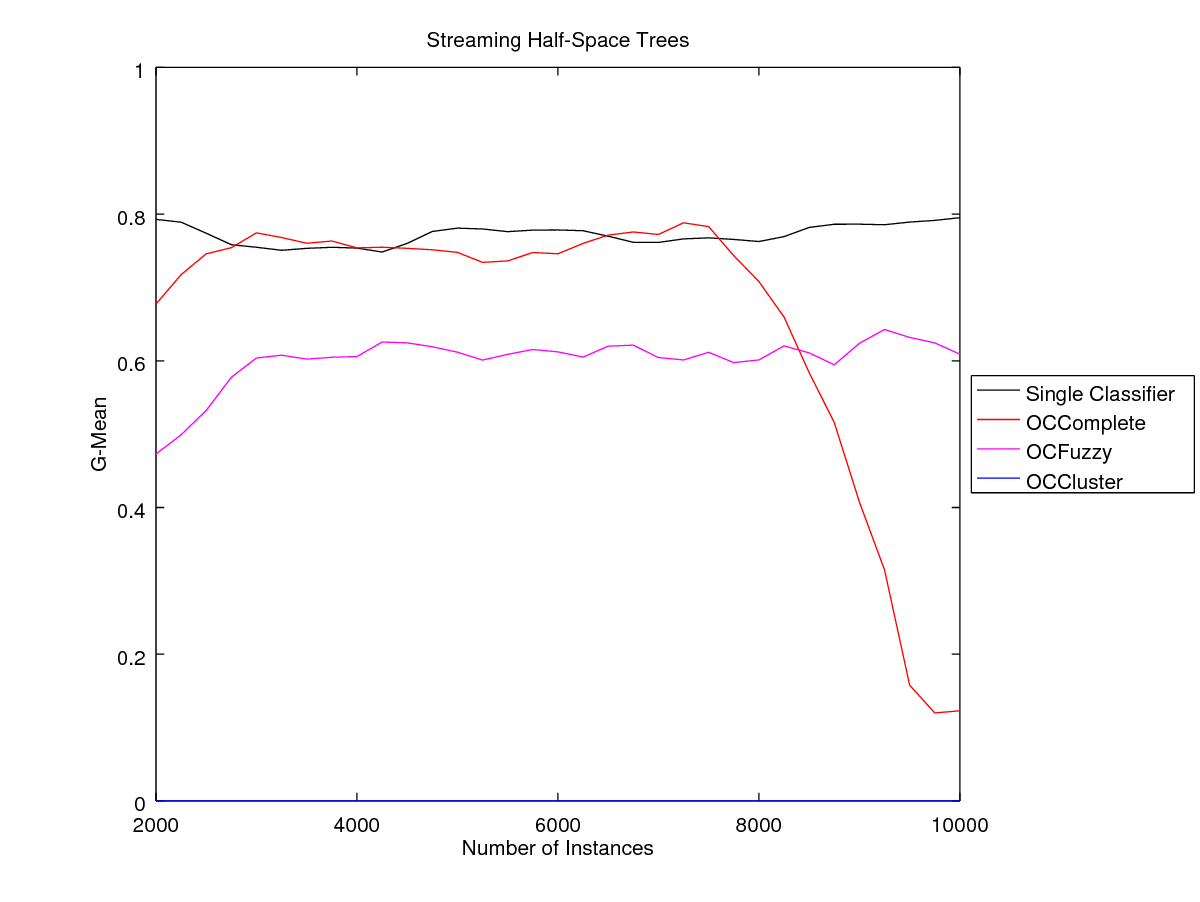}
\includegraphics[width=0.32\textwidth]{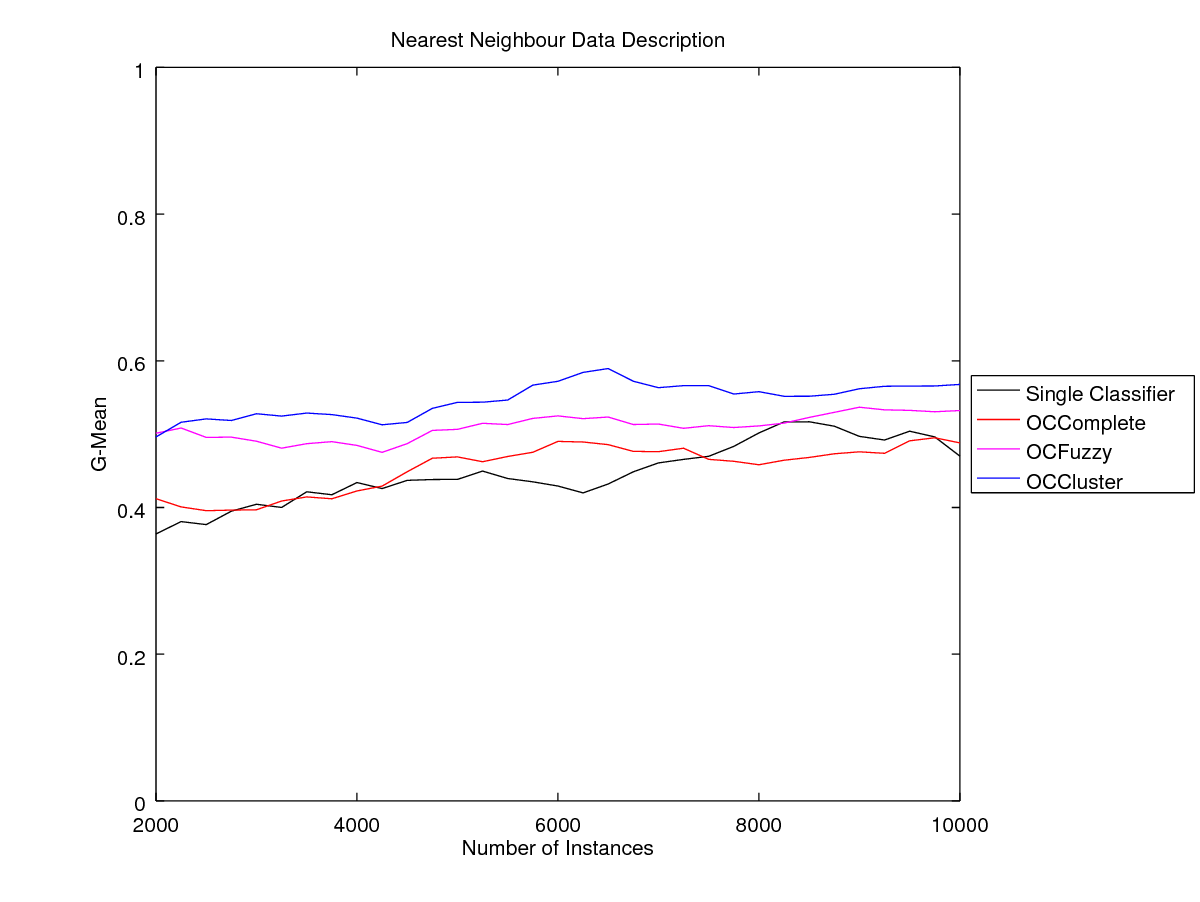}
\end{figure}

\begin{figure}[htbp]
\centering
\caption{Results for the Covertype 1/2/5 vs 3/4/6/7 data stream}
\label{fig:covertype125ResultsGMean}
\includegraphics[width=0.32\textwidth]{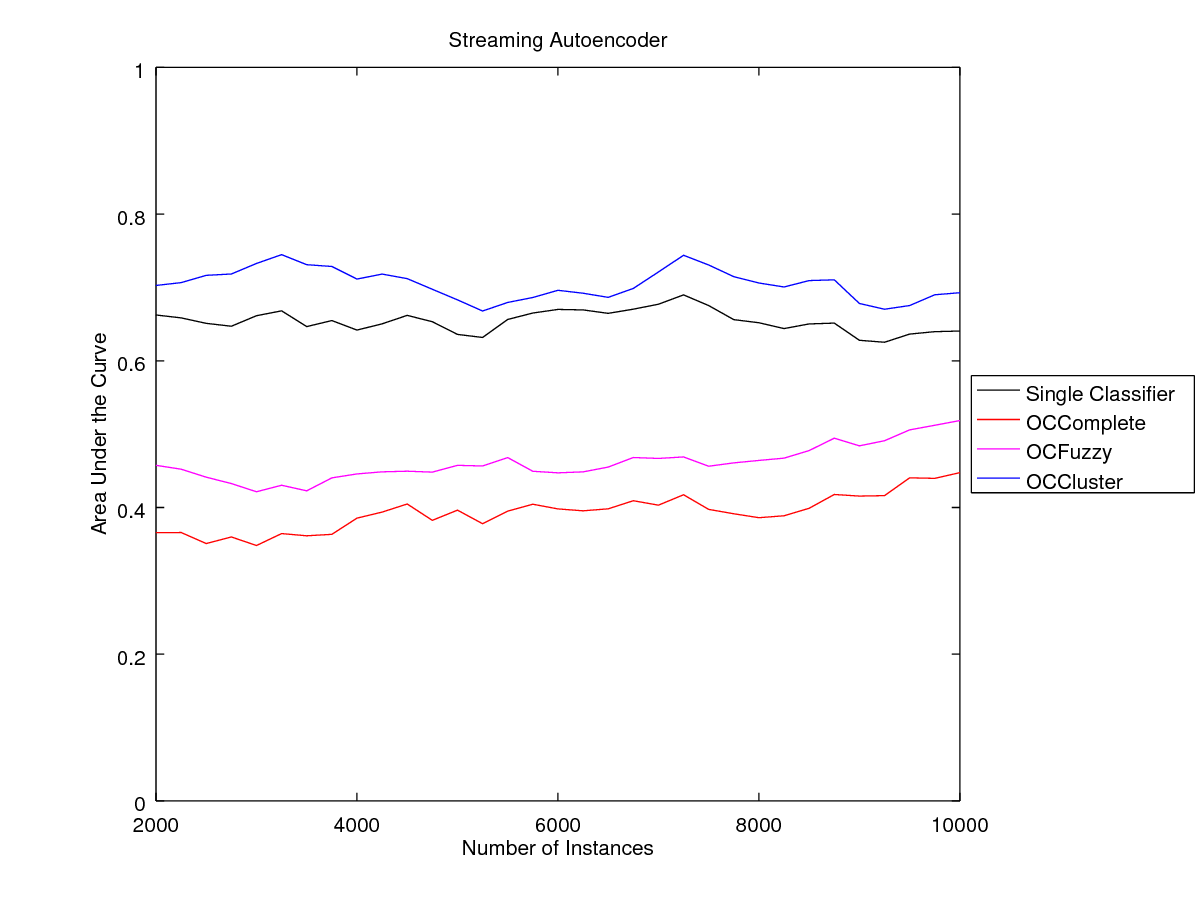}
\includegraphics[width=0.32\textwidth]{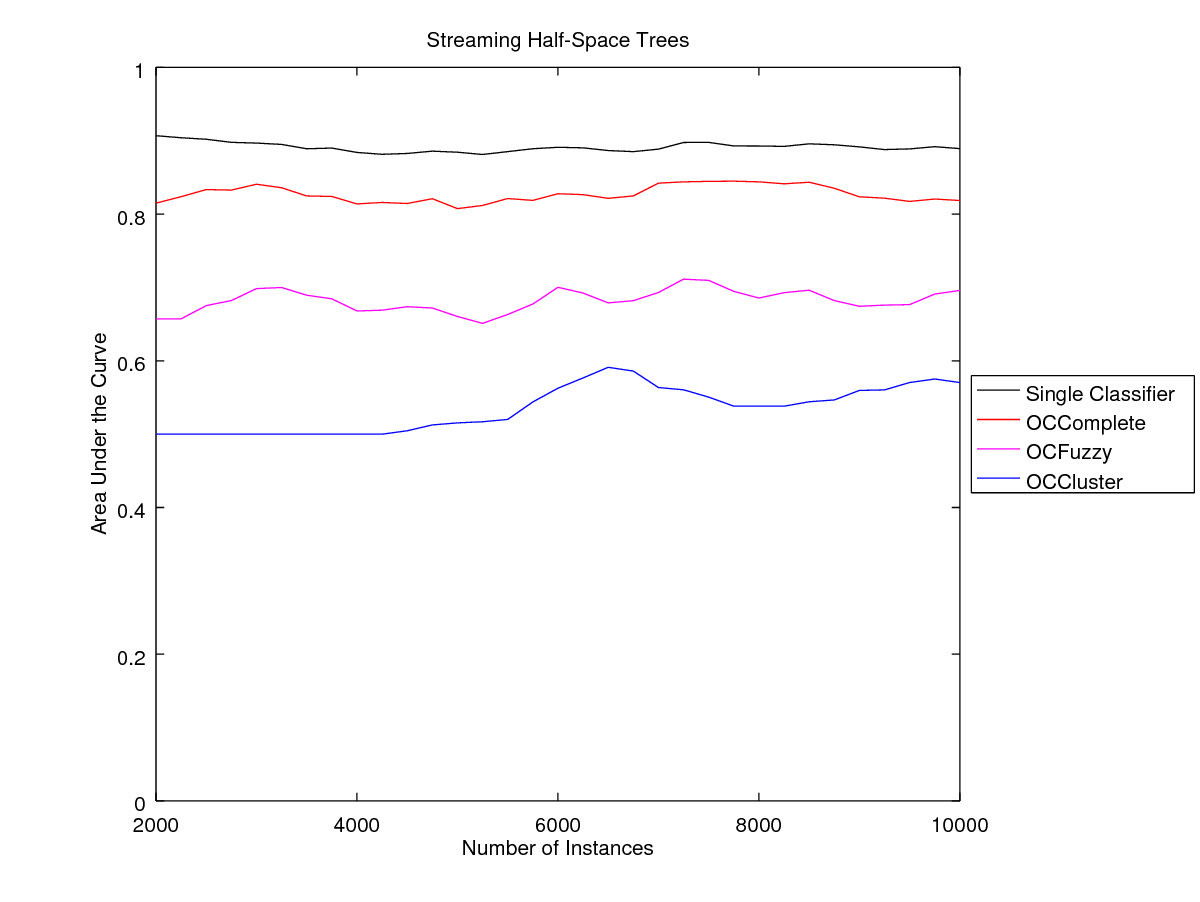}
\includegraphics[width=0.32\textwidth]{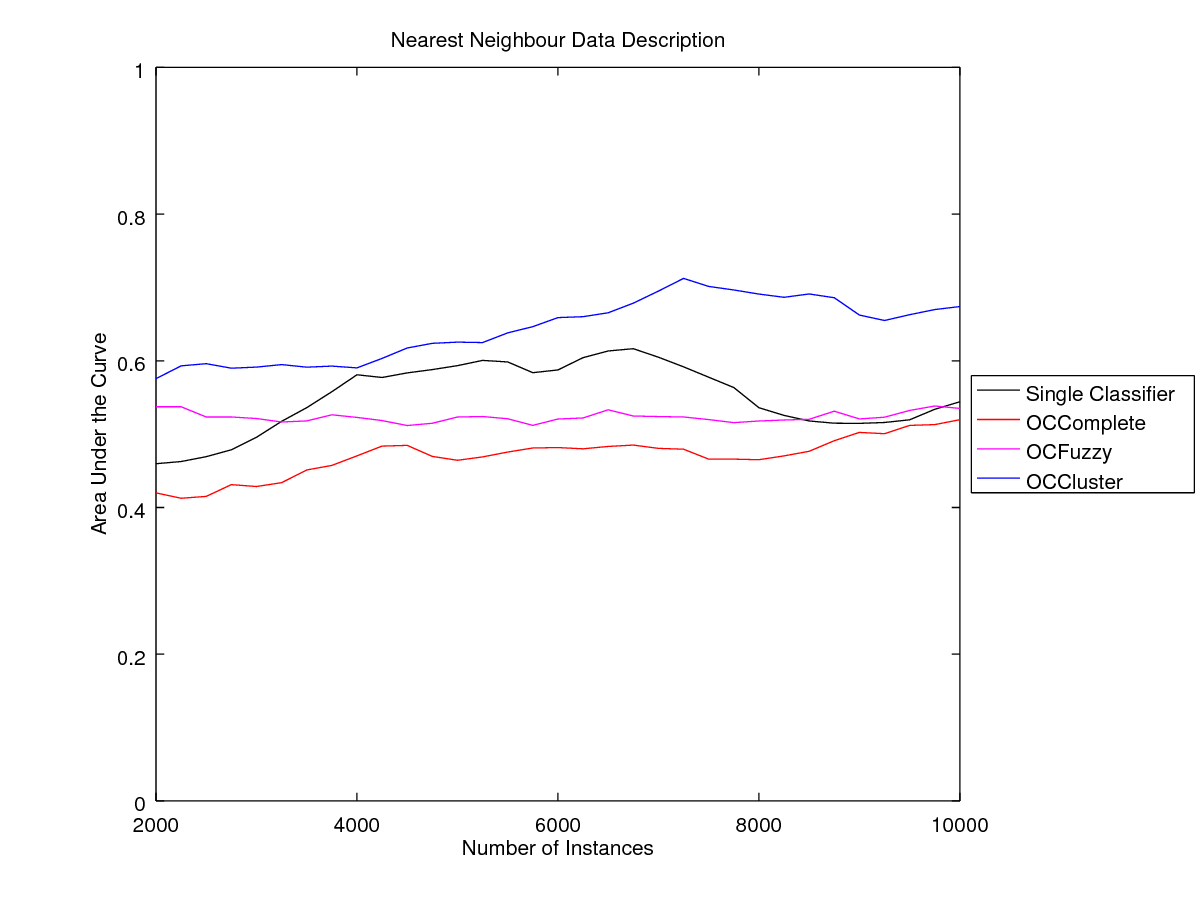}
\includegraphics[width=0.32\textwidth]{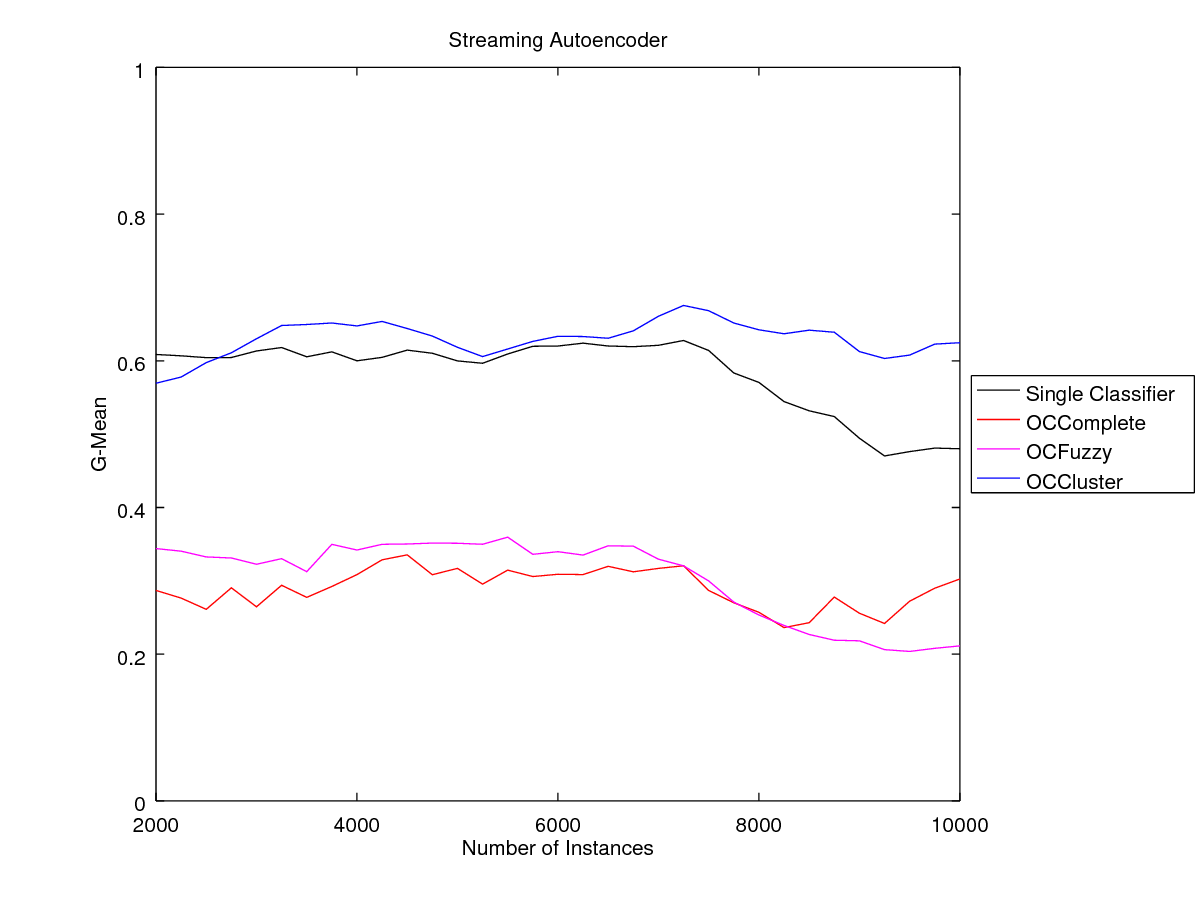}
\includegraphics[width=0.32\textwidth]{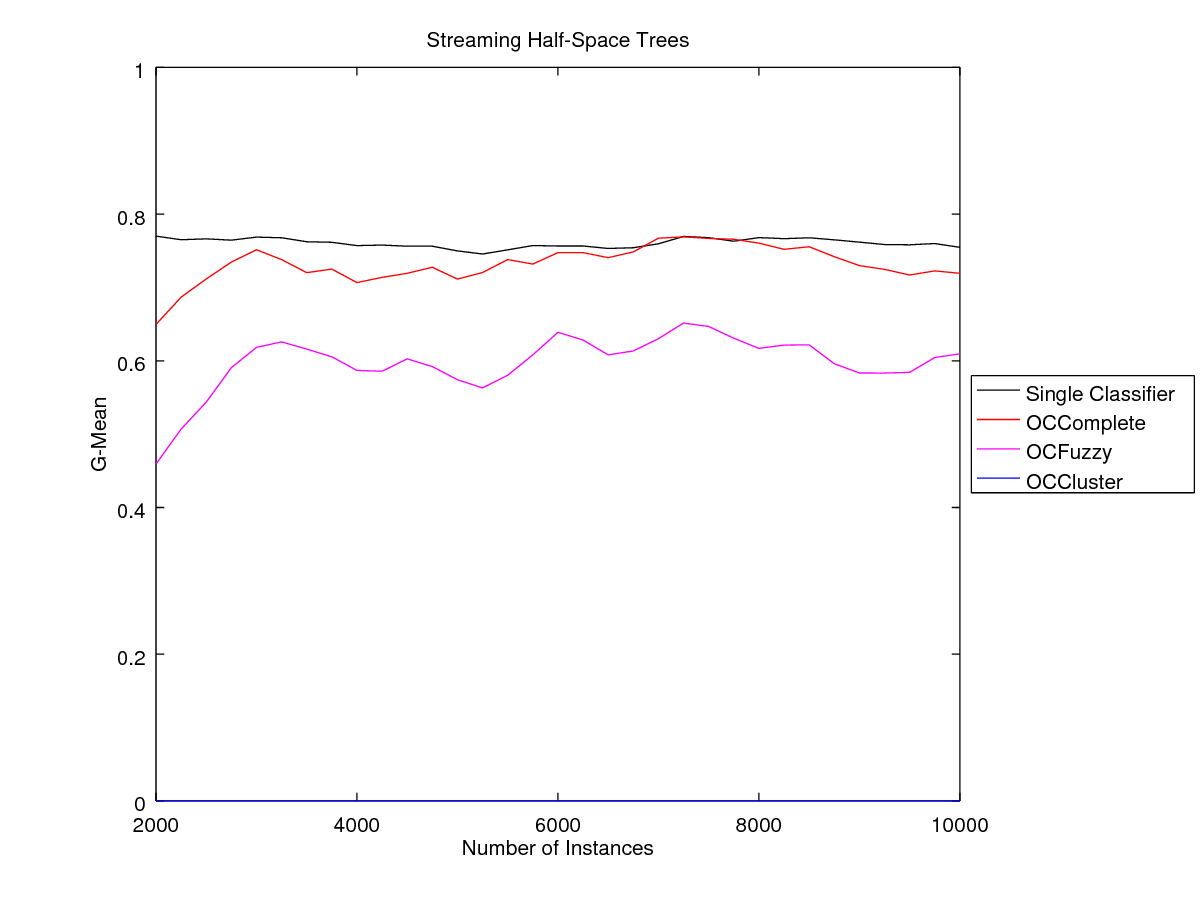}
\includegraphics[width=0.32\textwidth]{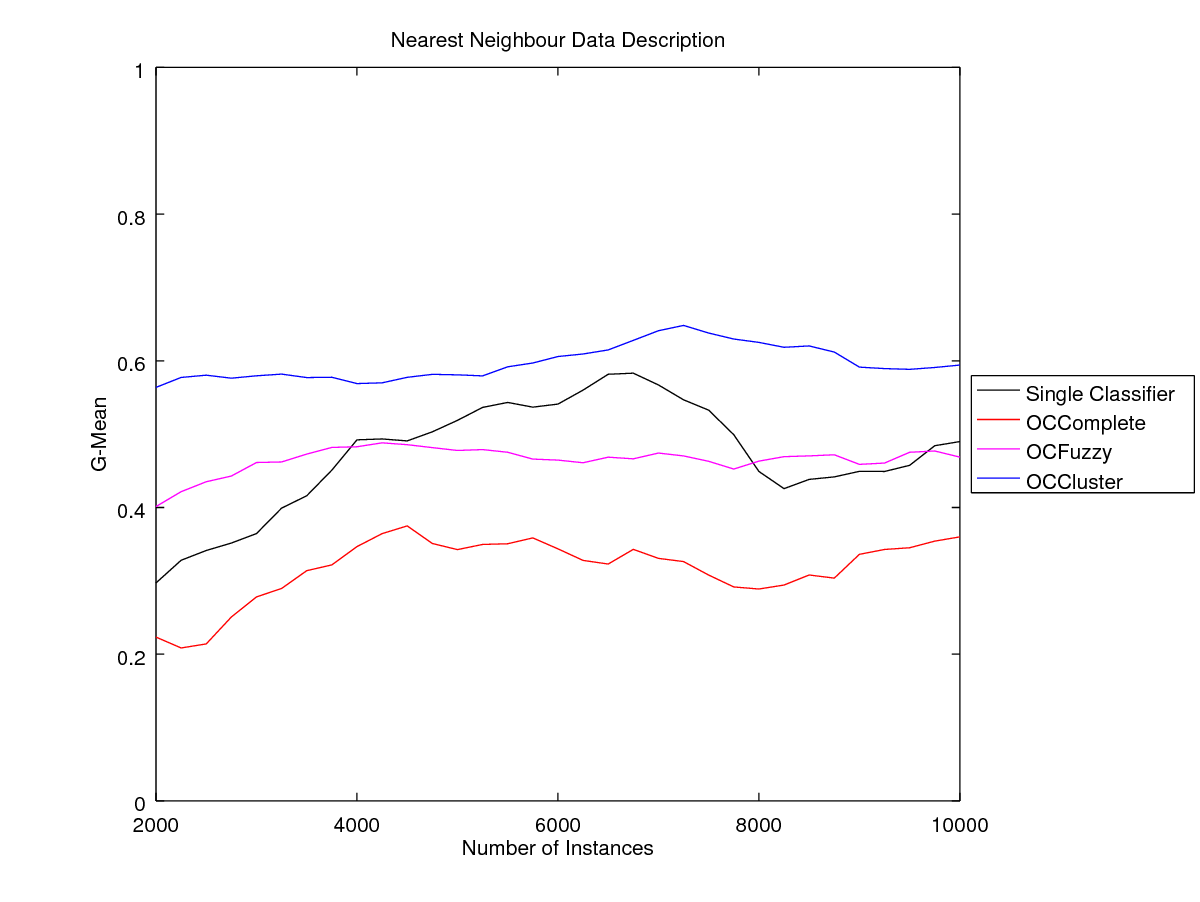}
\end{figure}

\begin{figure}[htbp]
\centering
\caption{Results for the High Time Resolution Universe survey data stream}
\label{fig:htruGMean}
\includegraphics[width=0.32\textwidth]{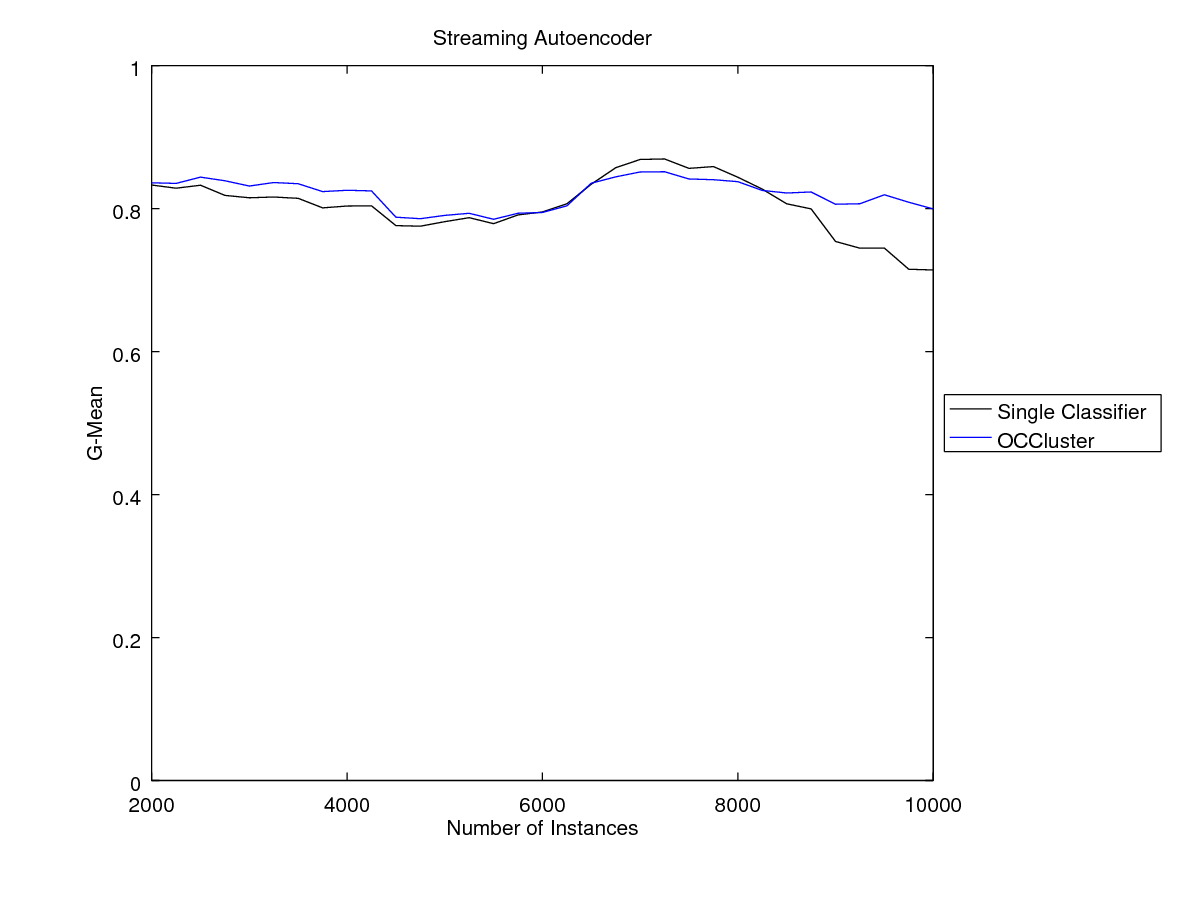}
\includegraphics[width=0.32\textwidth]{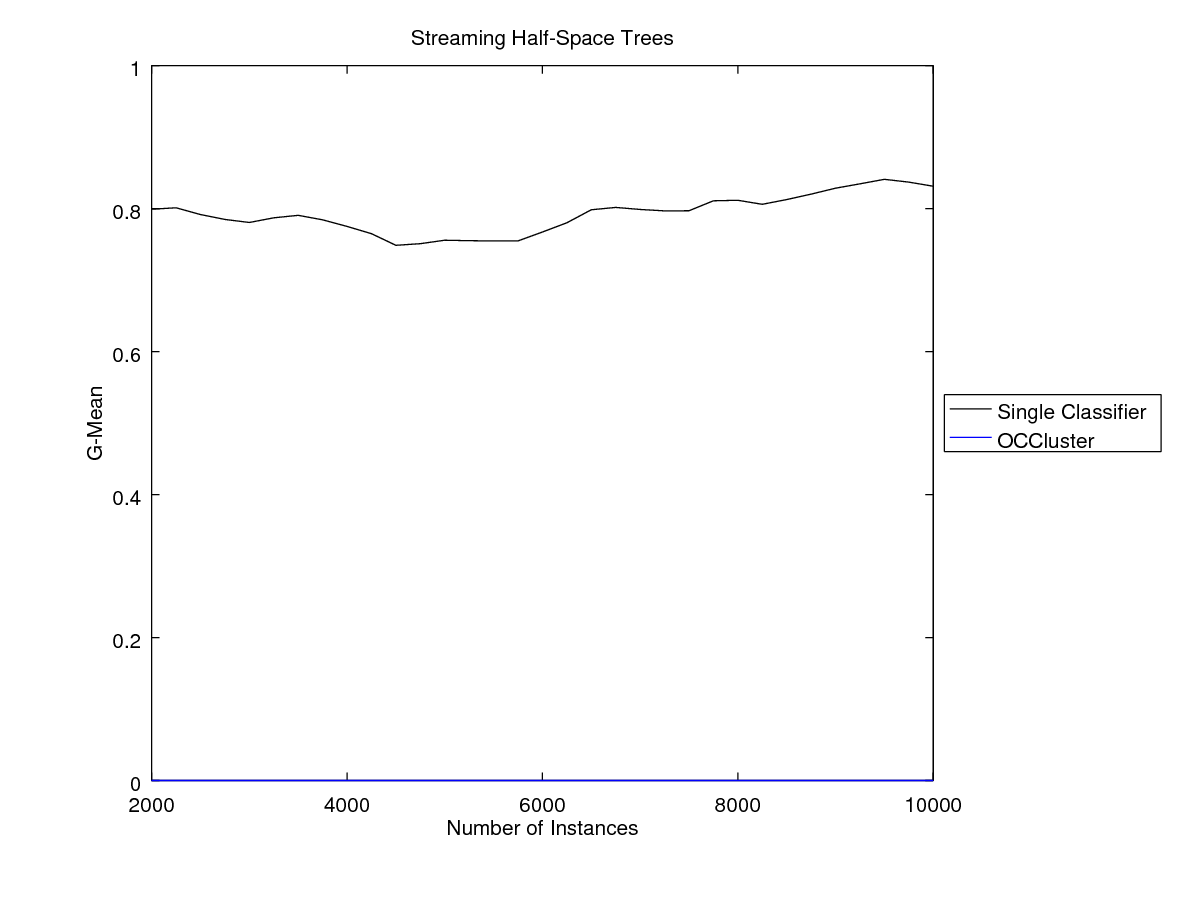}
\includegraphics[width=0.32\textwidth]{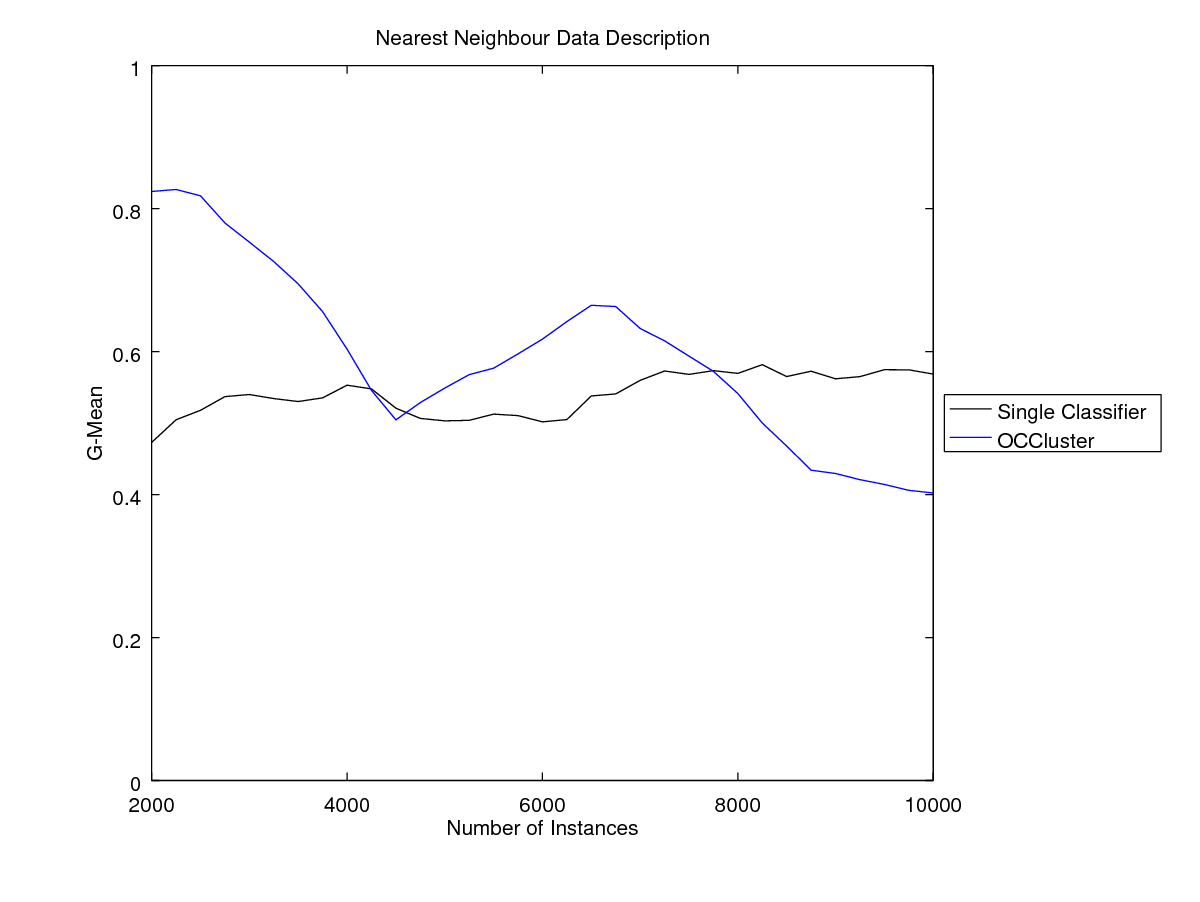}
\end{figure}

\FloatBarrier

\bibliographystyle{plainnat}
\bibliography{Context.bib}

\end{document}